\documentclass[11pt,a4paper,reqno]{amsart}
\pdfoutput=1
\usepackage{times}
\usepackage[margin=1.0in,tmargin=0.9in,bmargin=0.80in]{geometry}
\usepackage[dvipsnames,table]{xcolor}
\usepackage[english]{babel} 

\definecolor{bred}{rgb}{0.8,0,0}
\definecolor{Gray}{gray}{0.85}

\usepackage{amsmath,amssymb,graphicx,epsfig}
\usepackage{mathtools}
\usepackage{mathrsfs}
\usepackage{enumerate,amsthm,epstopdf,mathabx}
\usepackage{esvect}
\usepackage[shortlabels]{enumitem}
\setlist[enumerate]{label={\upshape(\roman*)}}
\usepackage[flushleft]{threeparttable}
\usepackage{multirow}
\usepackage{nicematrix}
\usepackage{longtable}
\usepackage[utf8]{inputenc} 
\usepackage[T1]{fontenc}    
\usepackage{bbm}
\usepackage{xargs}
\usepackage{latexsym}
\usepackage{wasysym}
\usepackage{float}
\usepackage{hyperref}       
\usepackage{url}            
\usepackage{booktabs}       
\usepackage{amsfonts}       
\usepackage{nicefrac}       
\usepackage{microtype}      
\usepackage{cleveref}
\usepackage{dsfont}
\usepackage{footmisc}
\usepackage{array}
\usepackage{subcaption}
\usepackage[linesnumbered,ruled,vlined]{algorithm2e}
\usepackage[section]{placeins}
\usepackage[labelfont=bf]{caption}
\usepackage[textsize=footnotesize]{todonotes}
\hypersetup{colorlinks,linkcolor={blue},citecolor={bred},urlcolor={blue}}
\usepackage[numbers,sort]{natbib}

\SetCommentSty{mycommfont}

\SetKwInput{KwInput}{Input}                
\SetKwInput{KwOutput}{Output}              

\newcommand{\ind}{\text{ind}}

\DeclareMathOperator{\argmin}{arg\,min}
\DeclareMathOperator{\argmax}{arg\,max}

\DeclareMathOperator*{\trace}{tr}

\newcommand{\Pol}{\textup{Pol}}
\DeclareMathOperator{\id}{id}
\DeclareMathOperator*{\linspan}{span}
\DeclareMathOperator{\supp}{supp}

\DeclareMathOperator{\conv}{conv}

\DeclareMathOperator*{\dom}{dom}

\usepackage{stackengine}

\newtheorem{theorem}{Theorem}[section]
\newtheorem{proposition}[theorem]{Proposition}
\newtheorem{lemma}[theorem]{Lemma}

\newtheorem{remark}[theorem]{Remark}
\newtheorem{definition}[theorem]{Definition}
\newtheorem{example}[theorem]{Example}
\newtheorem{assumption}[theorem]{Assumption}
\newtheorem*{standass}{Standing Assumption}

\newcolumntype{L}[1]{>{\raggedright\let\newline\\\arraybackslash\hspace{0pt}}m{#1}}
\newcolumntype{C}[1]{>{\centering\let\newline\\\arraybackslash\hspace{0pt}}m{#1}}
\newcolumntype{R}[1]{>{\raggedleft\let\newline\\\arraybackslash\hspace{0pt}}m{#1}}

\usepackage{scalerel,stackengine}
\stackMath
\newcommand\reallywidehat[1]{%
	\savestack{\tmpbox}{\stretchto{%
			\scaleto{%
				\scalerel*[\widthof{\ensuremath{#1}}]{\kern.1pt\mathchar"0362\kern.1pt}%
				{\rule{0ex}{\textheight}}
			}{\textheight}%
		}{2.4ex}}%
	\stackon[-6.9pt]{#1}{\tmpbox}%
}
\newcommand\reallywidetilde[1]{%
	\savestack{\tmpbox}{\stretchto{%
			\scaleto{%
				\scalerel*[\widthof{\ensuremath{#1}}]{\kern.1pt\mathchar"0366\kern.1pt}%
				{\rule{0ex}{\textheight}}
			}{\textheight}%
		}{2.4ex}}%
	\stackon[-6.9pt]{#1}{\tmpbox}%
}

\makeatletter
\newcommand{\labeltext}[3][]{%
	\@bsphack%
	\csname phantomsection\endcsname
	\def\tst{#1}%
	\def\labelmarkup{\emph}
	\def\refmarkup{}%
	\ifx\tst\empty\def\@currentlabel{\refmarkup{#2}}{\label{#3}}%
	\else\def\@currentlabel{\refmarkup{#1}}{\label{#3}}\fi%
	\@esphack%
	\labelmarkup{#2}
}
\makeatother

\makeatletter
\newcommand\footnoteref[1]{\protected@xdef\@thefnmark{\ref{#1}}\@footnotemark}
\makeatother

\allowdisplaybreaks

\begin{document}

\title[]{Solving stochastic partial differential equations\\using neural networks in the Wiener chaos expansion}

\author[]{Ariel Neufeld}

\address{Nanyang Technological University, Division of Mathematical Sciences, 21 Nanyang Link, Singapore}
\email{ariel.neufeld@ntu.edu.sg}

\author[]{Philipp Schmocker}

\address{ETH Zurich, Department of Mathematics, R\"amistrasse 101, Zurich, Switzerland}
\email{philipp.schmocker@math.ethz.ch}

\date{\today}

\begin{abstract}
	In this paper, we solve stochastic partial differential equations (SPDEs) numerically by using (possibly random) neural networks in the truncated Wiener chaos expansion of their corresponding solution. Moreover, we provide some approximation rates for learning the solution of SPDEs with additive and/or multiplicative noise. Finally, we apply our results in numerical examples to approximate the solution of three SPDEs: the stochastic heat equation, the Heath-Jarrow-Morton equation, and the Zakai equation.
\end{abstract}

\maketitle

\vspace{-0.7cm}

\section{Introduction}
\label{SecIntro}

Given $T > 0$ and a probability space $(\Omega,\mathcal{F},\mathbb{P})$, we consider the numerical approximation of stochastic partial differential equations (i.e.~semilinear stochastic Cauchy problems) of the form
\begin{equation}
	\tag{SPDE}
	\label{EqDefSPDE}
	\begin{cases}
		dX_t & = \left( A X_t + F(t,\cdot,X_t) \right) dt + B(t,\cdot,X_t) dW_t, \quad\quad t \in [0,T], \\
		X_0 & = \chi_0 \in H.
	\end{cases}
\end{equation}
Hereby, the solution of \eqref{EqDefSPDE} is a stochastic process $X: [0,T] \times \Omega \rightarrow H$ with values in a separable Hilbert space $(H,\langle \cdot, \cdot \rangle_H)$, where the initial value $\chi_0 \in H$ is assumed to be deterministic. Moreover, the randomness is induced by a $Q$-Brownian motion $W := (W_t)_{t \in [0,T]}: [0,T] \times \Omega \rightarrow Z$ with values in a (possibly different) separable Hilbert space $(Z,\langle \cdot, \cdot \rangle_Z)$, where the operator $Q \in L_1(Z;Z)$ has finite trace. In addition, the operator $A: \dom(A) \subseteq H \rightarrow H$ is the generator of a $C_0$-semigroup $(S_t)_{t \in [0,T]}$ on $H$, and the maps $F: [0,T] \times \Omega \times H \rightarrow H$ as well as $B: [0,T] \times \Omega \times H \rightarrow L_2(Z_0;H)$ have suitable Lipschitz properties. For more details on the mathematical background, we refer to Section~\ref{SecSPDE}.

In general, stochastic partial differential equations (SPDEs) are a powerful mathematical framework to model complex phenomena influenced by both deterministic dynamics and random fluctuations. These equations extend partial differential equations (PDEs) by incorporating stochastic processes, in the same way ordinary stochastic differential equations (SDEs) generalize ordinary differential equations (ODEs). SPDEs find various applications in a wide range of fields, including surface growth models (see \cite{hairer13}), Euclidean quantum field theories (see \cite{mourrat17}), fluid dynamics (see \cite{birnir13,flandoli23}), stochastic evolution models of biological or chemical quantities (see \cite{kallianpur94,kouritzin02}), interest-rate models (see \cite{filipovic10,harms18}), and stochastic filtering (see \cite{bain09,kushner64,zakai69}). For a more detailed introduction to the theoretical background of SPDEs, we refer to the lecture notes and textbooks \cite{daprato14,dalang09,hairer09,harms17,holden10,jentzen16,vanneerven08,walsh86,lototsky17}.

However, most SPDEs cannot be solved explicitly and therefore require a numerical method to approximate the solution, which entails all the challenges encountered in the numerical approximation of both PDEs and SDEs. Typical numerical approximations consist of temporal discretizations based on Euler type or higher order methods as well as spatial discretizations based on finite difference, finite element, or Galerkin methods (see the references in the overview articles \cite{gyongy02,jentzen09} and the monographs \cite{jentzen11,kruse14}). Recently, machine learning techniques have been used to solve SPDEs, for example by using physics-informed neural networks (see \cite{zhang20}), Fourier neural operators for parametric PDEs (see \cite{li20}), deep neural networks (see \cite{beck20,bagmark23,zhang22,teng21,yao21}), and neural SPDEs (see \cite{salvi22}). Some of these results are in turn inspired by the successful neural network applications for learning PDEs (see e.g.~\cite{beck21,becker21,han18}).

In this paper, we solve \eqref{EqDefSPDE} by first truncating the Wiener chaos expansion of its solution and then replacing the propagators (i.e.~the coefficients of the Wiener chaos expansion) by neural networks. To this end, we use the universal approximation property of neural networks (first proven in \cite{cybenko89,hornik89} and extended in \cite{leshno93,chen95,pinkus99,cuchiero23,neufeld24}) to obtain a universal approximation result for SPDEs. Moreover, we also consider \emph{random} neural networks defined as single-hidden-layer neural networks whose weights and biases inside the activation function are randomly initialized (see \cite{huang06,rahimi07,rahimi08,rahimi08b} and in particular \cite{gonon20,neufeld23}). Hence, only the linear readout needs to be trained, which significantly reduces the computational complexity compared to \emph{deterministic} (i.e.~fully trained) neural networks while maintaining comparable accuracy.

Furthermore, we provide some approximation rates for learning \eqref{EqDefSPDE} with coefficients of affine form. To this end, we derive the Malliavin regularity of the solution to \eqref{EqDefSPDE} (see also \cite{malliavin78,nualart06}), apply the Stroock-Taylor formula in \cite{stroock87} to certain Hilbert space-valued random variables, and use the approximation rates for deterministic/random neural networks in \cite{neufeld23,neufeld24}.

Finally, we provide three numerical experiments to learn the solution of \eqref{EqDefSPDE} with (possibly random) neural networks in the Wiener chaos expansion, which includes the stochastic heat equation, the Heath-Jarrow-Morton equation, and the Zakai equation. This contributes to other successful (random) neural network applications in scientific computation (see e.g.~\cite{dong21,dwivedi20,wang23,yang18,gonon21,herrera21,zuric23,neufeld22,wu24,krach23}).

\subsection{Outline}

In Section~\ref{SecSPDE}, we recall the mathematical framework of \eqref{EqDefSPDE} and introduce the Wiener chaos expansion. In Section~\ref{SecUAT}, we use (possibly random) neural networks in the chaos expansion of \eqref{EqDefSPDE} to obtain universal approximation results. In Section~\ref{SecAR}, we provide some approximation rates to learn \eqref{EqDefSPDE}. In Section~\ref{SecNumerics}, we present three numerical examples, while all proofs are given in Section~\ref{SecProofs}.

\subsection{Notation}
\label{SecNotation}

As usual, $\mathbb{N} := \lbrace 1, 2, 3, ... \rbrace$ and $\mathbb{N}_0 := \mathbb{N} \cup \lbrace 0 \rbrace$ denote the sets of natural numbers, $\mathbb{Z}$ represents the set of integers, and $\mathbb{R}$ as well as $\mathbb{C}$ are the sets of real and complex numbers (with imaginary unit $\mathbbm{i} := \sqrt{-1} \in \mathbb{C}$), respectively. For $s,t \in \mathbb{R}$, we define $s \wedge t := \min(s,t)$. In addition, for $m \in \mathbb{N}$, we denote by $\mathbb{R}^m$ (and $\mathbb{C}^m$) the (complex) Euclidean space equipped with $\Vert u \Vert = \big( \sum_{i=1}^m \vert u_i \vert^2 \big)^{1/2}$.

Moreover, a Hilbert space $(H,\langle \cdot, \cdot \rangle_H)$ is a (possibly infinite dimensional) vector space $H$ together with an inner product $\langle \cdot, \cdot \rangle_H$ such that $H$ is complete under the norm $\Vert x \Vert_H := \sqrt{\langle x, x \rangle_H}$. Hereby, $\mathcal{B}(H)$ represents the Borel $\sigma$-algebra of $H$. In addition, we denote by $L(H;H)$ the vector space of bounded linear operators $\Xi: H \rightarrow H$, e.g.~the identity $\id_H \in L(H;H)$, which is a Banach space under the norm $\Vert \Xi \Vert_{L(H;H)} := \sup_{x \in H, \, \Vert x \Vert_H \leq 1} \Vert \Xi x \Vert_H$. Moreover, a (possibly unbounded) operator $A: \dom(A) \subseteq H \rightarrow H$ is a linear operator that is only defined on a vector subspace $\dom(A) \subseteq H$ called the domain of $A$. Furthermore, for $T > 0$, we denote by $C^0([0,T];H)$ the vector space of continuous paths $f: [0,T] \rightarrow H$, which is a Banach space under the norm $\Vert f \Vert_{C^0([0,T];H)} := \sup_{t \in [0,T]} \Vert f(t) \Vert_H$. 

In addition, for $k \in \mathbb{N}_0$, $m,d \in \mathbb{N}$ and $U \subseteq \mathbb{R}^m$ (open, if $k \geq 1$), we denote by $C^k_b(U;\mathbb{R}^d)$ the vector space of bounded and $k$-times continuously differentiable functions $f: U \rightarrow \mathbb{R}^d$ such that for every $\beta \in \mathbb{N}^m_{0,k} := \lbrace \beta := (\beta_1,...,\beta_m) \in \mathbb{N}_0^m: \vert\beta\vert := \beta_1 + ... + \beta_m \leq k \rbrace$ the partial derivative $U \ni u \mapsto \partial_\beta f(u) := \frac{\partial^{\vert\beta\vert} f}{\partial u_1^{\beta_1} \cdots \partial u_m^{\beta_m}}(u) \in \mathbb{R}^d$ is bounded and continuous, which is a Banach space under the norm $\Vert f \Vert_{C^k_b(U;\mathbb{R}^d)} := \max_{\beta \in \mathbb{N}^m_{0,k}} \sup_{u \in U} \Vert \partial_\beta f(u) \Vert$. If $m = 1$, we write $f^{(j)} := \frac{\partial^j f}{\partial u^j}: \mathbb{R} \rightarrow \mathbb{R}^d$, $j = 0,...,k$. Moreover, for $\gamma \in [0,\infty)$, we denote by $C^k_{pol,\gamma}(U;\mathbb{R}^d)$ the vector space of $k$-times continuously differentiable functions $f: U \rightarrow \mathbb{R}^d$ such that $\Vert f \Vert_{C^k_{pol,\gamma}(U;\mathbb{R}^d)} := \max_{\beta \in \mathbb{N}^m_{0,k}} \sup_{u \in U} \frac{\Vert \partial_\beta f(u) \Vert}{(1+\Vert u \Vert)^\gamma} < \infty$. In addition, we define $\overline{C^k_b(\mathbb{R}^m;\mathbb{R}^d)}^\gamma$ as the closure of $C^k_b(\mathbb{R}^m;\mathbb{R}^d)$ with respect to $\Vert \cdot \Vert_{C^k_{pol,\gamma}(\mathbb{R}^m;\mathbb{R}^d)}$, which is a Banach space under $\Vert \cdot \Vert_{C^k_{pol,\gamma}(\mathbb{R}^m;\mathbb{R}^d)}$. Then, $f \in \overline{C^k_b(\mathbb{R}^m;\mathbb{R}^d)}^\gamma$ if and only if $f: \mathbb{R}^m \rightarrow \mathbb{R}^d$ is $k$-times continuously differentiable and $\lim_{r \rightarrow \infty} \max_{\alpha \in \mathbb{N}^m_{0,k}} \sup_{u \in \mathbb{R}^m, \, \Vert u \Vert \geq r} \frac{\Vert \partial_\alpha f(u) \Vert}{(1+\Vert u \Vert)^\gamma} = 0$ (see \cite[Notation~(v)]{neufeld24}).

Furthermore for $p \in [1,\infty)$, a measure space $(S,\Sigma,\mu)$, and a Banach space $(Y,\Vert \cdot \Vert_Y)$, we denote by $L^p(S,\Sigma,\mu;Y)$ the Bochner $L^p$-space of (equivalence classes of) strongly $\mu$-measurable maps $f: S \rightarrow Y$ such that $\Vert f \Vert_{L^p(S,\Sigma,\mu;Y)} := \left( \int_S \Vert f(s) \Vert_Y^p \mu(ds) \right)^{1/p} < \infty$, where the latter turns $L^p(S,\Sigma,\mu;Y)$ into a Banach space (see \cite[Section~1.2.b]{hytoenen16} for more details). For finite dimensional $(Y,\Vert \cdot \Vert_Y)$, we obtain the usual $L^p$-space of (equivalence classes of) $\Sigma/\mathcal{B}(Y)$-measurable functions $f \in L^p(S,\Sigma,\mu;Y)$. In addition, $\mathcal{L}(U)$ represents the $\sigma$-algebra of Lebesgue-measurable subsets of $U \in \mathcal{B}(\mathbb{R}^m)$.

Moreover, we define the (multi-dimensional) Fourier transform of $f \in L^1(\mathbb{R}^m,\mathcal{L}(\mathbb{R}^m),du;\mathbb{C}^d)$ as
\begin{equation}
	\label{EqDefFourier}
	\mathbb{R}^m \ni \zeta \quad \mapsto \quad \widehat{f}(\zeta) := \int_{\mathbb{R}^m} e^{-\mathbbm{i} \zeta^\top u} f(u) du \in \mathbb{C}^d.
\end{equation}
\vspace{-0.1cm}
In addition, we abbreviate the real-valued function spaces by $C^k_b(U) := C^k_b(U;\mathbb{R})$, $L^1(U,\mathcal{L}(U),du) := L^1(U,\mathcal{L}(U),du;\mathbb{R})$, etc. Furthermore, we define the complex-valued function spaces as $C^k_b(U;\mathbb{C}^d) \cong C^k_b(U;\mathbb{R}^{2d})$, $L^1(U,\mathcal{L}(U),du;\mathbb{C}^d) := L^1(U,\mathcal{L}(U),du;\mathbb{R}^{2d})$, etc.~by using the identification $\mathbb{C}^d \cong \mathbb{R}^{2d}$.

\section{Stochastic partial differential equations (SPDEs)}
\label{SecSPDE}

In this section, we provide some mathematical background on existence and uniqueness of solutions to \eqref{EqDefSPDE} and consider the Wiener chaos expansion. To this end, we fix throughout the paper some $T > 0$, a probability space $(\Omega,\mathcal{F},\mathbb{P})$, and two separable Hilbert spaces $(H,\langle \cdot, \cdot \rangle_H)$ and $(Z,\langle \cdot, \cdot \rangle_Z)$.

\subsection{Existence and uniqueness of mild solutions}

In order to recall existence and uniqueness results of mild solutions to \eqref{EqDefSPDE}, we follow the textbook \cite{daprato14}. To this end, we denote by $L_1(Z;Z) \subseteq L(Z;Z)$ the vector subspace of non-negative self-adjoint nuclear\footnote{An operator $Q \in L(Z;Z)$ is called \emph{non-negative} if $\langle Qz, z \rangle_Z \geq 0$ for all $z \in Z$, is called \emph{self-adjoint} if $\langle Qy, z \rangle_Z = \langle y, Qz \rangle_Z$ for all $y,z \in Z$, and is called \emph{nuclear} if there exist two sequences $(y_i)_{i \in \mathbb{N}},(z_i)_{i \in \mathbb{N}} \subseteq Z$ with $\sum_{i=1}^\infty \Vert y_i \Vert_Z \Vert z_i \Vert_Z < \infty$ such that for every $z \in Z$ it holds that $Qz = \sum_{i=1}^\infty \langle z, y_i \rangle_Z z_i$ (see \cite[Appendix~C]{daprato14}).} operators $Q \in L(Z;Z)$. Moreover, a $Z$-valued random variable $G: \Omega \rightarrow Z$ is called \emph{centered Gaussian (with covariance operator $Q \in L_1(Z;Z)$)} if for every $z \in Z$ the Fourier transform of $G$ satisfies $\mathbb{E}\big[ e^{-\mathbf{i} \xi \langle z, G \rangle_Z} \big] = e^{-\frac{1}{2} \langle Qz, z \rangle_Z}$.

\begin{definition}
	\label{DefBM}
	For $Q \in L_1(Z;Z)$, a process $W := (W_t)_{t \in [0,T]}: [0,T] \times \Omega \rightarrow Z$ is called a \emph{$Q$-Brownian motion} if
	\begin{enumerate}
		\item $W_0 = 0 \in H$,
		\item $W$ has continuous sample paths, i.e.~$[0,T] \ni t \mapsto W_t(\omega) \in Z$ is continuous for all $\omega \in \Omega$, and
		\item for every $0 \leq s < t \leq T$ the random variable $W_t-W_s$ is centered Gaussian with covariance operator $(t-s) Q \in L_1(Z;Z)$.
	\end{enumerate}
\end{definition}

\begin{standass}
	For some $Q \in L_1(Z;Z)$, we fix a $Q$-Brownian motion $W := (W_t)_{t \in [0,T]}: [0,T] \times \Omega \rightarrow Z$. Moreover, $\mathbb{F} := (\mathcal{F}_t)_{t \in [0,T]}$ denotes the usual $\mathbb{P}$-augmented filtration\footnote{The usual $\mathbb{P}$-augmented filtration generated by $W$ is defined as the smallest filtration $\mathbb{F} = (\mathcal{F}_t)_{t \in [0,T]}$ such that $\mathbb{F}$ is complete with respect to $(\Omega,\mathcal{F},\mathbb{P})$, $\mathbb{F}$ is right-continuous on $[0,T)$, and $W$ is $\mathbb{F}$-adapted (see \cite[p.~45]{revuz99}).} generated by $W$.
\end{standass}

Then, for every $y \in Z$, the process $[0,T] \ni t \mapsto \langle W_t, y \rangle_Z \in \mathbb{R}$ is a real-valued Brownian motion, which implies for every $s,t \in [0,T]$ and $y,z \in Z$ that
\begin{equation*}
	\mathbb{E}\left[ \langle W_s, y \rangle_Z \, \langle W_t, z \rangle_Z \right] = (s \wedge t) \, \langle Qy, z \rangle_Z.
\end{equation*}
Note that the Hilbert space\footnote{Every non-negative self-adjoint operator $Q \in L(Z;Z)$ has a unique non-negative self-adjoint \emph{square root} $Q^{1/2} \in L(Z;Z)$ satisfying $Q^{1/2} Q^{1/2} = Q \in L(Z;Z)$ (see \cite[Problem~39.C.1]{brezis11}). Moreover, the \emph{pseudo-inverse} $Q^{-1/2}: Z \rightarrow Z$ of $Q^{1/2} \in L(Z;Z)$ is defined as $Q^{-1/2} y := \argmin_{\lbrace z \in Z: Q^{1/2} z = y \rbrace} \Vert z \Vert_Z$ for $y \in Q^{1/2} Z$ (see \cite[Appendix~B.2]{daprato14}).} $Z_0 := Q^{1/2} Z \subseteq Z$ with inner product $\langle y,z \rangle_{Z_0} := \langle Q^{-1/2} y, Q^{-1/2} z \rangle_Z$, $y,z \in Z_0$, is a reproducing kernel Hilbert space (RKHS) of $W$ in the sense that $(Z_0,\langle \cdot, \cdot \rangle_{Z_0})$ is continuously embedded into $(Z,\langle \cdot,\cdot \rangle_Z)$ and for every $z \in Z$ the real-valued random variable $\langle W_t, z \rangle_Z$ is centered Gaussian with covariance $\langle z, z \rangle_{Z_0}$.

In addition, since $Q \in L_1(Z;Z)$ is a non-negative nuclear operator, there exists a complete orthonormal basis $(e_i)_{i \in \mathbb{N}} \subseteq Z$ of $(Z,\langle \cdot, \cdot \rangle_Z$) and a sequence $(\lambda_i)_{i \in \mathbb{N}} \subseteq [0,\infty)$ with $\sum_{i=1}^\infty \lambda_i < \infty$ such that for every $i \in \mathbb{N}$ it holds that $Q e_i = \lambda_i e_i$ (see \cite[Appendix~C]{daprato14}). Then, we define for every $i \in \mathbb{N}$ the process
\begin{equation}
	\label{EqDefBMi}
	[0,T] \ni t \quad \mapsto \quad W^{(i)}_t := \frac{1}{\sqrt{\lambda_i}} \langle W_t, e_i \rangle_Z \in \mathbb{R},
\end{equation}
which are pairwise independent real-valued Brownian motions that are able to recover the Hilbert space-valued Brownian motion $W: [0,T] \times \Omega \rightarrow Z$ in the following sense.

\begin{lemma}[{\cite[Proposition~4.3]{daprato14}}]
	\label{LemmaBM}
	The processes $(W^{(i)})_{i \in \mathbb{N}}$ defined in \eqref{EqDefBMi} are pairwise independent real-valued Brownian motions such that for every $t \in [0,T]$ it holds that
	\begin{equation}
		\label{EqLemmaBM1}
		W_t = \sum_{i=1}^\infty \sqrt{\lambda_i} W^{(i)}_t e_i,
	\end{equation}
	where the sum converges with respect to $\Vert \cdot \Vert_{L^2(\Omega,\mathcal{F},\mathbb{P};Z)}$.
\end{lemma}

Next, we recall the concept of a mild solution to \eqref{EqDefSPDE} (see also \cite[p.~187]{daprato14}). To this end, we denote by $L_2(Z_0;H) \subseteq L(Z_0;H)$ the vector subspace of Hilbert–Schmidt operators\footnote{An operator $\Xi \in L(Z_0;H)$ is called \emph{Hilbert-Schmidt} if $\sum_{i=1}^\infty \Vert \Xi \widetilde{e}_i \Vert_H^2 < \infty$ (see \cite[Appendix~C]{daprato14}).} $\Xi \in L(Z_0;H)$, equipped with the inner product $\langle \Xi_1, \Xi_2 \rangle_{L_2(Z_0;H)} := \sum_{i=1}^\infty \langle \Xi_1 \widetilde{e}_i, \Xi_2 \widetilde{e}_i \rangle_H$ for $\Xi_1,\Xi_2 \in L_2(Z_0;H)$, where $(\widetilde{e}_i) := \big( Q^{1/2} e_i \big)_{i \in \mathbb{N}}$ is an orthonormal basis of $(Z_0,\langle\cdot,\cdot\rangle_{Z_0})$.

We impose the following Lipschitz and linear growth conditions on the coefficients of \eqref{EqDefSPDE}, where we denote by $\mathcal{P}_T$ the $\mathbb{F}$-predictable $\sigma$-algebra on $[0,T] \times \Omega$ (see \cite[p.~71+72]{daprato14}).

\begin{assumption}
	\label{AssSPDE}
	Let the following hold true:
	\begin{enumerate}
		\item\label{AssSPDE1} The map $A: \dom(A) \subseteq H \rightarrow H$ is the generator\footnote{A family of operators $(S_t)_{t \in [0,T]} \subseteq L(H;H)$ is called a \emph{$C_0$-semigroup} if $S_0 = \id_H \in L(H;H)$, $S_s S_t = S_{s+t}$ for all $s,t \in [0,T]$ with $s+t \in [0,T]$, and $\lim_{t \rightarrow 0} \Vert S_t x - x \Vert_H = 0$ for all $x \in H$. Moreover, its generator $A: \dom(A) \subseteq H \rightarrow H$ is defined as $Ax := \lim_{t \rightarrow 0} \frac{S_t x - x}{t}$ for $x \in \dom(A) := \big\lbrace x \in H: \lim_{t \rightarrow 0} \frac{S_t x - x}{t} \in H \text{ exists} \big\rbrace$ (see \cite[Appendix~A.1+A.2]{daprato14}). Furthermore, the constant $C_S := \sup_{t \in [0,T]} \Vert S_t \Vert_{L(H;H)}$ is finite by the Banach-Steinhaus theorem (see \cite[Theorem~2.2]{brezis11}).} of a $C_0$-semigroup $(S_t)_{t \in [0,T]}$.
		\item\label{AssSPDE2} The map $F: [0,T] \times \Omega \times H \rightarrow H$ is $(\mathcal{P}_T \otimes \mathcal{B}(H))/\mathcal{B}(H)$-measurable.
		\item\label{AssSPDE3} The map $B: [0,T] \times \Omega \times H \rightarrow L_2(Z_0;H)$ is $(\mathcal{P}_T \otimes \mathcal{B}(H))/\mathcal{B}(L_2(Z_0;H))$-measurable.
		\item\label{AssSPDE4} There exist some $C_{F,B} > 0$ such that for every $t \in [0,T]$, $\omega \in \Omega$, and $x,y \in H$ it holds that
		\begin{equation*}
			\begin{aligned}
				\quad\quad \Vert F(t,\omega,x) - F(t,\omega,y) \Vert_H + \Vert B(t,\omega,x) - B(t,\omega,y) \Vert_{L_2(Z_0;H)} & \leq C_{F,B} \Vert x-y \Vert_H, \quad\quad \text{and} \\
				\Vert F(t,\omega,x) \Vert_H^2 + \Vert B(t,\omega,x) \Vert_{L_2(Z_0;H)}^2 & \leq C_{F,B}^2 \left( 1 + \Vert x \Vert_H^2 \right).
			\end{aligned}
		\end{equation*}
		\item\label{AssSPDE5} The initial condition $\chi_0 \in H$ is deterministic.
	\end{enumerate}
\end{assumption}

In addition, we recall that the stochastic integral of an $\mathbb{F}$-predictable integrand with values in $L_2(Z_0;H)$ is defined as the $L^2$-limit of stochastic integrals of elementary processes (see \cite[Section~4.2]{daprato14}).

\begin{definition}
	Let Assumption~\ref{AssSPDE}~\ref{AssSPDE1}-\ref{AssSPDE3} hold. Then, an $\mathbb{F}$-predictable process $X: [0,T] \times \Omega \rightarrow H$ is called \emph{a mild solution} of \eqref{EqDefSPDE} if $\mathbb{P}\big[ \int_0^T \Vert X_t \Vert_H^2 dt < \infty \big] = 1$ and for every $t \in [0,T]$ it holds that
	\begin{equation}
		\label{EqDefMildSolution}
		X_t = S_t \chi_0 + \int_0^t S_{t-s} F(s,\cdot,X_s) ds + \int_0^t S_{t-s} B(s,\cdot,X_s) dW_s, \quad\quad \mathbb{P}\text{-a.s.}
	\end{equation}
\end{definition}

Under Assumption~\ref{AssSPDE}, one can then show that \eqref{EqDefSPDE} admits a mild solution. The proof is based on \cite[Theorem~7.2~(i)+(iii)]{daprato14} with a slight modification for $p \in [1,2]$, see Section~\ref{SecAuxProofsSPDE}.

\begin{proposition}[Existence and uniqueness of mild solutions to \eqref{EqDefSPDE}]
	\label{PropSPDE}
	Let $p \in [1,\infty)$ and let Assumption~\ref{AssSPDE} hold. Then, there exists a unique\footnote{Unique up to indistinguishability among the $\mathbb{F}$-predictable processes $\widetilde{X}: [0,T] \times \Omega \rightarrow H$ with $\mathbb{P}\big[ \int_0^T \Vert \widetilde{X}_t \Vert_H^2 dt < \infty \big] = 1$.} mild solution $X: [0,T] \times \Omega \rightarrow H$ of \eqref{EqDefSPDE} and a constant $C^{(p)}_{F,B,S,T} \geq 1$ (depending only on $p \in [1,\infty)$, $C_{F,B} > 0$, $C_S := \sup_{t \in [0,T]} \Vert S_t \Vert_{L(H;H)} < \infty$, and $T > 0$) such that
	\begin{equation}
		\label{EqPropSPDE1}
		\mathbb{E}\left[ \sup_{t \in [0,T]} \Vert X_t \Vert_H^p \right] \leq C^{(p)}_{F,B,S,T} \left( 1 + \Vert \chi_0 \Vert_H^p \right) < \infty.
	\end{equation}
	In addition, the process $X: [0,T] \times \Omega \rightarrow H$ admits a continuous modification.
\end{proposition}

By a slight abuse of notation, we denote the continuous modification also by $X: [0,T] \times \Omega \rightarrow H$. Then, for every $p \in [1,\infty)$, Proposition~\ref{PropSPDE} shows that the map $\Omega \ni \omega \mapsto (t \mapsto X_t(\omega)) \in C^0([0,T];H)$ is well-defined in the Bochner space $L^p(\Omega,\mathcal{F},\mathbb{P};C^0([0,T];H))$, where the Banach space $(C^0([0,T];H),\Vert \cdot \Vert_{C^0([0,T];H)})$ is by Lemma~\ref{LemmaC0HSep} separable.

\subsection{Wiener chaos expansion}
\label{SecWienerChaos}

In this section, we recall the Wiener chaos expansion of the solution to \eqref{EqDefSPDE}. To this end, we fix throughout the paper a set of a.e.~continuous functions $(g_j)_{j \in \mathbb{N}}$ that form a complete orthonormal basis of the Hilbert space $(L^2([0,T],\mathcal{B}([0,T]),dt),\langle \cdot, \cdot \rangle_{L^2([0,T],\mathcal{B}([0,T]),dt)})$. Moreover, we define for every $i,j \in \mathbb{N}$ the random variable
\begin{equation}
	\label{EqDefXi}
	\xi_{i,j} := \int_0^T g_j(t) dW^{(i)}_t.
\end{equation}
Then, by using Lemma~\ref{LemmaBM} and Ito's isometry, the random variables $(\xi_{i,j})_{i,j \in \mathbb{N}} \sim \mathcal{N}(0,1)$ are Gaussian with $\mathbb{E}[\xi_{i_1,j_1}] = 0$ and $\mathbb{E}[\xi_{i_1,j_1} \xi_{i_2,j_2}] = \delta_{i_1,i_2} \delta_{j_1,j_2}$ for all $i_1,i_2,j_1,j_2 \in \mathbb{N}$, where we define $\delta_{k_1,k_2} := 1$ if $k_1 = k_2$, and $\delta_{k_1,k_2} := 0$ otherwise, for $k_1,k_2 \in \mathbb{N}$. Conversely, the following result shows how the Brownian motions can be recovered, whose proof is given in Section~\ref{SecAuxProofsSPDE}.

\begin{lemma}
	\label{LemmaBMFourier}
	Let $t \in [0,T]$ and $i \in \mathbb{N}$. Then, the following holds true:
	\begin{enumerate}
		\item\label{LemmaBMFourier1} $W^{(i)}_t = \sum_{j=1}^\infty \xi_{i,j} \int_0^t g_j(s) ds$, where the sum converges with respect to $\Vert \cdot \Vert_{L^2(\Omega,\mathcal{F},\mathbb{P})}$.
		\item\label{LemmaBMFourier2} $W_t = \sum_{i,j=1}^\infty \sqrt{\lambda_i} \xi_{i,j} \big( \int_0^t g_j(s) ds \big) e_i$, where the sum converges with respect to $\Vert \cdot \Vert_{L^2(\Omega,\mathcal{F},\mathbb{P};Z)}$.
	\end{enumerate}
\end{lemma}

\begin{example}
	\label{ExFourierBM}
	For example, we could choose the basis functions $(g_j)_{j \in \mathbb{N}}$ given by $[0,T] \ni t \mapsto g_1(t) := \sqrt{1/T} \in \mathbb{R}$ and $[0,T] \ni t \mapsto g_j(t) := \sqrt{2/T} \cos\big( \frac{(j-1) \pi t}{T} \big) \in \mathbb{R}$ for $j \in \mathbb{N} \cap [2,\infty)$, which yield the Fourier representation of each real-valued Brownian motion $(W^{(i)})_{i \in \mathbb{N}}$. On the other hand, the Haar wavelets give the Levy-Ciesielski construction of $(W^{(i)})_{i \in \mathbb{N}}$ (see \cite[Section~2.3]{karatzas98}). 
\end{example}

Next, we introduce the Wick polynomials associated to the $Q$-Brownian motion $W: [0,T] \times \Omega \rightarrow Z$. To this end, we recall that the Hermite polynomials $(h_n)_{n \in \mathbb{N}_0}$ are for every $n \in \mathbb{N}_0$ defined as
\begin{equation}
	\label{EqDefHermite}
	\mathbb{R} \ni s \quad \mapsto \quad h_n(s) := (-1)^n e^\frac{s^2}{2} \frac{d^n}{ds^n} \Big( e^{-\frac{s^2}{2}} \Big) \in \mathbb{R}.
\end{equation}
Moreover, we consider the set of infinite matrices with finitely many non-zero positive integers $\mathcal{J} := \big\lbrace \alpha := (\alpha_{i,j})_{i,j \in \mathbb{N}} \in \mathbb{N}_0^{\mathbb{N} \times \mathbb{N}}: \vert \alpha \vert := \sum_{i,j=1}^\infty \alpha_{i,j} < \infty \big\rbrace$ and define $\alpha! := \prod_{i,j=1}^\infty \alpha_{i,j}!$ for $\alpha \in \mathcal{J}$.

\begin{definition}
	\label{DefWick}
	The \emph{Wick polynomials} $(\xi_\alpha)_{\alpha \in \mathcal{J}}$ are for every $\alpha \in \mathcal{J}$ defined by
	\vspace{-0.05cm}
	\begin{equation}
		\label{EqDefWick}
		\Omega \ni \omega \quad \mapsto \quad \xi_\alpha(\omega) := \frac{1}{\sqrt{\alpha!}} \prod_{i,j=1}^\infty h_{\alpha_{i,j}}\big( \xi_{i,j}(\omega) \big) \in \mathbb{R}.
	\end{equation}
\end{definition}

\begin{remark}
	Since $\vert \alpha \vert < \infty$ and $h_0(s) = 1$ for all $s \in \mathbb{R}$, the product in \eqref{EqDefWick} consists only of finitely many factors that are not equal to one. Hence, the Wick polynomials $(\xi_\alpha)_{\alpha \in \mathcal{J}}$ are well-defined.
\end{remark}

Note that the Wick polynomials form a complete orthonormal basis for $L^2(\Omega,\mathcal{F}_T,\mathbb{P})$. In particular, for every $\alpha,\beta \in \mathcal{J}$, the orthogonality relation $\mathbb{E}[\xi_\alpha \xi_\beta] = \delta_{\alpha,\beta}$ holds true (see \cite[Proposition~1.1.1]{nualart06}). In the following, we extend this result towards $L^p$-spaces, whose proof can be found in Section~\ref{SecAuxProofsSPDE}.

\begin{lemma}
	\label{LemmaWick}
	For every $p \in [1,\infty)$ the linear span of $(\xi_\alpha)_{\alpha \in \mathcal{J}}$ is dense in $L^p(\Omega,\mathcal{F}_T,\mathbb{P})$. In particular, for $p = 2$, the Wick polynomials $(\xi_\alpha)_{\alpha \in \mathcal{J}}$ form a complete orthonormal basis of $L^2(\Omega,\mathcal{F}_T,\mathbb{P})$.
\end{lemma}

For $p = 2$, Cameron and Martin proved in \cite[Theorem~1]{cameron47} that every non-linear square-integrable functional of the Brownian motion can be represented as infinite sum of Wick polynomials. In the following, we extend this result to the $L^p$-case. To this end, we define for every $I,J,K \in \mathbb{N}$ the truncated set of indices $\mathcal{J}_{I,J,K} := \big\lbrace \alpha := (\alpha_{i,j})_{i,j \in \mathbb{N}} \in \mathbb{N}_0^{\mathbb{N} \times \mathbb{N}}: \alpha_{i,j} = 0 \text{ if } i > I \text{ or } j > J, \, \vert \alpha \vert \leq K \big\rbrace \subset \mathcal{J}$, which satisfies $\vert \mathcal{J}_{I,J,K} \vert = \sum_{k=0}^K \binom{IJ+k-1}{k} = \frac{(IJ+K)!}{(IJ)!K!}$ (see \cite[p.~38]{luo06}). The proof can be found in Section~\ref{SecProofsCameronMartin}.

\begin{theorem}[Cameron-Martin]
	\label{ThmCameronMartin}
	Let $p \in [1,\infty)$, let Assumption~\ref{AssSPDE} hold, and let $X: [0,T] \times \Omega \rightarrow H$ be a mild solution of \eqref{EqDefSPDE}. Then, for every $\varepsilon > 0$ there exist some $I,J,K \in \mathbb{N}$ and some functions $(x_\alpha)_{\alpha \in \mathcal{J}_{I,J,K}} \subseteq C^0([0,T];H)$ called the \emph{propagators (of $X$)} such that
	\vspace{-0.05cm}
	\begin{equation}
		\label{EqThmCameronMartin1}
		\mathbb{E}\left[ \sup_{t \in [0,T]} \left\Vert X_t - \sum_{\alpha \in \mathcal{J}_{I,J,K}} x_\alpha(t) \xi_\alpha \right\Vert_H^p \right]^\frac{1}{p} < \varepsilon.
	\end{equation}
	In particular, for $p = 2$, the propagators $(x_\alpha)_{\alpha \in \mathcal{J}} \subseteq C^0([0,T];H)$ defined by $[0,T] \ni t \mapsto x_\alpha(t) := \mathbb{E}[X_t \xi_\alpha] \in H$, $\alpha \in \mathcal{J}$, satisfy for every $t \in [0,T]$ that
	\begin{equation}
		\label{EqThmCameronMartin2}
		X_t = \sum_{\alpha \in \mathcal{J}} x_\alpha(t) \xi_\alpha,
	\end{equation}
	where the sum in \eqref{EqThmCameronMartin2} converges with respect to $\Vert \cdot \Vert_{L^2(\Omega,\mathcal{F},\mathbb{P};H)}$.
\end{theorem}

\begin{remark}
	\label{RemGaussian}
	Since $\xi_0 = 1$ and $(\xi_\alpha)_{\alpha \in \mathcal{J}, \, \vert \alpha \vert = 1}$ are Gaussian random variables (as $h_1(s) = s$ for all $s \in \mathbb{R}$), we can decompose the expansion \eqref{EqThmCameronMartin1} into three different parts, i.e.~for every $t \in [0,T]$ we have
	\begin{equation*}
		X_t \approx \underbrace{x_0(t) \xi_0}_{\text{constant approximation}} + \underbrace{\sum_{\alpha \in \mathcal{J}_{I,J,K}, \, \vert \alpha \vert = 1} x_\alpha(t) \xi_\alpha}_{\text{Gaussian approximation}} + \underbrace{\sum_{\alpha \in \mathcal{J}_{I,J,K}, \, \vert \alpha \vert \geq 2} x_\alpha(t) \xi_\alpha}_{\text{non-Gaussian approximation}}.
	\end{equation*}
	In numerical examples, one can then analyze when Gaussian approximation is sufficient ($K = 1$), or when higher order (non-Gaussian) terms are necessary ($K > 1$).
\end{remark}

\begin{remark}
	If $(H,\langle \cdot, \cdot \rangle_H)$ is a function space and \eqref{EqDefSPDE} is of particular form, the functions $\big( (t,u) \mapsto x_\alpha(t)(u) \big)_{\alpha \in \mathcal{J}}$ satisfy a system of coupled PDEs (see e.g.~\cite{lototsky97,mikulevicius98,lototsky06,lototsky07,kalpinelli11}), which can be solved with traditional numerical schemes, e.g.~Fourier methods in \cite{luo06,hou06}. In this paper, we also assume that $(H,\langle \cdot, \cdot \rangle_H)$ is a function space, but learn the propagators by (possibly random) neural networks.
\end{remark}

\newpage

\section{Universal approximation of SPDEs}
\label{SecUAT}

In this section, we approximate the solution of \eqref{EqDefSPDE} by using (possibly random) neural networks in the truncated chaos expansion \eqref{EqThmCameronMartin1}. To this end, we assume that $(H,\langle \cdot, \cdot \rangle_H)$ is a function space.

\begin{assumption}
	\label{AssHilbert}
	For $k \in \mathbb{N}_0$, $U \subseteq \mathbb{R}^m$ (open, if $k \geq 1$), and $\gamma \in (0,\infty)$, let $(H,\langle \cdot, \cdot \rangle_H)$ be a Hilbert space consisting of functions $f: U \rightarrow \mathbb{R}^d$ such that the restriction map
	\begin{equation}
		\label{EqAssHilbert1}
		(C^k_b(\mathbb{R}^m;\mathbb{R}^d),\Vert \cdot \Vert_{C^k_{pol,\gamma}(\mathbb{R}^m;\mathbb{R}^d)}) \ni f \quad \mapsto \quad f\vert_U \in (H,\Vert \cdot \Vert_H)
	\end{equation}
	is a continuous dense embedding, i.e.~\eqref{EqAssHilbert1} is continuous and its image is dense in $(H,\Vert \cdot \Vert_H)$.
\end{assumption}

\begin{lemma}[{\cite[Lemma~4.1~(ii)]{neufeld24}}]
	\label{LemmaSep}
	Let $(H,\langle \cdot, \cdot \rangle_H)$ be a Hilbert space satisfying Assumption~\ref{AssHilbert}. Then, $(H,\Vert \cdot \Vert_H)$ is separable.
\end{lemma}

Let us give some examples of function spaces that are Hilbert spaces satisfying Assumption~\ref{AssHilbert}. For $k \in \mathbb{N}$, $U \subseteq \mathbb{R}^m$ open, and a strictly positive function $w: U \rightarrow (0,\infty)$ with\footnote{For $U \subseteq \mathbb{R}^m$ open, $L^1_{loc}(U,\mathcal{L}(U),du)$ denotes the vector space of locally Lebesgue-integrable functions $f: U \rightarrow \mathbb{R}$.} $w^{-1} \in L^1_{loc}(U,\mathcal{L}(U),du)$, we introduce the (weighted) Sobolev space $W^{k,2}(U,\mathcal{L}(U),w;\mathbb{R}^d)$ consisting of (equivalence classes of) $k$-times weakly differentiable functions $f: U \rightarrow \mathbb{R}^d$ satisfying $\partial_\beta f \in L^2(U,\mathcal{L}(U),w(u) du;\mathbb{R}^d)$ for all $\beta \in \mathbb{N}^m_{0,k}$ (see \cite{kufner84}). Note that for bounded $U \subseteq \mathbb{R}^m$ and $U \ni u \mapsto w(u) := 1 \in (0,\infty)$, we obtain the classical Sobolev space $W^{k,2}(U,\mathcal{L}(U),du;\mathbb{R}^d) := W^{k,2}(U,\mathcal{L}(U),w;\mathbb{R}^d)$, see \cite[Chapter~3]{adams75}.

\begin{example}[{\cite[Example~2.6]{neufeld24}}]
	\label{ExHilbert}
	The following Hilbert spaces $(H,\langle \cdot, \cdot \rangle_H)$ satisfy Assumption~\ref{AssHilbert}:
	\begin{enumerate}
		\item\label{ExHilbertL2} $H := L^2(U,\mathcal{B}(U),\mu;\mathbb{R}^d)$ with $U \subseteq \mathbb{R}^m$ and Borel measure $\mu: \mathcal{B}(U) \rightarrow [0,\infty]$ such that $\int_U (1+\Vert u \Vert)^{2\gamma} \mu(du) < \infty$, where $\langle f,g \rangle_{L^2(U,\mathcal{B}(U),\mu;\mathbb{R}^d)} := \int_U f(u)^\top g(u) \mu(du)$.
		
		\item\label{ExHilbertHkw} $H := W^{k,2}(U,\mathcal{L}(U),w;\mathbb{R}^d)$ with $k \in \mathbb{N}$, open subset $U \subseteq \mathbb{R}^m$ having the segment property\footnote{An open subset $U \subseteq \mathbb{R}^m$ has the \emph{segment property} if for every $u \in \partial U := \overline{U} \setminus U$ there exists some $V \subseteq \mathbb{R}^m$ with $u \in \partial U$ and some $y \in \mathbb{R}^m \setminus \lbrace 0 \rbrace$ such that for every $z \in \overline{U} \cap V$ and $t \in (0,1)$ it holds that $z + t y \in U$ (see \cite[p.~54]{adams75}).}, and strictly positive bounded $w: U \rightarrow (0,\infty)$ with $w^{-1} \in L^1_{loc}(U,\mathcal{L}(U),du)$ such that $\int_U (1+\Vert u \Vert)^{2\gamma} w(u) du < \infty$, where $\langle f, g \rangle_{W^{k,2}(U,\mathcal{L}(U),w;\mathbb{R}^d)} := \sum_{\beta \in \mathbb{N}^m_{0,k}} \int_U \partial_\beta f(u)^\top \partial_\beta g(u) w(u) du$.
		
		\item\label{ExHilbertHk0} $H := W^{0,2}(U,\mathcal{L}(U),w;\mathbb{R}^d) := L^2(U,\mathcal{L}(U),w;\mathbb{R}^d) := L^2(U,\mathcal{L}(U),w(u) du;\mathbb{R}^d)$ with $U \in \mathcal{B}(\mathbb{R}^m)$ and $\mathcal{L}(U)/\mathcal{B}(\mathbb{R})$-measurable function $w: U \rightarrow (0,\infty)$ such that $\int_U (1+\Vert u \Vert)^{2\gamma} w(u) du < \infty$, where $\langle f,g \rangle_{W^{0,2}(U,\mathcal{L}(U),w;\mathbb{R}^d)} := \int_U f(u)^\top g(u) w(u) du$.
	\end{enumerate}
\end{example}

Now, we first use deterministic neural networks to approximate the propagators $(x_\alpha)_{\alpha \in \mathcal{J}}$ in the truncated chaos expansion \eqref{EqThmCameronMartin1}, followed in the subsequent section by random neural networks. This provides us with a universal approximation result to learn the solution of \eqref{EqDefSPDE}.

\subsection{Deterministic neural networks}
\label{SecNN}

We now recall deterministic neural networks, which are in this paper defined as single-hidden-layer feed-forward neural networks, where all the parameters are trained.

\begin{definition}
	\label{DefNN}
	For $T > 0$ and $U \subseteq \mathbb{R}^m$, a \emph{(time-extended) deterministic neural network} is of the form
	\begin{equation}
		\label{EqDefNN}
		[0,T] \times U \ni (t,u) \quad \mapsto \quad \varphi(t,u) := \sum_{n=1}^N y_n \rho\left( a_{0,n} t + a_{1,n}^\top u - b_n \right) \in \mathbb{R}^d
	\end{equation}
	for some $N \in \mathbb{N}$ denoting the number of neurons and some $\rho \in \overline{C^k_b(\mathbb{R})}^\gamma$ representing the activation function. The parameters of \eqref{EqDefNN} consist of the \emph{weights} $a_{0,1},...,a_{0,N} \in \mathbb{R}$ and $a_{1,1},...,a_{1,N} \in \mathbb{R}^m$, the \emph{biases} $b_1,...,b_N \in \mathbb{R}$, and the \emph{linear readouts} $y_1,...,y_N \in \mathbb{R}^d$. 
\end{definition}

\begin{remark}
	\label{RemNN}
	For $\rho \in \overline{C^k_b(\mathbb{R})}^\gamma$, we denote by $\mathcal{NN}^\rho_{[0,T] \times U,d}$ the set of all deterministic neural networks of the form \eqref{EqDefNN}. Note that Lemma~\ref{LemmaNNWellDef} shows for every $\varphi \in \mathcal{NN}^\rho_{[0,T] \times U,d}$ that the map $(t \mapsto \varphi(t,\cdot)) \in C^0([0,T];H)$ is well-defined. Moreover, we denote by $\mathcal{NN}^\rho_{U,d}$ the set of all deterministic neural networks of the form \eqref{EqDefNN}, but without time-dependent part, i.e.~$U \ni u \mapsto \varphi(u) := \sum_{n=1}^N y_n \rho\left( a_{1,n}^\top u_n - b_n \right) \in \mathbb{R}^d$.
\end{remark}

Then, by combining Theorem~\ref{ThmCameronMartin} with the universal approximation property of deterministic neural networks (see e.g.~\cite[Theorem~2.8]{neufeld24}), we can approximate the propagators $(x_\alpha)_{\alpha \in \mathcal{J}_{I,J,K}} \subseteq C^0([0,T];H)$ in the truncated chaos expansion \eqref{EqThmCameronMartin1} by (time-extended) deterministic neural networks, which leads to the following universal approximation result for solutions of \eqref{EqDefSPDE}. The proof is given in Section~\ref{SecProofUAT}.

\begin{theorem}[Universal approximation]
	\label{ThmUAT}
	Let $(H,\langle \cdot, \cdot \rangle_H)$ satisfy Assumption~\ref{AssHilbert}, let $\rho \in \overline{C^k_b(\mathbb{R})}^\gamma$ be non-polynomial, and let $p \in [1,\infty)$. Moreover, let Assumption~\ref{AssSPDE} hold and let $X: [0,T] \times \Omega \rightarrow H$ be a mild solution of \eqref{EqDefSPDE}. Then, for every $\varepsilon > 0$ there exist $I,J,K \in \mathbb{N}$ and $(\varphi_\alpha)_{\alpha \in \mathcal{J}_{I,J,K}} \subseteq \mathcal{NN}^\rho_{[0,T] \times U,d}$ such that
	\vspace{-0.2cm}
	\begin{equation*}
		\mathbb{E}\left[ \sup_{t \in [0,T]} \left\Vert X_t - \sum_{\alpha \in \mathcal{J}_{I,J,K}} \varphi_\alpha(t,\cdot) \xi_\alpha \right\Vert_H^p \right]^\frac{1}{p} < \varepsilon.
	\end{equation*}
\end{theorem}

Theorem~\ref{ThmUAT} shows that the solution of \eqref{EqDefSPDE} can be learned by first truncating the Wiener chaos expansion in Theorem~\ref{ThmCameronMartin} and then replacing the propagators with deterministic neural networks.

\subsection{Random neural networks}
\label{SecRN}

To reduce the computational complexity, we also consider random neural networks which are defined as single-hidden-layer neural networks whose parameters inside the activation function are randomly initialized, whence only the linear readout needs to be trained (see \cite{gonon21,neufeld23}). For the random initialization, we assume that $(\Omega,\mathcal{F},\mathbb{P})$ also supports an independent and identically distributed (i.i.d.)~sequence $(A_{0,n}, A_{1,n}, B_n)_{n \in \mathbb{N}}: \Omega \rightarrow \mathbb{R} \times \mathbb{R}^m \times \mathbb{R}$ satisfying the following.

\begin{assumption}
	\label{AssCDF}
	Let $(A_{0,n}, A_{1,n}, B_n)_{n \in \mathbb{N}}: \Omega \rightarrow \mathbb{R} \times \mathbb{R}^m \times \mathbb{R}$ be an i.i.d.~sequence of random variables, being independent of $(W_t)_{t \in [0,T]}$, such that for every $(a_0,a_1,b) \in \mathbb{R} \times \mathbb{R}^m \times \mathbb{R}$ and $\varepsilon > 0$ we have
	\begin{equation*}
		\mathbb{P}\left[ \left\lbrace \omega \in \Omega: \left\Vert \left( A_{0,1}(\omega), A_{1,1}(\omega), B_1(\omega) \right)-(a_0,a_1,b) \right\Vert < \varepsilon \right\rbrace \right] > 0.
	\end{equation*}
\end{assumption}

Now, we introduce random neural networks whose parameters inside the activation function are taken from these random variables. In addition, the linear readout is also a random variable, but measurable with respect to these random initializations, i.e.~with respect to $\mathcal{F}_{A,B} := \sigma\left( \left\lbrace A_{0,n}, A_{1,n}, B_n: n \in \mathbb{N} \right\rbrace \right) \subseteq \mathcal{F}$.

\begin{definition}
	\label{DefRN}
	Let $(H,\langle\cdot,\cdot\rangle_H)$ satisfy Assumption~\ref{AssHilbert}. Then, a \emph{(time-extended) random neural network} is of the form
	\vspace{-0.1cm}
	\begin{equation}
		\label{EqDefRN}
		\Omega \ni \omega \quad \mapsto \quad \left( (t,u) \mapsto \Phi(\omega) := \sum_{n=1}^N Y_n(\omega) \rho\left( A_{0,n}(\omega) t + A_{1,n}(\omega)^\top u - B_n(\omega) \right) \right) \in C^0([0,T];H)
	\end{equation}
	for some $N \in \mathbb{N}$ denoting the number of neurons and some $\rho \in \overline{C^k_b(\mathbb{R})}^\gamma$ representing the activation function. Hereby, the i.i.d.~random variables $(A_{0,n})_{n=1,...,N}$ and $(A_{1,n})_{n=1,...,N}$ are the \emph{random weights}, and the i.i.d.~random variables $(B_n)_{n=1,...,N}$ are the \emph{random biases}. Moreover, the $\mathcal{F}_{A,B}/\mathcal{B}(\mathbb{R}^d)$-measurable random variables $Y_1,...,Y_N: \Omega \rightarrow \mathbb{R}^d$ are the \emph{linear readouts}. 
\end{definition}

\begin{remark}
	\label{RemRN}
	For $(H,\langle\cdot,\cdot\rangle_H)$ satisfying Assumption~\ref{AssHilbert} and $\rho \in \overline{C^k_b(\mathbb{R})}^\gamma$, we denote by $\mathcal{RN}^\rho_{[0,T] \times U,d}$ the set of all random neural networks of the form \eqref{EqDefRN}. Moreover, we denote by $\mathcal{RN}^\rho_{U,d}$ the set of all random neural networks of the form $\Omega \ni \omega \mapsto \Phi(\omega) := \big( u \mapsto \sum_{n=1}^N Y_n(\omega) \rho\left( A_{1,n}(\omega)^\top u - B_n(\omega) \right) \big) \in H$.
\end{remark}

\begin{remark}
	For the implementation of $\Phi \in \mathcal{RN}^\rho_{U,d}$ of the form \eqref{EqDefRN}, we initialize the random variables $(A_{0,n}, A_{1,n}, B_n)_{n=1,...,N}$, i.e.~we draw some $\omega \in \Omega$ and fix the values $(A_{0,n}(\omega), A_{1,n}(\omega), B_n(\omega))_{n=1,...,N}$. Since $Y_1,...,Y_N: \Omega \rightarrow \mathbb{R}^d$ are $\mathcal{F}_{A,B}/\mathcal{B}(\mathbb{R}^d)$-measurable, the training of \eqref{EqDefRN} consists of finding the optimal vectors $Y_1(\omega),...,Y_N(\omega) \in \mathbb{R}^d$ given $(A_{0,n}(\omega), A_{1,n}(\omega), B_n(\omega))_{n=1,...,N} \subseteq \mathbb{R} \times \mathbb{R}^m \times \mathbb{R}$.
\end{remark}

Then, by combining Theorem~\ref{ThmCameronMartin} with the universal approximation property of random neural networks (see \cite[Corollary~3.8]{neufeld23}), we can approximate the propagators $(x_\alpha)_{\alpha \in \mathcal{J}_{I,J,K}} \subseteq C^0([0,T];H)$ in the truncated chaos expansion \eqref{EqThmCameronMartin1} by (time-extended) random neural networks, leading to the following universal approximation result for solutions of \eqref{EqDefSPDE}. The proof can be found in Section~\ref{SecProofRUAT}.

\begin{theorem}[Universal approximation]
	\label{ThmRUAT}
	Let $(H,\langle \cdot, \cdot \rangle_H)$ satisfy Assumption~\ref{AssHilbert}, let $\rho \in \overline{C^k_b(\mathbb{R})}^\gamma$ be non-polynomial, and let $p \in [1,\infty)$. Moreover, let Assumption~\ref{AssSPDE}+\ref{AssCDF} hold and let $X: [0,T] \times \Omega \rightarrow H$ be a mild solution of \eqref{EqDefSPDE}. Then, for every $\varepsilon > 0$ there exist $I,J,K \in \mathbb{N}$ and $(\Phi_\alpha)_{\alpha \in \mathcal{J}_{I,J,K}} \subseteq \mathcal{RN}^\rho_{[0,T] \times U,d}$ such that
	\vspace{-0.2cm}
	\begin{equation*}
		\mathbb{E}\left[ \sup_{t \in [0,T]} \left\Vert X_t - \sum_{\alpha \in \mathcal{J}_{I,J,K}} \Phi_\alpha(t,\cdot) \xi_\alpha \right\Vert_H^p \right]^\frac{1}{p} < \varepsilon.
	\end{equation*}
\end{theorem}

Theorem~\ref{ThmRUAT} is the analogue of Theorem~\ref{ThmUAT} with \textit{random} neural networks instead of deterministic ones. Since only the linear readout needs to be trained, random neural networks outperform deterministic ones in terms of computational complexity (see also Section~\ref{SecNumerics} for numerical examples).

\section{Approximation rates for certain SPDEs at terminal time}
\label{SecAR}

In this section, we provide some approximation rates to learn the solution of \eqref{EqDefSPDE} at terminal time $T > 0$ by using deterministic and random neural networks in its chaos expansion. To this end, we assume that the Hilbert space $(H,\langle \cdot,\cdot \rangle_H)$ is a weighted Sobolev space as in Example~\ref{ExHilbert}~\ref{ExHilbertHkw}. 

\begin{assumption}
	\label{AssSob}
	For $k \in \mathbb{N}_0$, $U \subseteq \mathbb{R}^m$ (open, if $k \geq 1$), and $\gamma \in [0,\infty)$, let $H := W^{k,2}(U,\mathcal{L}(U),w;\mathbb{R}^d)$ be as in Example~\ref{ExHilbert}~\ref{ExHilbertHkw}+\ref{ExHilbertHk0} with constant $C^{(\gamma)}_{U,w} := \big( \int_U (1+\Vert u \Vert)^{2\gamma} w(u) du \big)^{1/2} < \infty$.
\end{assumption}

Moreover, we assume that \eqref{EqDefSPDE} is linear with affine coefficients $F: [0,T] \times \Omega \times H \rightarrow H$ and $B: [0,T] \times \Omega \times H \rightarrow L_2(Z_0;H)$, allowing us to consider SPDEs with additive and multiplicative noise.

\begin{assumption}
	\label{AssSPDEAff}
	Let the following hold true:
	\begin{enumerate}
		\item\label{AssSPDEAff1} The map $A: \dom(A) \subseteq H \rightarrow H$ is the generator of a $C_0$-semigroup (see \cite[Appendix~A.2]{daprato14}).
		\item\label{AssSPDEAff2} The map $F$ is of the form $[0,T] \times \Omega \times H \ni (t,\omega,x) \mapsto F(t,\omega,x) := f_0(t) + f_1(t) x \in H$ for some $\mathcal{B}([0,T])/\mathcal{B}(H)$-measurable map $f_0: [0,T] \rightarrow H$ and $\mathcal{B}([0,T])/\mathcal{B}(L(H;H))$-measurable map $f_1: [0,T] \rightarrow L(H;H)$.
		\item\label{AssSPDEAff3} The map $B$ is of the form $[0,T] \times \Omega \times H \ni (t,\omega,x) \mapsto B(t,\omega,x) := b_0(t) + b_1(t) x \in L_2(Z_0;H)$ for some $\mathcal{B}([0,T])/\mathcal{B}(L_2(Z_0;H))$-measurable map $b_0: [0,T] \rightarrow L_2(Z_0;H)$ and $\mathcal{B}([0,T])/\mathcal{B}(L(H;L_2(Z_0;H)))$-measurable map $b_1: [0,T] \rightarrow L(H;L_2(Z_0;H))$.
		\item\label{AssSPDEAff4} There exists a constant $C_{F,B} > 0$ such that for every $t \in [0,T]$ it holds that
		\begin{equation*}
			\begin{aligned}
				\quad\quad\quad \Vert f_0(t) \Vert_H + \Vert f_1(t) \Vert_{L(H;H)} + \sqrt{C_\lambda} \sup_{z \in Z_0 \atop \Vert z \Vert_Z \leq 1} \Vert b_0(t) z \Vert_{H} + \sqrt{C_\lambda} \sup_{x \in H \atop \Vert x \Vert_H \leq 1} \sup_{z \in Z_0 \atop \Vert z \Vert_Z \leq 1} \Vert (b_1(t) x) z \Vert_H & \leq C_{F,B},
			\end{aligned}
		\end{equation*}
		where $C_\lambda := \sum_{i=1}^\infty \lambda_i < \infty$ (cf.~\eqref{EqDefBMi}).
		\item\label{AssSPDEAff5} The initial condition $\chi_0 \in H$ is deterministic.
	\end{enumerate}
\end{assumption}

\begin{remark}
	\label{RemSPDE}
	Note that Assumption~\ref{AssSPDEAff} implies Assumption~\ref{AssSPDE} with the same constant $C_{F,B} > 0$. Hence, by using Proposition~\ref{PropSPDE}, there exists a unique mild solution $X: [0,T] \times \Omega \rightarrow H$ of \eqref{EqDefSPDE}.
\end{remark}

\footnotetext{\label{FootnoteSchwartz}$\mathcal{S}(\mathbb{R};\mathbb{C})$ consists of smooth functions $\psi: \mathbb{R} \rightarrow \mathbb{C}$ satisfying $\max_{j=0,...,n} \sup_{s \in \mathbb{R}} \left( 1+\vert s \vert^2 \right)^n \big\vert \psi^{(j)}(s) \big\vert < \infty$ for all $n \in \mathbb{N}$. Moreover, its dual space $\mathcal{S}'(\mathbb{R};\mathbb{C})$ consists of tempered distributions $T: \mathcal{S}(\mathbb{R};\mathbb{C}) \rightarrow \mathbb{C}$.}

In order to obtain the approximation rates, we first show that $X_T$ is infinitely many times Malliavin differentiable (see Proposition~\ref{PropMall}) and apply the Stroock-Taylor formula (see Proposition~\ref{PropStroock}) to upper bound the contributions of the higher order terms $(\xi_\alpha x_\alpha(T))_{\alpha \in \mathcal{J}, \, \vert \alpha \vert > K}$ in the chaos expansion \eqref{EqThmCameronMartin2} beyond some given order $K \in \mathbb{N}$. This approach also allows us to estimate the approximation error by using only $I \in \mathbb{N}$ real-valued Brownian motions instead of $W: [0,T] \times \Omega \rightarrow Z$ (cf.~\eqref{EqDefBMi}) and by using only $J \in \mathbb{N}$ basis functions $(g_j)_{j=1,...,J}$ of the orthonormal basis of $L^2([0,T],\mathcal{B}([0,T]),dt)$ introduced in Section~\ref{SecWienerChaos}. For the latter, we additionally need the following assumption.

\begin{assumption}
	\label{AssL2Basis}
	For the orthonormal basis $(g_j)_{j \in \mathbb{N}}$ of $(L^2([0,T],\mathcal{B}([0,T]),dt),\langle\cdot,\cdot\rangle_{L^2([0,T],\mathcal{B}([0,T]),dt)})$ introduced in Section~\ref{SecWienerChaos}, we assume that $C_g := \sum_{j=1}^\infty \Vert g_j \Vert_{L^1([0,T],\mathcal{B}([0,T]),dt)}^2 < \infty$.
\end{assumption}

To approximate the propagators $(x_\alpha^{(T)})_{\alpha \in \mathcal{J}} := \big( x_\alpha(T) \big)_{\alpha \in \mathcal{J}}$ in the chaos expansion \eqref{EqThmCameronMartin2}, we use the rates for deterministic and random neural networks in \cite[Theorem~3.6]{neufeld24} and \cite[Corollary~4.20]{neufeld23}, respectively. To this end, we consider pairs $(\psi,\rho) \in \mathcal{S}_0(\mathbb{R};\mathbb{C}) \times C^k_{pol,\gamma}(\mathbb{R})$ consisting of a ridgelet function\footnote{$\mathcal{S}_0(\mathbb{R};\mathcal{C}) \subseteq \mathcal{S}(\mathbb{R};\mathbb{C})$ is defined as the vector subspace\footref{FootnoteSchwartz} of $\psi \in \mathcal{S}(\mathbb{R};\mathbb{C})$ satisfying $\int_{\mathbb{R}} s^n \psi(s) ds = 0$ for all $n \in \mathbb{N}_0$.} $\psi \in \mathcal{S}_0(\mathbb{R};\mathbb{C})$ and an activation function $\rho \in C^k_{pol,\gamma}(\mathbb{R})$ (see also \cite[Definition~5.1]{sonoda17}).

\begin{definition}
	\label{DefAdm}
	A pair $(\psi,\rho) \in \mathcal{S}_0(\mathbb{R};\mathbb{C}) \times C^k_{pol,\gamma}(\mathbb{R})$ is called \emph{$m$-admissible} if the Fourier transform $\widehat{T_\rho} \in \mathcal{S}'(\mathbb{R};\mathbb{C})$ of $\rho \in C^k_{pol,\gamma}(\mathbb{R})$ (in the sense of distribution) coincides\footnote{Since $\rho \in C^k_{pol,\gamma}(\mathbb{R})$ induces $\big( g \mapsto T_\rho(g) := \int_\mathbb{R} \rho(s) g(s) ds \big) \in \mathcal{S}'(\mathbb{R};\mathbb{C})$ (see \cite[Equation~9.26]{folland92}), the Fourier transform $\widehat{T_\rho} \in \mathcal{S}'(\mathbb{R};\mathbb{C})$ is defined as $\widehat{T_\rho}(g) := T_\rho(\widehat{g})$ for all $g \in \mathcal{S}(\mathbb{R};\mathbb{C})$. Moreover, $\widehat{T_\rho} \in \mathcal{S}'(\mathbb{R};\mathbb{C})$ \emph{coincides on $\mathbb{R} \setminus \lbrace 0 \rbrace$} with $f_{\widehat{T_\rho}} \in L^1_{loc}(\mathbb{R} \setminus \lbrace 0 \rbrace;\mathbb{C})$ if $\widehat{T_\rho}(g) = \int_{\mathbb{R} \setminus \lbrace 0 \rbrace} f_{\widehat{T_\rho}}(\xi) g(\xi) d\xi$ for all $g \in C^\infty_c(\mathbb{R} \setminus \lbrace 0 \rbrace;\mathbb{C})$, where $L^1_{loc}(\mathbb{R} \setminus \lbrace 0 \rbrace;\mathbb{C})$ is the vector space of $\mathcal{L}(\mathbb{R} \setminus \lbrace 0 \rbrace)/\mathcal{B}(\mathbb{C})$-measurable functions with $\int_K \vert f(u) \vert du < \infty$ for all compact subsets $K \subseteq \mathbb{R}$ with $K \subseteq \mathbb{R} \setminus \lbrace 0 \rbrace$.} on $\mathbb{R} \setminus \lbrace 0 \rbrace$ with a function $f_{\widehat{T_\rho}} \in L^1_{loc}(\mathbb{R} \setminus \lbrace 0 \rbrace;\mathbb{C})$ such that $C^{(\psi,\rho)}_m := (2\pi)^{m-1} \int_{\mathbb{R} \setminus \lbrace 0 \rbrace} \vert \xi \vert^{-m} \overline{\widehat{\psi}(\xi)} f_{\widehat{T_\rho}}(\xi) d\xi \in \mathbb{C} \setminus \lbrace 0 \rbrace$.
\end{definition}

\begin{remark}[{\cite[Remark~3.2]{neufeld24}}]
	If $(\psi,\rho) \in \mathcal{S}_0(\mathbb{R};\mathbb{C}) \times C^k_{pol,\gamma}(\mathbb{R})$ is $m$-admissible, then $\rho \in C^k_{pol,\gamma}(\mathbb{R})$ has to be non-polynomial.
\end{remark}

\pagebreak

Furthermore, we impose the following conditions on the propagators defined in \eqref{EqThmCameronMartin2}.

\begin{assumption}
	\label{AssPropag}
	Let Assumption~\ref{AssSob}+\ref{AssSPDEAff} hold and let $\zeta_1 > 0$. Then, we assume for every $\alpha \in \mathcal{J}$ that the function $x_\alpha^{(T)} := \mathbb{E}[X_T \xi_\alpha]: \mathbb{R}^m \rightarrow \mathbb{R}^d$ is either constant (and set $c_\alpha := 0$) or satisfies $x_\alpha^{(T)} \in L^1(\mathbb{R}^m,\mathcal{L}(\mathbb{R}^m),du;\mathbb{R}^d)$ with an $(\lceil\gamma\rceil+2)$-times differentiable Fourier transform such that
	\vspace{-0.1cm}
	\begin{equation}
		\label{EqAssPropag}
		c_\alpha := \sum_{\beta \in \mathbb{N}^m_{0,\lceil\gamma\rceil+2}} \left( \int_{\mathbb{R}^m} \Big\vert \partial_\beta \widehat{x_\alpha^{(T)}}(\zeta) \Big\vert^2 \left( 1 + \Vert \zeta/\zeta_1 \Vert^2 \right)^{2\lceil\gamma\rceil+k+\frac{m+5}{2}} d\zeta \right)^\frac{1}{2} < \infty.
	\end{equation}
\end{assumption}

\subsection{Deterministic neural networks}

Now, we present the approximation rate to learn \eqref{EqDefSPDE} with deterministic neural networks, where $\mathcal{NN}^\rho_{U,d}$ was defined in Remark~\ref{RemNN}. The proof is given in Section~\ref{SecProofDetAR}.

\begin{theorem}[Approximation rates]
	\label{ThmDetAR}
	Let $H := W^{k,2}(U,\mathcal{L}(U),w;\mathbb{R}^d)$ satisfy Assumption~\ref{AssSob} (with constant $C^{(\gamma)}_{U,w} := \big( \int_U (1+\Vert u \Vert)^{2\gamma} w(u) du \big)^{1/2} < \infty$), let Assumption~\ref{AssSPDEAff} hold (with constant $C_{F,B} > 0$), and let $X: [0,T] \times \Omega \rightarrow H$ be a mild solution of \eqref{EqDefSPDE}. Moreover, let $(\psi,\rho) \in \mathcal{S}_0(\mathbb{R};\mathbb{C}) \times C^k_{pol,\gamma}(\mathbb{R})$ be $m$-admissible with $\zeta_1 := \inf\big\lbrace \vert \zeta \vert: \zeta \in \mathbb{R}, \, \widehat{\psi}(\zeta) \neq 0 \big\rbrace > 0$. In addition, let Assumption~\ref{AssL2Basis} hold (with constant $C_g > 0$) and let Assumption~\ref{AssPropag} hold (with constants $\zeta_1 > 0$ and $(c_\alpha)_{\alpha \in \mathcal{J}} \subseteq [0,\infty)$). Then, there exist some constants\footnote{\label{FootnoteConst}The constant $C_1 > 0$ depends only on $\gamma \in [0,\infty)$ and $\psi \in \mathcal{S}_0(\mathbb{R};\mathbb{C})$, while the constant $C_2 > 0$ depends only on $C_{F,B} > 0$, $C_\lambda := \sum_{i=1}^\infty \lambda_i < \infty$ (cf.~\eqref{EqDefBMi}), $C_g > 0$, $C_S := \sup_{t \in [0,T]} \Vert S_t \Vert_{L(H;H)} < \infty$, and $T > 0$.} $C_1,C_2 > 0$ such that for every $I,J,K,N \in \mathbb{N}$ there exist some $\big( \varphi_\alpha^{(T)} \big)_{\alpha \in \mathcal{J}_{I,J,K}} \in \mathcal{NN}^\rho_{U,d}$ with $N$ neurons satisfying
	\vspace{-0.1cm}
	\begin{equation*}
		\begin{aligned}
			& \mathbb{E}\left[ \left\Vert X_T - \sum_{\alpha \in \mathcal{J}_{I,J,K}} \varphi_\alpha^{(T)} \xi_\alpha \right\Vert_H^2 \right]^\frac{1}{2} \leq C_1 \Vert \rho \Vert_{C^k_{pol,\gamma}(\mathbb{R})} \frac{C^{(\gamma)}_{U,w} m^\frac{k}{2} \pi^\frac{m+1}{4}}{\zeta_1^\frac{m}{2} \left\vert C^{(\psi,\rho)}_m \right\vert \Gamma\left( \frac{m+1}{2} \right)^\frac{1}{2}} \frac{\left( \sum_{\alpha \in \mathcal{J}_{I,J,K}} c_\alpha^2 \right)^\frac{1}{2}}{\sqrt{N}} \\
			& \quad + C_2 \left( 2 + \Vert \chi_0 \Vert_H^2 \right)^\frac{1}{2} \left( \left( \sum_{i=I+1}^\infty \lambda_i \right)^\frac{1}{2} + \left( \sum_{j=J+1}^\infty \Vert g_j \Vert_{L^1([0,T])}^2 \right)^\frac{1}{2} + \frac{\left( C_S C_{F,B} \sqrt{T} e^{C_S C_{F,B} T} \right)^{K+1}}{\sqrt{(K+1)!}} \right).
		\end{aligned}
	\end{equation*}
\end{theorem}

\subsection{Random neural networks}

Next, we present the approximation rate for learning \eqref{EqDefSPDE} using random neural networks $\mathcal{RN}^\rho_{U,d}$ (see Remark~\ref{RemRN}), whose proof can be found in Section~\ref{SecProofRandAR}.

\begin{assumption}
	\label{AssPDF}
	Let\footnote{Hereby, $A_1 \sim t_m$ has probability density function $\mathbb{R}^m \ni a \mapsto p_A(a) = \frac{\Gamma((m+1)/2)}{\pi^{(m+1)/2}} \left( 1 + \Vert a \Vert^2 \right)^{-(m+1)/2} \in (0,\infty)$, where $\Gamma$ denotes the Gamma function (see \cite[Section~6.1]{abramowitz70}).} $(A_{1,n}, B_n)_{n \in \mathbb{N}} \sim t_m \otimes t_1$ be an i.i.d.~sequence, being independent of $(W_t)_{t \in [0,T]}$.
\end{assumption}

\begin{theorem}[Approximation rates]
	\label{ThmRandAR}
	Assume the setting of Theorem~\ref{ThmDetAR} (with mild solution $X: [0,T] \times \Omega \rightarrow H := W^{k,2}(U,\mathcal{L}(U),w;\mathbb{R}^d)$ of \eqref{EqDefSPDE} satisfying Assumption~\ref{AssSob}+\ref{AssSPDEAff}+\ref{AssL2Basis}+\ref{AssPropag}). Moreover, let Assumption~\ref{AssPDF} hold and let $C_1,C_2 > 0$ be the same constants\footref{FootnoteConst} as in Theorem~\ref{ThmDetAR}. Then, for every $I,J,K,N \in \mathbb{N}$ there exist some $\big( \Phi_\alpha^{(T)} \big)_{\alpha \in \mathcal{J}_{I,J,K}} \in \mathcal{RN}^\rho_{U,d}$ with $N$ neurons satisfying
	\vspace{-0.1cm}
	\begin{equation*}
		\begin{aligned}
			& \mathbb{E}\left[ \left\Vert X_T - \sum_{\alpha \in \mathcal{J}_{I,J,K}} \Phi_\alpha^{(T)} \xi_\alpha \right\Vert_H^2 \right]^\frac{1}{2} \leq C_1 \Vert \rho \Vert_{C^k_{pol,\gamma}(\mathbb{R})} \frac{C^{(\gamma)}_{U,w} m^\frac{k}{2} \pi^\frac{m+1}{4}}{\zeta_1^\frac{m}{2} \left\vert C^{(\psi,\rho)}_m \right\vert \Gamma\left( \frac{m+1}{2} \right)^\frac{1}{2}} \frac{\left( \sum_{\alpha \in \mathcal{J}_{I,J,K}} c_\alpha^2 \right)^\frac{1}{2}}{\sqrt{N}} \\
			& \quad + C_2 \left( 2 + \Vert \chi_0 \Vert_H^2 \right)^\frac{1}{2} \left( \left( \sum_{i=I+1}^\infty \lambda_i \right)^\frac{1}{2} + \left( \sum_{j=J+1}^\infty \Vert g_j \Vert_{L^1([0,T])}^2 \right)^\frac{1}{2} + \frac{\left( C_S C_{F,B} \sqrt{T} e^{C_S C_{F,B} T} \right)^{K+1}}{\sqrt{(K+1)!}} \right).
		\end{aligned}
	\end{equation*}
\end{theorem}

Since deterministic and random neural networks admit the same rates (see \cite[Theorem~3.6]{neufeld24} and \cite[Corollary~4.20]{neufeld23}), we obtain the same approximation rate in Theorem~\ref{ThmDetAR}+\ref{ThmRandAR} for learning the solution of \eqref{EqDefSPDE}, which decays in the number of approximative Brownian motions $I$, the number of basis functions $J$ from $(g_j)_{j \in \mathbb{N}}$, the order of the chaos expansion $K$, and the number of neurons $N$.

\begin{remark}
	Let us analyze the rate in Theorem~\ref{ThmDetAR}+\ref{ThmRandAR} with respect to $J \in \mathbb{N}$. For example, for $(g_j)_{j \in \mathbb{N}}$ as in Example~\ref{ExFourierBM}, we have $\Vert g_j \Vert_{L^1([0,T],\mathcal{B}([0,T]),dt)} \leq \frac{\sqrt{2T}}{\pi(j-1)}$ for all $j \in \mathbb{N} \cap [2,\infty)$ and thus $\sum_{j=J+1}^\infty \Vert g_j \Vert_{L^1([0,T],\mathcal{B}([0,T]),dt)}^2 \leq \sum_{j=J+1}^\infty \frac{2T}{\pi^2 (j-1)^2} \leq \frac{2T}{\pi^2} \big( \frac{1}{J^2} + \int_J^\infty s^{-2} ds \big) \leq \frac{4T}{\pi^2 J}$ for all $J \in \mathbb{N}$. Hence, the rate in Theorem~\ref{ThmDetAR}+\ref{ThmRandAR} with respect to $J$ is of order $\mathcal{O}\big(1/\sqrt{J}\big)$.
\end{remark}

\begin{remark}
	\label{RemPolynCoeff}
	We could generalize the approximation results in Theorem~\ref{ThmDetAR}+\ref{ThmRandAR} to coefficients $F: [0,T] \times \Omega \times H \rightarrow H$ and $B: [0,T] \times \Omega \times H \rightarrow L_2(Z_0;H)$ of polynomial form, i.e.
	\vspace{-0.1cm}
	\begin{equation*}
		\begin{aligned}
			[0,T] \times \Omega \times H \ni (t,\omega,x) \quad \mapsto \quad F(t,\omega,x) & = f_0(t) + \sum_{l=1}^L f_l(t)(x,...,x) \in H \\
			\vspace{-0.1cm}
			[0,T] \times \Omega \times H \ni (t,\omega,x) \mapsto B(t,\omega,x) & = b_0(t) + \sum_{l=1}^L b_l(t)(x,...,x) \in L_2(Z_0;H),
		\end{aligned}
		\vspace{-0.1cm}
	\end{equation*}
	with $\mathcal{B}([0,T])/\mathcal{B}(H)$-measurable mappings $f_l: [0,T] \rightarrow L^{(l)}(H;H)$ and $\mathcal{B}([0,T])/\mathcal{B}(L_2(Z_0;H))$-measurable mappings $b_l: [0,T] \rightarrow L^{(l)}(H;L_2(Z_0;H))$, for all $l = 0,...,L$, where $L^{(l)}(H;H)$ (resp., $L^{(l)}(H;L_2(Z_0;H))$) with $L^{(0)}(H;H) := \mathbb{R}$ (resp., $L^{(0)}(H;L_2(Z_0;H)) := \mathbb{R}$) denote the space of multilinear (i.e., $l$-linear) symmetric maps from $H^l$ to $H$ (resp., $L_2(Z_0;H)$). In this case, we can follow the proofs of Theorem~\ref{ThmDetAR}+\ref{ThmRandAR} (including the proof of Proposition~\ref{PropMall}) to conclude that the solution of \eqref{EqDefSPDE} remains infinitely many times Malliavin-differentiable, while its $k$-th Malliavin derivative still grows  appropriately in $k$ as the application of the chain rule in the proof of Proposition~\ref{PropMall} takes into account that the $j$-th (Fr\'echet) derivatives of $B$ and $F$ in $x$ vanish for all $j \in \mathbb{N} \cap (L,\infty)$.
\end{remark}

\section{Numerical Experiments}
\label{SecNumerics}

In this section, we illustrate in three numerical examples how the solution $X: [0,T] \times \Omega \rightarrow H$ of \eqref{EqDefSPDE} can be learned by using (possibly random) neural networks in its Wiener chaos expansion. To this end, we assume that $(H,\langle \cdot,\cdot \rangle_H)$ is a separable Hilbert space consisting of $k$-times (weakly) differentiable functions $f: U \rightarrow \mathbb{R}^d$, where $k \in \mathbb{N}_0$ and $U \subseteq \mathbb{R}^m$ (open, if $k \geq 1$). Moreover, we fix some $I,J,K \in \mathbb{N}$ and approximate $X: [0,T] \times \Omega \rightarrow H$ by the process $X^{(I,J,K)}: [0,T] \times \Omega \rightarrow H$ defined as
\vspace{-0.1cm}
\begin{equation*}
	X^{(I,J,K)}_t(\omega) :=
	\begin{cases}
		\sum_{\alpha \in \mathcal{J}_{I,J,K}} \varphi_\alpha(t,\cdot) \xi_\alpha(\omega), & (\varphi_\alpha)_{\alpha \in \mathcal{J}_{I,J,K}} \subseteq \mathcal{NN}^\rho_{[0,T] \times U,d}, \\
		\sum_{\alpha \in \mathcal{J}_{I,J,K}} \Phi_\alpha(\omega)(t,\cdot) \xi_\alpha(\omega), & (\Phi_\alpha)_{\alpha \in \mathcal{J}_{I,J,K}} \subseteq \mathcal{RN}^\rho_{[0,T] \times U,d},
	\end{cases}
\end{equation*}
depending on whether we use deterministic neural networks or random neural networks in the chaos expansion. In this setting, we consider the two following learning frameworks.

In a \emph{supervised learning} approach, we assume that the solution $X: [0,T] \times \Omega \rightarrow H$ of \eqref{EqDefSPDE} is known at some given data points $(\omega_{m_1})_{m_1=1,...,M_1} \subseteq \Omega$, $0 \leq t_0 < t_1 < ... < t_{M_2} \leq T$, and $(u_{m_3})_{m_3=1,...,M_3} \subseteq U$. Then, we aim to minimize the empirical (weighted Sobolev) error
\vspace{-0.1cm}
\begin{equation}
	\label{EqDefLoss1}
	\left( \sum_{m_1=1}^{M_1} \sum_{m_2=0}^{M_2} \sum_{m_3=1}^{M_3} \sum_{\beta \in \mathbb{N}^m_{0,k}} \widetilde{c}_{\beta,m_1,m_2,m_3}^2 \left\vert \partial_\beta X_{t_{m_2}}(\omega_{m_1})(u_{m_3}) - \partial_\beta X^{(I,J,K)}_{t_{m_2}}(\omega_{m_1})(u_{m_3}) \right\vert^2 \right)^\frac{1}{2}
	\vspace{-0.1cm}
\end{equation}
either over $(\varphi_\alpha)_{\alpha \in \mathcal{J}_{I,J,K}} \subseteq \mathcal{NN}^\rho_{[0,T] \times U,d}$ or over $(\Phi_\alpha)_{\alpha \in \mathcal{J}_{I,J,K}} \subseteq \mathcal{RN}^\rho_{[0,T] \times U,d}$. Hereby, the constants $(\widetilde{c}_{\beta,m_1,m_2,m_3})_{\beta \in \mathbb{N}^m_{0,k}, \, m_1=1,...,M_1, \, m_2=0,...,M_2, \, m_3 = 1,...,M_3} \subseteq [0,\infty)$ control the contributions of the derivatives, e.g.~$\widetilde{c}_{\beta,m_1,m_2,m_3} = m^{-\vert\beta\vert}$ for all $(\beta,m_1,m_2,m_3) \in \mathbb{N}^m_{0,k} \times \lbrace 1,...,M_1 \rbrace \times \lbrace 0,...,M_2 \rbrace \times \lbrace 1,...,M_3 \rbrace$ means equal contribution of each order. This supervised learning approach is summarized in Algorithm~\ref{AlgSV}.

\begin{algorithm}[!ht]
	\DontPrintSemicolon
	\begin{small}
		\KwInput{$I,J,K,N \in \mathbb{N}$ and the solution $X: [0,T] \times \Omega \rightarrow H$ of \eqref{EqDefSPDE}.}
		\vspace{0.1cm}
		\KwOutput{Approximation $X^{(I,J,K)}: [0,T] \times \Omega \rightarrow H$ of the solution $X: [0,T] \times \Omega \rightarrow H$ to \eqref{EqDefSPDE}.}
		
		\vspace{0.2cm}
		
		For $M_1, M_2, M_3 \in \mathbb{N}$, let $(\omega_{m_1})_{m_1=1,...,M_1} \subseteq \Omega$, $0 \leq t_0 < t_1 < ... < t_{M_2} \leq T$, $(u_{m_3})_{m_3=1,...,M_3} \subseteq U$, and $(\widetilde{c}_{\beta,m_1,m_2,m_3})_{\beta \in \mathbb{N}^m_{0,k}, \, m_1=1,...,M_1, \, m_2=0,...,M_2, \, m_3=1,...,M_3} \subseteq [0,\infty)$.
		
		Generate $I \cdot J \cdot M_1$ realizations $\xi_{i,j}(\omega_{m_1})$, $i = 1,...,I$, $j = 1,...,J$, and $m_1 = 1,...,M_1$, of i.i.d.~$\mathcal{N}(0,1)$-random variables.
		
		For every $m_1 = 1,...,M_1$ and $\alpha \in \mathcal{J}_{I,J,K}$ compute the Wick polynomial $\xi_\alpha(\omega_{m_1}) := \frac{1}{\sqrt{\alpha!}} \prod_{i,j=1}^\infty h_{\alpha_{i,j}}(\xi_{i,j}(\omega_{m_1}))$, see Definition~\ref{DefWick}.
		
		Choose a non-polynomial activation function $\rho \in \overline{C^k_b(\mathbb{R})}^\gamma$ for some $\gamma \geq 0$.
		
		\If{\textbf{deterministic}}{			
			For every $\alpha \in \mathcal{J}_{I,J,K}$ initialize a (time-extended) deterministic neural network of the form $[0,T] \times U \ni (t,u) \mapsto \varphi_\alpha(t,u) := \sum_{n=1}^N y_n \rho\left( a_{0,n} t + a_{1,n}^\top u - b_n \right) \in \mathbb{R}^d \quad\quad\quad\quad\quad\quad\quad\quad\quad\quad$ for some $(a_{0,n},a_{1,n},b_n)_{n=1,...,N} \subseteq \mathbb{R} \times \mathbb{R}^m \times \mathbb{R}$ and $(y_n)_{n=1,...,N} \subseteq \mathbb{R}^d$.
			
			Initialize the process $[0,T] \times \Omega \ni (t,\omega) \mapsto X^{(I,J,K)}_t(\omega) := \sum_{\alpha \in \mathcal{J}_{I,J,K}} \varphi_\alpha(t,\cdot) \xi_\alpha(\omega) \in H$.
			
			Minimize \eqref{EqDefLoss1} over $(a_{0,n},a_{1,n},b_n,y_n)_{n=1,...,N} \subseteq \mathbb{R} \times \mathbb{R}^m \times \mathbb{R} \times \mathbb{R}^d$ by, e.g., using (stochastic) gradient descent algorithms (see e.g.~\cite[Section~4.3]{goodfellow16}).
		}
		\Else{			
			Generate $M_1 \cdot N$ realizations $(A_{0,n}(\omega_{m_1}),A_{1,n}(\omega_{m_1}),B_n(\omega_{m_1}))$, $m_1 = 1,...,M_1$ and $n = 1,...,N$, of i.i.d.~random variables satisfying Assumption~\ref{AssCDF}.
			
			For every $\alpha \in \mathcal{J}_{I,J,K}$ initialize a (time-extended) random neural network of the form $\Omega \ni \omega \mapsto \big( (t,u) \mapsto \Phi_\alpha(\omega)(t,u) := \sum_{n=1}^N y_n \rho\left( A_{0,n}(\omega) t + A_{1,n}(\omega)^\top u - B_n(\omega) \right) \big) \in H \quad\quad\quad$ for some $(y_n)_{n=1,...,N} \subseteq \mathbb{R}^d$.
			
			Initialize the process $[0,T] \times \Omega \ni (t,\omega) \mapsto X^{(I,J,K)}_t(\omega) := \sum_{\alpha \in \mathcal{J}_{I,J,K}} \Phi_\alpha(\omega)(t,\cdot) \xi_\alpha(\omega) \in H$.
			
			Minimize \eqref{EqDefLoss1} over $(y_n)_{n=1,...,N} \subseteq \mathbb{R}^d$ by using the least squares method (see e.g.~\cite[Section~5.1]{neufeld23}).
		}
		
		\Return{$X^{(I,J,K)}: [0,T] \times \Omega \rightarrow H$}
	\end{small}
	\caption{Supervised learning to approximate the solution of \eqref{EqDefSPDE}}
	\label{AlgSV}
\end{algorithm}

Moreover, in an \emph{unsupervised learning} approach, we assume that the initial value $\chi_0 \in H$, the numbers $(\lambda_i)_{i \in \mathbb{N}}$ in \eqref{EqDefBMi}, and the coefficients $A: \dom(A) \subseteq H \rightarrow H$, $F: [0,T] \times \Omega \times H \rightarrow H$, and $B: [0,T] \times \Omega \times H \rightarrow L_2(Z_0;H)$ are given. Then, we learn $X^{(I,J,K)}: [0,T] \times \Omega \rightarrow H)$ that approximately satisfies \eqref{EqDefSPDE}, i.e.~we minimize the empirical (weighted Sobolev) error
\begin{equation}
	\label{EqDefLoss2}
	\begin{aligned}
		& \Bigg( \sum_{m_1=1}^{M_1} \sum_{m_2=0}^{M_2} \sum_{m_3=1}^{M_3} \sum_{\beta \in \mathbb{N}^m_{0,k}} \widetilde{c}_{\beta,m_1,m_2,m_3}^2 \Bigg\vert \partial_\beta X^{(I,J,K)}_{t_{m_2}}(\omega_{m_1})(u_{m_3}) - \partial_\beta \Bigg( \chi_0(u_{m_3}) \\
		& \quad\quad\quad\quad + \sum_{l=1}^{m_2} \left( A X^{(I,J,K)}_{t_{l-1}}(\omega_{m_1}) + F\left( t_{l-1}, \omega_{m_1}, X^{(I,J,K)}_{t_{l-1}}(\omega_{m_1}) \right) \right) (t_l - t_{l-1}) \\
		& \quad\quad\quad\quad + \sum_{l=1}^{m_2} B\left( t_{l-1}, \omega_{m_1}, X^{(I,J,K)}_{t_{l-1}}(\omega_{m_1}) \right) \left( W^{(I,J)}_{t_l}(\omega_{m_1}) - W^{(I,J)}_{t_{l-1}}(\omega_{m_1}) \right) \Bigg)(u_{m_3}) \Bigg\vert^2 \Bigg)^\frac{1}{2}
	\end{aligned}
	\vspace{-0.1cm}
\end{equation}
either over $(\varphi_\alpha)_{\alpha \in \mathcal{J}_{I,J,K}} \subseteq \mathcal{NN}^\rho_{[0,T] \times U,d}$ or over $(\Phi_\alpha)_{\alpha \in \mathcal{J}_{I,J,K}} \subseteq \mathcal{RN}^\rho_{[0,T] \times U,d}$, where we use an Euler–Maruyama approximation of \eqref{EqDefSPDE}, where the data points $(\omega_{m_1})_{m_1=1,...,M_1} \subseteq \Omega$, $0 \leq t_0 < t_1 < ... < t_{M_2} \leq T$, and $(u_{m_3})_{m_3=1,...,M_3} \subseteq U$ are given, and where $\Omega \times [0,T] \ni (t,\omega) \mapsto W^{(I,J)}_t(\omega) := \sum_{i=1}^I \sum_{j=1}^J \sqrt{\lambda_i} \xi_{i,j}(\omega) \big( \int_0^t g_j(s) ds \big) e_i \in Z$ is an approximation of $W: [0,T] \times \Omega \rightarrow H$ (see Lemma~\ref{LemmaBMFourier}~\ref{LemmaBMFourier2}). This unsupervised learning approach is summarized in Algorithm~\ref{AlgUSV}.

\begin{algorithm}[!ht]
	\DontPrintSemicolon
	\begin{small}
		\KwInput{$I,J,K,N \in \mathbb{N}$, the initial value $\chi_0 \in H$, the numbers $(\lambda_i)_{i \in \mathbb{N}}$ in \eqref{EqDefBMi}, and the coefficients $A: \dom(A) \subseteq H \rightarrow H$, $F: [0,T] \times \Omega \times H \rightarrow H$, and $B: [0,T] \times \Omega \times H \rightarrow L_2(Z_0;H)$.}
		\vspace{0.1cm}
		\KwOutput{Approximation $X^{(I,J,K)}: [0,T] \times \Omega \rightarrow H$ of the solution $X: [0,T] \times \Omega \rightarrow H$ to \eqref{EqDefSPDE}.}
		
		\vspace{0.2cm}
		
		For $M_1, M_2, M_3 \in \mathbb{N}$, let $(\omega_{m_1})_{m_1=1,...,M_1} \subseteq \Omega$, $0 \leq t_0 < t_1 < ... < t_{M_2} \leq T$, $(u_{m_3})_{m_3=1,...,M_3} \subseteq U$, and $(\widetilde{c}_{\beta,m_1,m_2,m_3})_{\beta \in \mathbb{N}^m_{0,k}, \, m_1=1,...,M_1, \, m_2=0,...,M_2, \, m_3=1,...,M_3} \subseteq [0,\infty)$.
		
		Generate $I \cdot J \cdot M_1$ realizations $\xi_{i,j}(\omega_{m_1})$, $i = 1,...,I$, $j = 1,...,J$, and $m_1 = 1,...,M_1$, of i.i.d.~$\mathcal{N}(0,1)$-random variables.
		
		For every $m_1 = 1,...,M_1$ and $m_2 = 0,...,M_2$ compute the approximative Brownian motion $W^{(I,J)}_{t_{m_2}}(\omega_{m_1}) := \sum_{i=1}^I \sum_{j=1}^J \sqrt{\lambda_i} \xi_{i,j}(\omega_{m_1}) \int_0^{t_{m_2}} g_j(s) ds$, see Lemma~\ref{LemmaBMFourier}~\ref{LemmaBMFourier2}.
		
		For every $m_1 = 1,...,M_1$ and $\alpha \in \mathcal{J}_{I,J,K}$ compute the Wick polynomial $\xi_\alpha(\omega_{m_1}) := \frac{1}{\sqrt{\alpha!}} \prod_{i,j=1}^\infty h_{\alpha_{i,j}}(\xi_{i,j}(\omega_{m_1}))$, see Definition~\ref{DefWick}.
		
		Choose a non-polynomial activation function $\rho \in \overline{C^k_b(\mathbb{R})}^\gamma$ for some $\gamma \geq 0$.
		
		\If{\textbf{deterministic}}{			
			For every $\alpha \in \mathcal{J}_{I,J,K}$ initialize a (time-extended) deterministic neural network of the form $[0,T] \times U \ni (t,u) \mapsto \varphi_\alpha(t,u) := \sum_{n=1}^N y_n \rho\left( a_{0,n} t + a_{1,n}^\top u - b_n \right) \in \mathbb{R}^d \quad\quad\quad\quad\quad\quad\quad\quad\quad\quad$ for some $(a_{0,n},a_{1,n},b_n)_{n=1,...,N} \subseteq \mathbb{R} \times \mathbb{R}^m \times \mathbb{R}$ and $(y_n)_{n=1,...,N} \subseteq \mathbb{R}^d$.
			
			Initialize the process $[0,T] \times \Omega \ni (t,\omega) \mapsto X^{(I,J,K)}_t(\omega) := \sum_{\alpha \in \mathcal{J}_{I,J,K}} \varphi_\alpha(t,\cdot) \xi_\alpha(\omega) \in H$.
			
			Minimize \eqref{EqDefLoss2} over $(a_{0,n},a_{1,n},b_n,y_n)_{n=1,...,N} \subseteq \mathbb{R} \times \mathbb{R}^m \times \mathbb{R} \times \mathbb{R}^d$ by, e.g., using (stochastic) gradient descent algorithms (see \cite[Section~4.3]{goodfellow16}).
		}
		\Else{			
			Generate $M_1 \cdot N$ realizations $(A_{0,n}(\omega_{m_1}),A_{1,n}(\omega_{m_1}),B_n(\omega_{m_1}))$, $m_1 = 1,...,M_1$ and $n = 1,...,N$, of i.i.d.~random variables satisfying Assumption~\ref{AssCDF}.
			
			For every $\alpha \in \mathcal{J}_{I,J,K}$ initialize a (time-extended) random neural network of the form $\Omega \ni \omega \mapsto \big( (t,u) \mapsto \Phi_\alpha(\omega)(t,u) := \sum_{n=1}^N y_n \rho\left( A_{0,n}(\omega) t + A_{1,n}(\omega)^\top u - B_n(\omega) \right) \big) \in H \quad\quad\quad$ for some $(y_n)_{n=1,...,N} \subseteq \mathbb{R}^d$.
			
			Initialize the process $[0,T] \times \Omega \ni (t,\omega) \mapsto X^{(I,J,K)}_t(\omega) := \sum_{\alpha \in \mathcal{J}_{I,J,K}} \Phi_\alpha(\omega)(t,\cdot) \xi_\alpha(\omega) \in H$.
			
			Minimize \eqref{EqDefLoss2} over $(y_n)_{n=1,...,N} \subseteq \mathbb{R}^d$ by using the least squares method (see \cite[Section~5.1]{neufeld23}).
		}
		
		\Return{$X^{(I,J,K)}: [0,T] \times \Omega \rightarrow H$}
	\end{small}
	\caption{Unsupervised learning to approximate the solution of \eqref{EqDefSPDE}}
	\label{AlgUSV}
\end{algorithm}

In the following\footnotemark, we consider three numerical examples: the stochastic heat equation, the Heath-Jarrow-Morton (HJM) equation, and the Zakai equation. Hereby, we always choose the basis functions $(g_j)_{j \in \mathbb{N}}$ of $(L^2([0,T],\mathcal{B}([0,T]),dt),\langle\cdot,\cdot\rangle_{L^2([0,T],\mathcal{B}([0,T]),dt)})$ as in Example~\ref{ExFourierBM}.

\footnotetext{The numerical experiments have been implemented in \texttt{Python} on an average laptop (Lenovo ThinkPad X13 Gen2a with Processor AMD Ryzen 7 PRO 5850U and Radeon Graphics, 1901 Mhz, 8 Cores, 16 Logical Processors). The code can be found under the following link: \url{https://github.com/psc25/ChaosSPDE}}

\subsection{Stochastic heat equation}
\label{SecStochHeat}

In the first numerical experiment, we learn the solution of the stochastic heat equation. To this end, we consider the Hilbert space $H := L^2(\mathbb{R}^m,\mathcal{L}(\mathbb{R}^m),w)$ with inner product $\langle f,g \rangle_{L^2(\mathbb{R}^m,\mathcal{L}(\mathbb{R}^m),w)} := \int_{\mathbb{R}^m} f(u) g(u) w(u)du$, where $\mathbb{R}^m \ni u \mapsto w(u) := (2\pi)^{-m/2} \exp\left( - \Vert u \Vert^2/2 \right) \in [0,\infty)$. Moreover, we assume that \eqref{EqDefSPDE} takes the form
\begin{equation}
	\label{EqDefStochHeat}
	\begin{cases}
		dX_t & = \Delta X_t dt + b_0 dW_t, \\
		X_0 & = \chi_0 \in L^2(\mathbb{R}^m,\mathcal{L}(\mathbb{R}^m),w),
	\end{cases}
\end{equation}
where $\chi_0 \in L^2(\mathbb{R}^m,\mathcal{L}(\mathbb{R}^m),w)$ is deterministic, and where $\Delta =: A$ denotes the Laplacian
\begin{equation}
	\label{EqDefLapl}
	W^{2,2}(\mathbb{R}^m,\mathcal{L}(\mathbb{R}^m),w) \ni x \quad \mapsto \quad \Delta x := \left( u \mapsto \sum_{l=1}^m \frac{\partial^2 x}{\partial u_l^2}(u) \right) \in L^2(\mathbb{R}^m,\mathcal{L}(\mathbb{R}^m),w).
\end{equation}
Moreover, the map $\left( z \mapsto b_0(z) := z \right) \in L_2(\mathbb{R};L^2(\mathbb{R}^m,\mathcal{L}(\mathbb{R}^m),w))$ returns the constant function $z \in L^2(\mathbb{R}^m,\mathcal{L}(\mathbb{R}^m),w)$, while $W: [0,T] \times \Omega \rightarrow Z := \mathbb{R}$ is a $Q$-Brownian motion, with $Q := \id_\mathbb{R} \in L_1(\mathbb{R};\mathbb{R})$. Thus, $\lambda_1 := 1$ and $\lambda_i := 0$ for all $i \in \mathbb{N} \cap [2,\infty)$, cf.~\eqref{EqDefBMi}.

Then, \eqref{EqDefStochHeat} admits the following mild solution (see Section~\ref{SecProofsNE} for the proof).

\begin{lemma}
	\label{LemmaStochHeat}
	The SPDE~\eqref{EqDefStochHeat} has a mild solution $X: [0,T] \times \Omega \rightarrow L^2(\mathbb{R}^m,\mathcal{L}(\mathbb{R}^m),w)$ given by
	\begin{equation}
		\label{EqLemmaStochHeat1}
		X_t = S_t \chi_0 + b_0 W_t, \quad\quad t \in [0,T],
	\end{equation}
	where the Laplacian $A := \Delta$ in \eqref{EqDefLapl} is the generator of the $C_0$-semigroup $(S_t)_{t \in [0,T]}$ defined by
	\begin{equation}
		\label{EqDefLaplSemigr}
		L^2(\mathbb{R}^m,\mathcal{L}(\mathbb{R}^m),w) \ni x \quad \mapsto \quad S_t x := 
		\begin{cases}
			x, & \text{if } t = 0, \\
			\left( u \mapsto \int_{\mathbb{R}^m} \phi_t(u-v) x(v) dv \right), & \text{if } t \in (0,T],
		\end{cases}
	\end{equation}
	for $t \in [0,T]$, with $\mathbb{R}^m \ni y \mapsto \phi_t(y) := (4\pi t)^{-m/2} \exp\left( -\Vert y \Vert^2/(4t) \right) \in \mathbb{R}$.
\end{lemma}

\begin{remark}
	Let us verify Assumption~\ref{AssPropag}. By using Lemma~\ref{LemmaStochHeat}, Lemma~\ref{LemmaBMFourier}~\ref{LemmaBMFourier2}, and Lemma~\ref{LemmaWick}, it follows for every $\alpha \in \mathcal{J}$ that the propagator $x^{(T)}_\alpha := \mathbb{E}\left[ X_T \xi_\alpha \right] \in L^2(\mathbb{R}^m,\mathcal{L}(\mathbb{R}^m),w)$ satisfies
	\begin{equation*}
		x^{(T)}_\alpha := \mathbb{E}\left[ X_T \xi_\alpha \right] = \left( \phi_T * \chi_0 \right) \mathbb{E}[\xi_\alpha] + \mathbb{E}\left[ W_T \xi_\alpha \right] =
		\begin{cases}
			\phi_T * \chi_0, & \text{if } \alpha = 0, \\
			\int_0^T g_j(t) dt, & \text{if } \alpha = \epsilon(1,j) \text{ for some } j \in \mathbb{N}, \\
			0, & \text{otherwise},
		\end{cases}
	\end{equation*}
	where $\epsilon(1,j) := (\delta_{(1,j),(k,l)})_{k,l \in \mathbb{N}} \in \mathcal{J}$ has zero entries except a one at position $(1,j) \in \mathbb{N} \times \mathbb{N}$. Hence, for any $\alpha \in \mathcal{J}$ with $\vert \alpha \vert \geq 1$, we observe that $x^{(T)}_\alpha := \mathbb{E}\left[ X_T \xi_\alpha \right]: \mathbb{R}^m \rightarrow \mathbb{R}^d$ is constant, thus satisfying Assumption~\ref{AssPropag}. On the other hand, for $\alpha = 0 \in \mathcal{J}$, we refer to \cite[Lemma~10.1]{neufeld23} for conditions on $\chi_0 \in L^2(\mathbb{R}^m,\mathcal{L}(\mathbb{R}^m),w)$ such that $x^{(T)}_0 := \mathbb{E}\left[ X_T \xi_0 \right] = \phi_T * \chi_0 \in L^1(\mathbb{R}^m,\mathcal{L}(\mathbb{R}^m),du)$ has $(\lceil\gamma\rceil+2)$-times differentiable Fourier transform with finite constant $c_0$ defined in \eqref{EqAssPropag}.
\end{remark}

For the numerical example, we choose $(I,J,K) = (1,5,1)$, $T = 1$, and $\mathbb{R}^m \ni u \mapsto \chi_0(u) := 10 \exp\left( -\Vert u \Vert^2/\left( 2\sigma^2 \right) \right) \in \mathbb{R}$ with $\sigma = 6$. Moreover, let $(\omega_{m_1})_{m_1=1,...,M_1} \subseteq \Omega$ with $M_1 = 200$, let $(t_{m_2})_{m_2=0,...,M_2} \subseteq [0,T]$ be equidistant with $M_2 = 20$, and let $(u_{m_3})_{m_3=1,...,M_3} \sim \mathcal{N}_m(0,I_m)$ be an i.i.d.~sequence of normally distributed random variables with $M_3 = 1000$. After splitting the data into 80\%/20\% for training/testing along $(\omega_{m_1})_{m_1=1,...,M_1}$, we run the algorithm for both the supervised and unsupervised approach, and both with deterministic and random neural networks (using the activation function $\rho(s) := \tanh(s)$, and with $N = 25$ and $N = 75$ neurons, respectively). For the training of deterministic networks, we apply the Adam algorithm (see \cite{kingma15}) over $10^5$ epochs with learning rate $2 \cdot 10^{-3}$ and batchsize $40$, while the random networks are learned with the least squares method.

Figure~\ref{FigStochHeat} shows that (possibly random) neural networks in the chaos expansion are able to approximate the solution of the stochastic heat equation \eqref{EqDefStochHeat} via both learning approaches described in Algorithm~\ref{AlgSV}+\ref{AlgUSV}.

\begin{figure}[ht]
	\centering
	\begin{minipage}[t][][t]{0.48\textwidth}
		\centering
		\includegraphics[height = 4.8cm, trim={0 8 0 0}, clip]{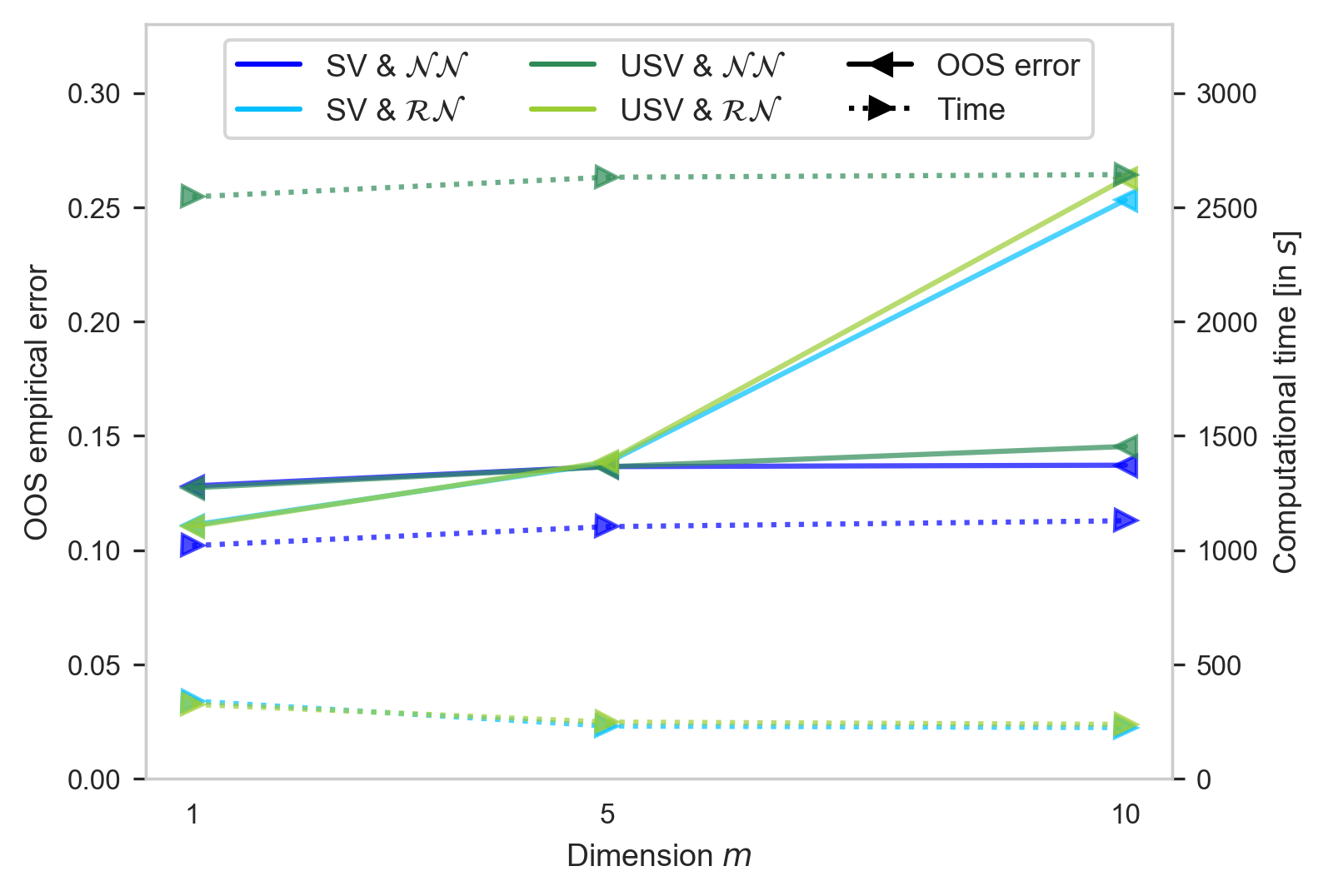}
		\subcaption{Learning performance and computational time}
		\label{FigStochHeat1}
	\end{minipage}\hfill
	\begin{minipage}[t][][t]{0.48\textwidth}
		\centering
		\includegraphics[height = 4.8cm, trim={34 8 4 0}, clip]{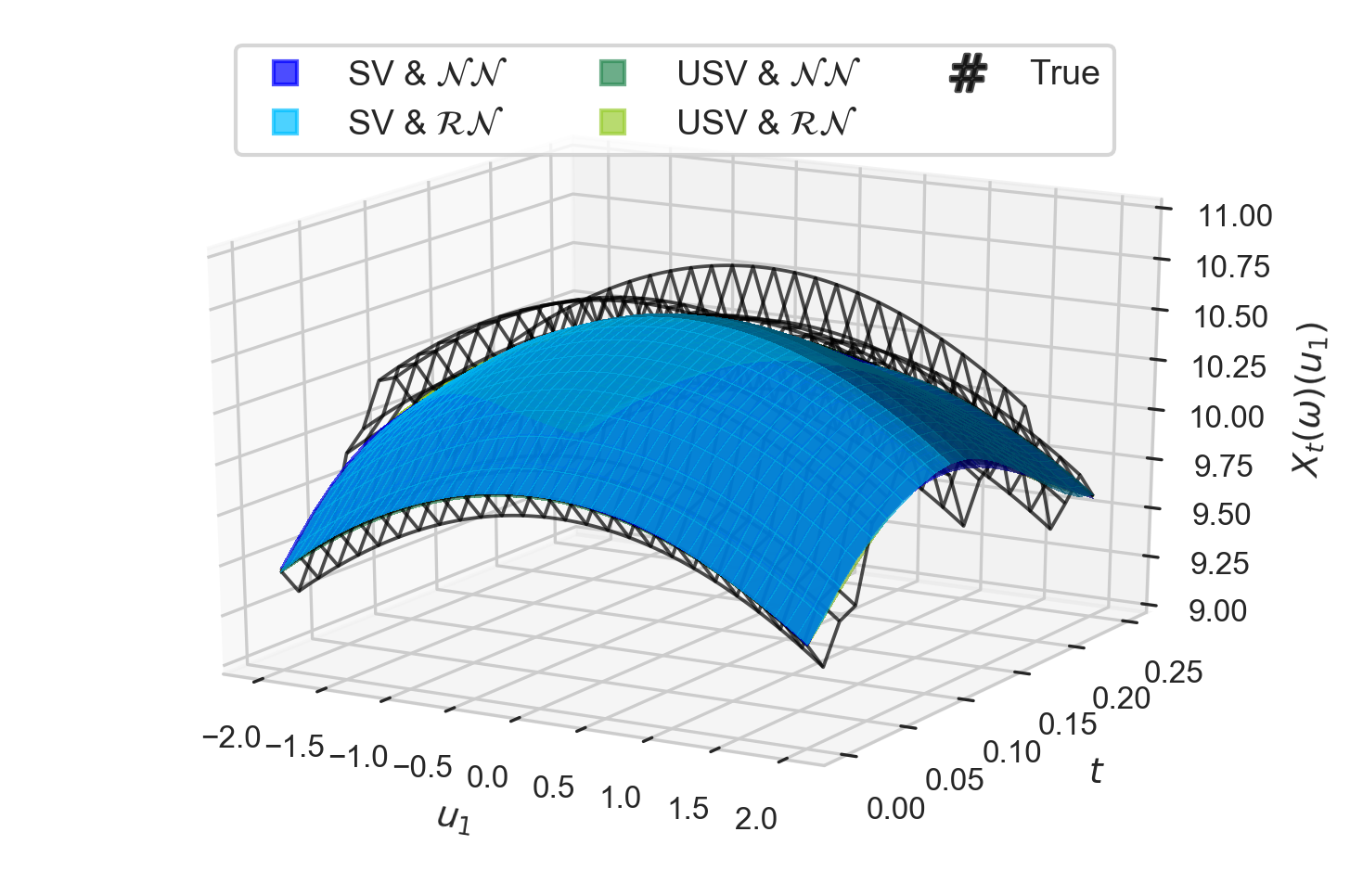}
		\subcaption{Approximation for $m = 1$}
		\label{FigStochHeat2}
	\end{minipage}
	\begin{minipage}[t][][t]{0.48\textwidth}
		\centering
		\includegraphics[height = 4.8cm, trim={34 8 4 0}, clip]{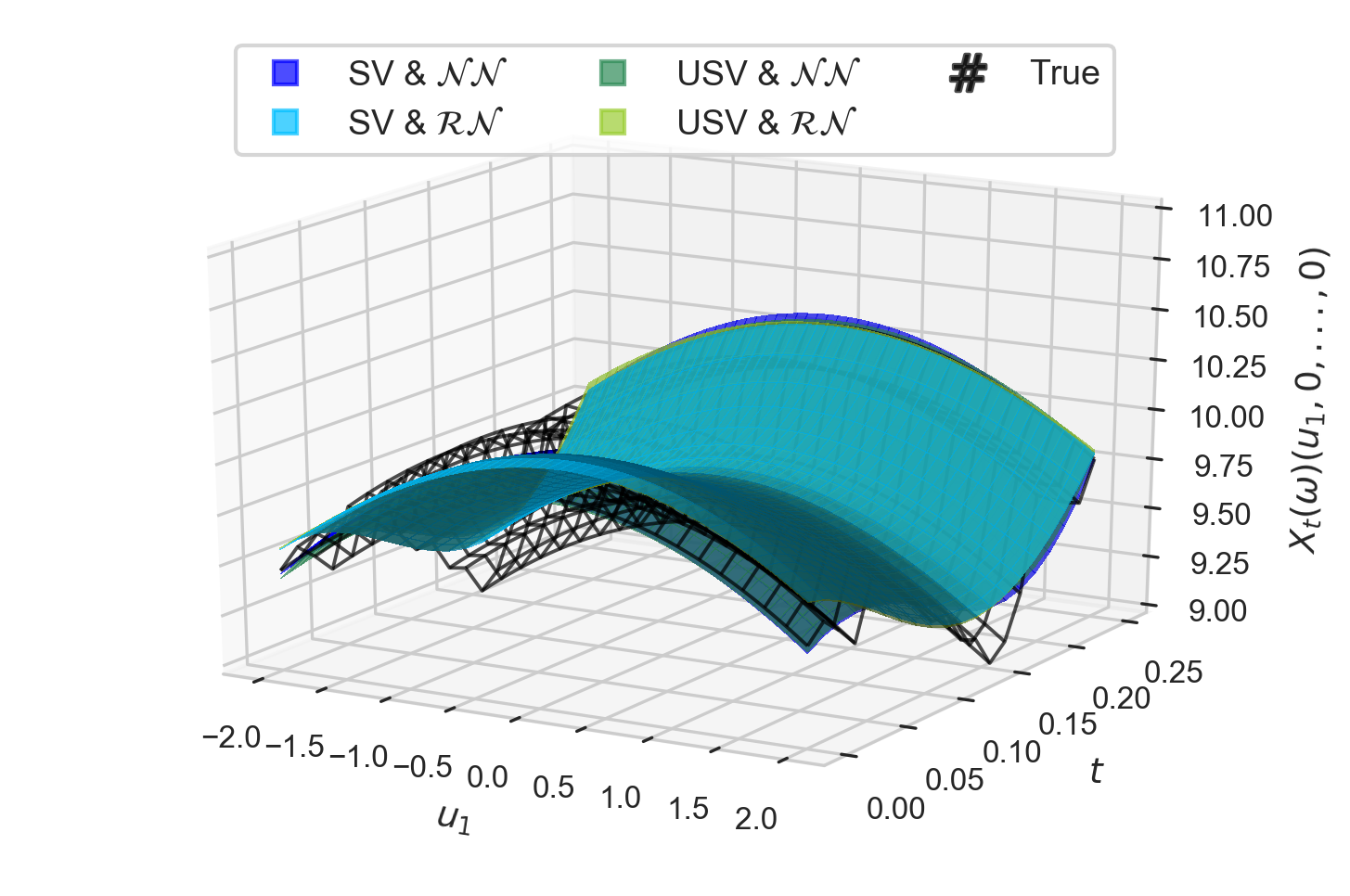}
		\subcaption{Approximation for $m = 5$}
		\label{FigStochHeat3}
	\end{minipage}\hfill
	\begin{minipage}[t][][t]{0.48\textwidth}
		\centering
		\includegraphics[height = 4.8cm, trim={34 8 4 0}, clip]{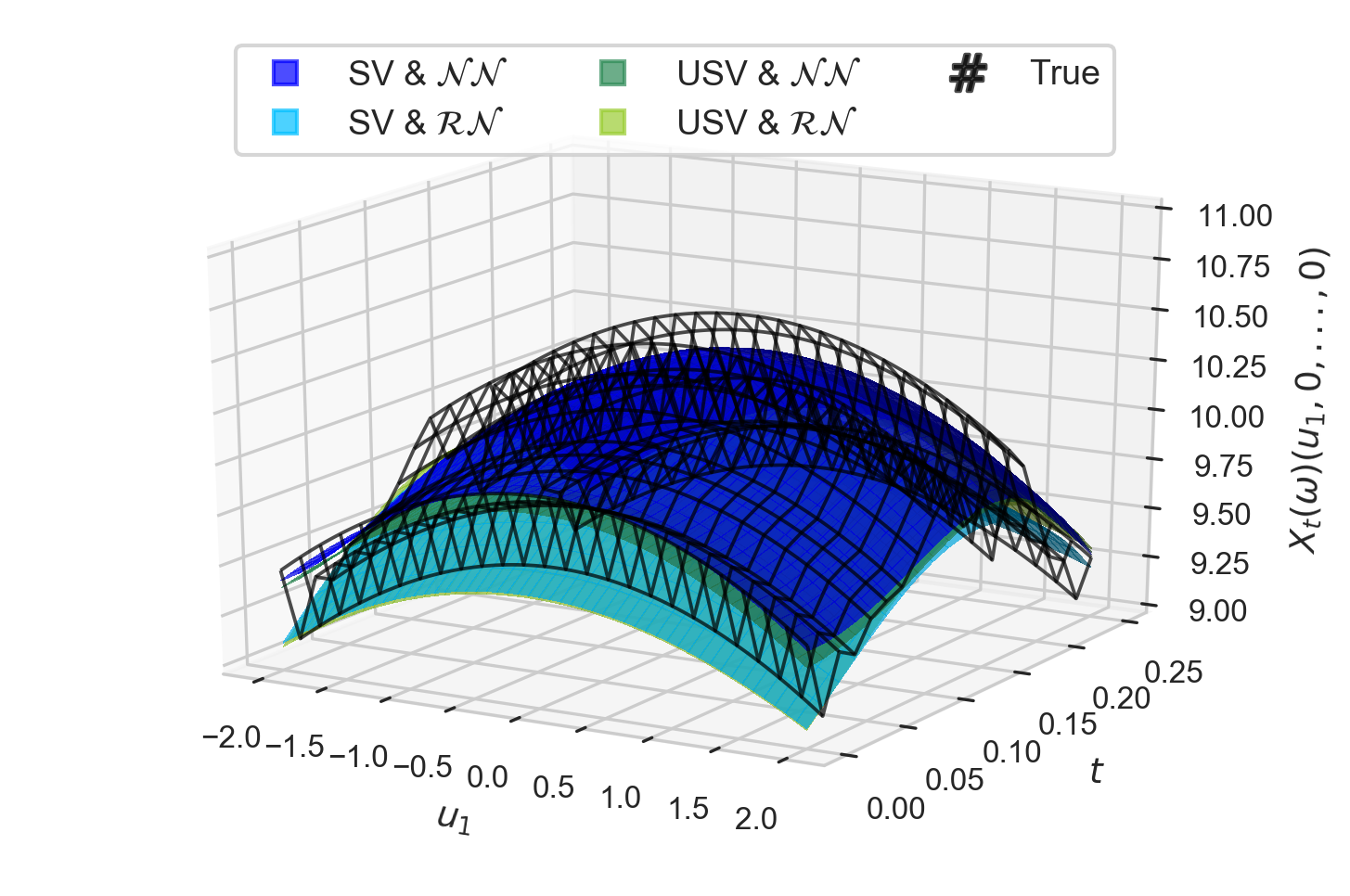}
		\subcaption{Approximation for $m = 10$}
		\label{FigStochHeat4}
	\end{minipage}
	\caption{\small Learning the solution $X: [0,T] \times \Omega \rightarrow H$ of the stochastic heat equation \eqref{EqDefStochHeat} with neural networks (label ``$\mathcal{NN}$'') and random neural networks (label ``$\mathcal{RN}$'') in the chaos expansion, either via supervised learning (Algorithm~\ref{AlgSV}; label ``SV'') or unsupervised learning (Algorithm~\ref{AlgUSV}; label ``USV''). In (\subref{FigStochHeat1}), the learning performance is displayed in terms of the out-of-sample (OOS) empirical error \eqref{EqDefLoss1} together with the computational time. In (\subref{FigStochHeat2})-(\subref{FigStochHeat4}), the learned solutions $[0,T] \times \mathbb{R} \ni (t,u_1) \mapsto X^{(I,J,K)}_t(\omega)(u_1,0,...,0) \in \mathbb{R}$, with $(I,J,K) = (1,5,1)$, are compared to the true solution $[0,T] \times \mathbb{R} \ni (t,u_1) \mapsto X_t(\omega)(u_1,0,...,0) \in \mathbb{R}$ obtained via \eqref{EqLemmaStochHeat1} for different $m \in \lbrace 1,5,10 \rbrace$ and some $\omega \in \Omega$ of the test set.}
	\label{FigStochHeat}
\end{figure}

\subsection{Heath-Jarrow-Morton (HJM) equation in interest rate theory}
\label{SecNumericsHJM}

In the second example, we learn the solution of the Heath-Jarrow-Morton (HJM) equation. To this end, we assume that $(H,\langle \cdot,\cdot \rangle_H)$ is a separable Hilbert space consisting of continuous curves $x: [0,\infty) \rightarrow \mathbb{R}$ such that for every $u \in [0,\infty)$ the evaluation map $H \ni x \mapsto x(u) \in \mathbb{R}$ is continuous. Moreover, we assume that \eqref{EqDefSPDE} takes the form
\begin{equation}
	\label{EqDefHJM}
	\begin{cases}
		dX_t & = \left( \frac{d}{du} X_t + F(t,\cdot,X_t) \right) dt + B(t,\cdot,X_t) dW_t, \quad\quad t \in [0,\infty), \\
		X_0 & = \chi_0 \in H,
	\end{cases}
\end{equation}
where $\chi_0 \in H$ is deterministic, where $\frac{d}{du} =: A: \dom(\frac{d}{du}) \subseteq H \rightarrow H$ is the generator of the $C_0$-semigroup $(S_t)_{t \in [0,T]}$ defined by $H \ni x \mapsto S_t x := x(\cdot + t) \in H$, and where $W: [0,T] \times \Omega \rightarrow Z$ is a $Q$-Brownian motion with values on a (possibly different) separable Hilbert space $(Z,\langle\cdot,\cdot\rangle_Z)$, with $Q \in L_1(Z;Z)$ and RKHS $(Z_0,\langle\cdot,\cdot\rangle_{Z_0})$. Moreover, the diffusion coefficient $B: [0,T] \times \Omega \times H \rightarrow L_2(Z_0;H)$ is assumed to be $(\mathcal{P}_T \otimes \mathcal{B}(H))/\mathcal{B}(L_2(Z_0;H))$-measurable, and the drift coefficient $F: [0,T] \times \Omega \times H \rightarrow H$ is given by the famous HJM no-arbitrage condition (see \cite[Eq.~6.2]{carmona07}), i.e.
\begin{equation}
	\vspace{-0.02cm}
	\label{EqDefHJMDrift}
	F(t,\omega,x) := \left( u \mapsto \sum_{i=1}^\infty B(t,\omega,x)(\widetilde{e}_i)(u) \int_0^u B(t,\omega,x)(\widetilde{e}_i)(v) dv + B(t,\omega,x)(\eta(t,\omega,x))(u) \right) \in H,
\end{equation}
for $(t,\omega,x) \in [0,T] \times \Omega \times H$, where $(\widetilde{e}_i)_{i \in \mathbb{N}} := \left( Q^{1/2} e_i \right)_{i \in \mathbb{N}}$ is an orthonormal basis of $(Z_0,\langle \cdot,\cdot \rangle_{Z_0})$, and $\eta := (\eta_t)_{t \in [0,T]}: [0,T] \times \Omega \times H \rightarrow Z_0$ is any $(\mathcal{P}_T \otimes \mathcal{B}(H))/\mathcal{B}(Z_0)$-measurable map.

\begin{remark}
	\label{RemHJM}
	The solution $X := (X_t)_{t \in [0,T]}$ of the HJM equation~\eqref{EqDefHJM} describes the instantaneous forward rate. Indeed, let $r := (r_t)_{t \in [0,\infty)}: [0,\infty) \times \Omega \rightarrow \mathbb{R}$ be the spot interest rate. Then, for any $t \in [0,T]$ and $u \in [0,\infty)$, the price of a zero coupon bond with maturity date $u$ and nominal value $\$1$ is defined as $P_t(u) := \mathbb{E}\big[ \exp\big( \!\!-\int_t^{t+u} r_s ds \big) \big\vert \mathcal{F}_t \big]$. Hence, if $B(t,\omega,X_t) \eta(t,\omega,X_t) = 0 \in H$ for $(dt \otimes d\mathbb{P})$-a.e.~$(t,\omega) \in [0,T] \times \Omega$, the process $X := (X_t)_{t \in [0,T]}: [0,T] \times \Omega \rightarrow H$ defined as
	\vspace{-0.02cm}
	\begin{equation*}
		X_t := \left( u \mapsto -\frac{\partial}{\partial u} \log(P_t(u)) \right), \quad\quad t \in [0,T],
	\end{equation*}
	expresses the instantaneous forward rate in the sense that for every $t \in [0,T]$ and $u \in [0,\infty)$ it holds that $P_t(u) = \exp\big( \!\!-\int_t^{t+u} X_t(v-t) dv \big)$. For more details, we refer to \cite[Chapter~6]{carmona07}.
\end{remark}

For the numerical experiment, we let $[0,\infty) \ni u \mapsto w(u) := \exp(0.1 u) \in [0,\infty)$ and consider the Hilbert space $(H,\langle \cdot,\cdot \rangle_H)$ consisting of absolutely continuous curves $x: [0,\infty) \rightarrow \mathbb{R}$ satisfying $\int_0^\infty x'(u)^2 w(u) du < \infty$, equipped with the inner product $\langle x,y \rangle_H := x(0) y(0) + \int_0^\infty x'(u) y'(u) w(u) du$, for $x,y \in H$ (see also \cite[Section~6.3.3]{carmona07} and \cite{filipovic01}). Moreover, we set $Z := \mathbb{R}$ and consider a real-valued $Q$-Brownian motion $W: [0,T] \times \Omega \rightarrow \mathbb{R}$, with $Q := \id_\mathbb{R} \in L_1(\mathbb{R};\mathbb{R})$, and therefore $\lambda_1 := 1$ and $\lambda_i := 0$ for all $i \in \mathbb{N} \cap [2,\infty)$, cf.~\eqref{EqDefBMi}. In addition, we assume that the spot interest rate $r := (r_t)_{t \in [0,T]}: [0,T] \times \Omega \rightarrow \mathbb{R}$ follows a Vasi\v{c}ek model
\begin{equation*}
	\begin{cases}
		dr_t & = (\mu - \kappa r_t) dt + \sigma dW_t, \quad\quad t \in [0,T], \\
		r_0 & \in \mathbb{R},
	\end{cases}
\end{equation*}
where $r_0 \in \mathbb{R}$ is deterministic, and where $\mu, \kappa \in \mathbb{R}$ as well as $\sigma > 0$ are given parameters. Then, by \cite[Equation~2.33]{carmona07} (see also \cite[Section~5.4.1]{filipovic09}), this corresponds to the HJM equation~\eqref{EqDefHJM} with coefficients $[0,T] \times \Omega \times H \ni (t,\omega,x) \mapsto B(t,\omega,x) := \left( z \mapsto \sigma \exp(-\kappa \, \cdot) z \right) \in L_2(\mathbb{R};H)$ and $[0,T] \times \Omega \times H \ni (t,\omega,x) \mapsto F(t,\omega,x) := \frac{\sigma^2}{\kappa} \exp(-\kappa \, \cdot) \left( 1 - \exp(-\kappa \, \cdot) \right) \in H$ by \eqref{EqDefHJMDrift}, whose solution is given by
\begin{equation}
	\label{EqDefHJMSol}
	X_t = \left( u \mapsto -\frac{\partial}{\partial u} \log(P_t(u)) \right) = \left( u \mapsto r_t e^{-\kappa u} + \frac{\mu}{\kappa} \left( 1 - e^{-\kappa u} \right) - \frac{\sigma^2}{2\kappa^2} \left( 1 - e^{-\kappa u} \right)^2 \right).
\end{equation}
Now, we choose $I = 1$, $T = 1$, $r_0 = 4$, $\mu = 4$, $\kappa = 0.9$, and $\sigma = 0.5$, let $(\omega_{m_1})_{m_1=1,...,M_1} \subseteq \Omega$ with $M_1 = 200$, let $(t_{m_2})_{m_2=0,...,M_2} \subseteq [0,T]$ be an equidistant time grid with $M_2 = 20$, and let $u_1 := 0$ and $(u_{m_3})_{m_3=2,...,M_3} \sim C_u \mathds{1}_{[0,3]}(u) \exp(0.1 u) du$ be an i.i.d.~sequence of random variables with $M_3 = 81$, where $C_u > 0$ is a normalizing constant. After splitting the data into 80\%/20\% for training/testing along $(\omega_{m_1})_{m_1=1,...,M_1}$, we run the algorithm for both the supervised and unsupervised approach, and both with deterministic and random neural networks (using the activation function $\rho(s) := \tanh(s)$, and with $N = 25$ and $N = 75$ neurons, respectively), where we choose $\widetilde{c}_{\beta,m_1,m_2,m_3} := \mathds{1}_{\lbrace 0 \rbrace \times \lbrace 0 \rbrace}(\beta,m_3) + \mathds{1}_{\lbrace 1 \rbrace \times \lbrace 2,...,M_3 \rbrace}(\beta,m_3)$ for $\beta \in \mathbb{N}^1_{0,1} = \lbrace 0,1 \rbrace$, $m_1 = 1,...,M_1$, $m_2 = 0,...,M_2$, and $m_3 = 1,...,M_3$. This means that the empirical error \eqref{EqDefLoss1} becomes now
\begin{equation*}
	\left( \sum_{m_1=1}^{M_1} \sum_{m_2=0}^{M_2} \left( \left\vert \left( X_{t_{m_2}} - X^{(I,J,K)}_{t_{m_2}} \right)(\omega_{m_1})(0) \right\vert^2 + \sum_{m_3=1}^{M_3} \left\vert \frac{d}{du} \left( X_{t_{m_2}} - X^{(I,J,K)}_{t_{m_2}} \right)(\omega_{m_1})(u_{m_3}) \right\vert^2 \right) \right)^\frac{1}{2}
\end{equation*}
and analogously for \eqref{EqDefLoss2}, which is a truncated empirical version of the norm of $(H,\langle \cdot,\cdot \rangle_H)$. For the training of deterministic networks, we apply the Adam algorithm (see \cite{kingma15}) over $10^5$ epochs with learning rate $5 \cdot 10^{-4}$ and batchsize $40$, while the random networks are learned with the least squares method.

Figure~\ref{FigHJM} shows that deterministic and random neural networks in the Wiener chaos expansion are able to learn the solution of the HJM equation \eqref{EqDefHJM} via both learning approaches described in Algorithm~\ref{AlgSV}+\ref{AlgUSV}. Moreover, Figure~\ref{FigHJM}~(\subref{FigHJM3}) shows that a Gaussian approximation ($K = 1$) is sufficient (see Remark~\ref{RemGaussian}).

\begin{figure}[ht]
	\centering
	\begin{minipage}[t][][t]{0.48\textwidth}
		\centering
		\includegraphics[height = 4.8cm, trim={0 8 0 0}, clip]{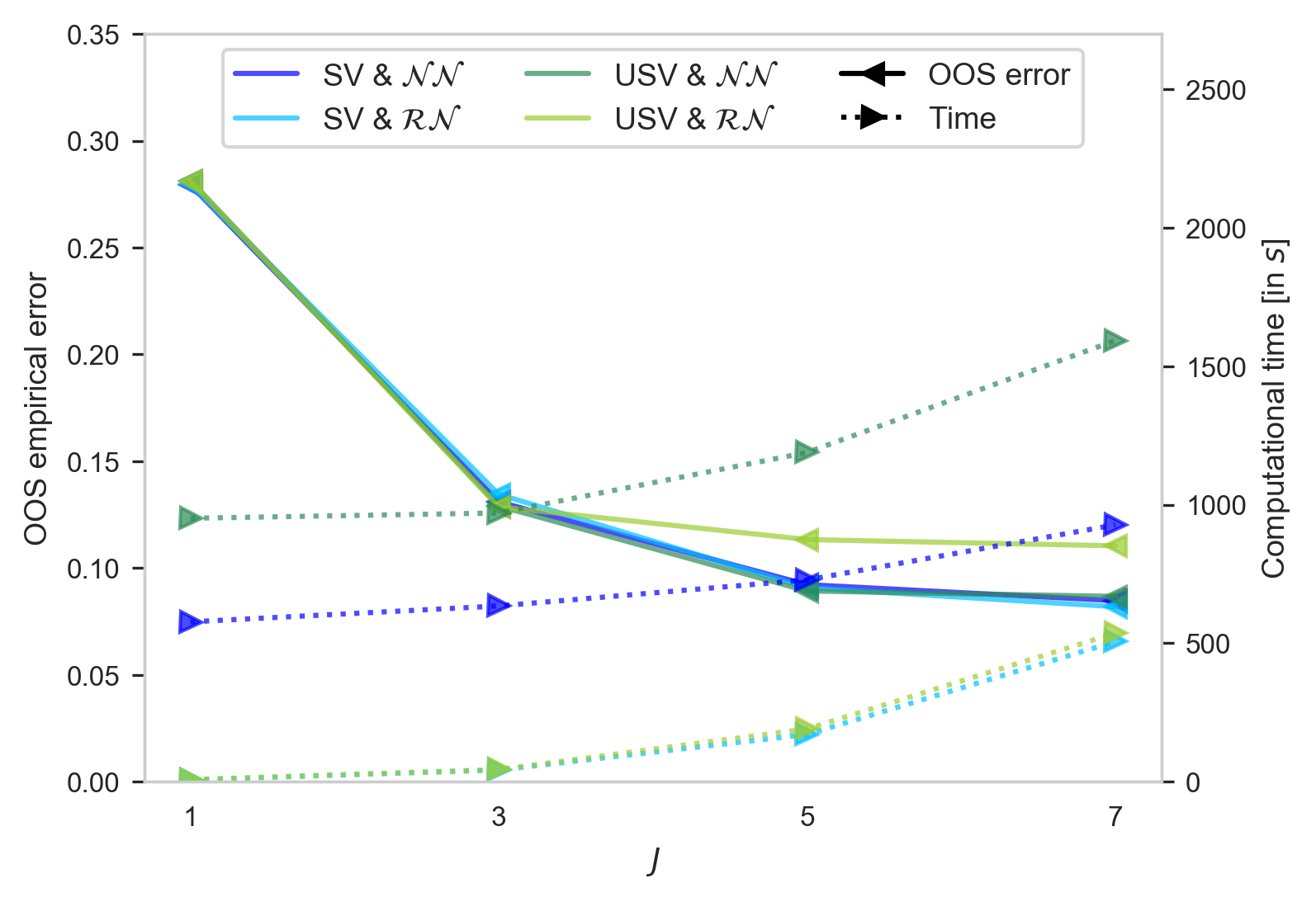}
		\subcaption{Learning performance and computational time for $(J,K) \in \lbrace 1,3,5,7 \rbrace \times \lbrace 2 \rbrace$}
		\label{FigHJM1}
	\end{minipage}\hfill
	\begin{minipage}[t][][t]{0.48\textwidth}
		\centering
		\includegraphics[height = 4.8cm, trim={34 8 4 0}, clip]{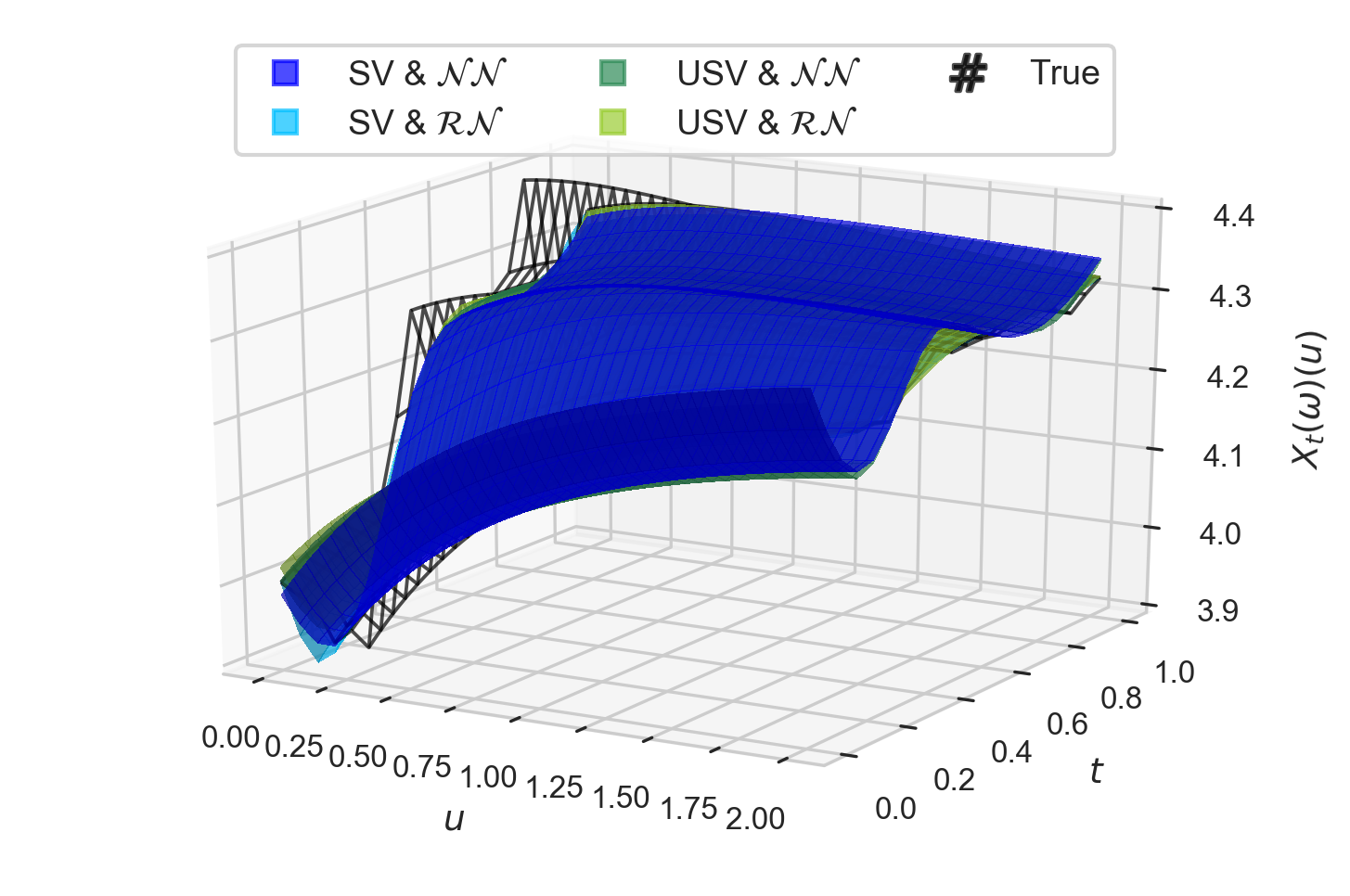}
		\subcaption{Approximation for $(J,K) = (7,2)$}
		\label{FigHJM2}
	\end{minipage}
	\begin{minipage}[t][][t]{0.48\textwidth}
		\centering
		\includegraphics[height = 4.8cm, trim={0 8 0 0}, clip]{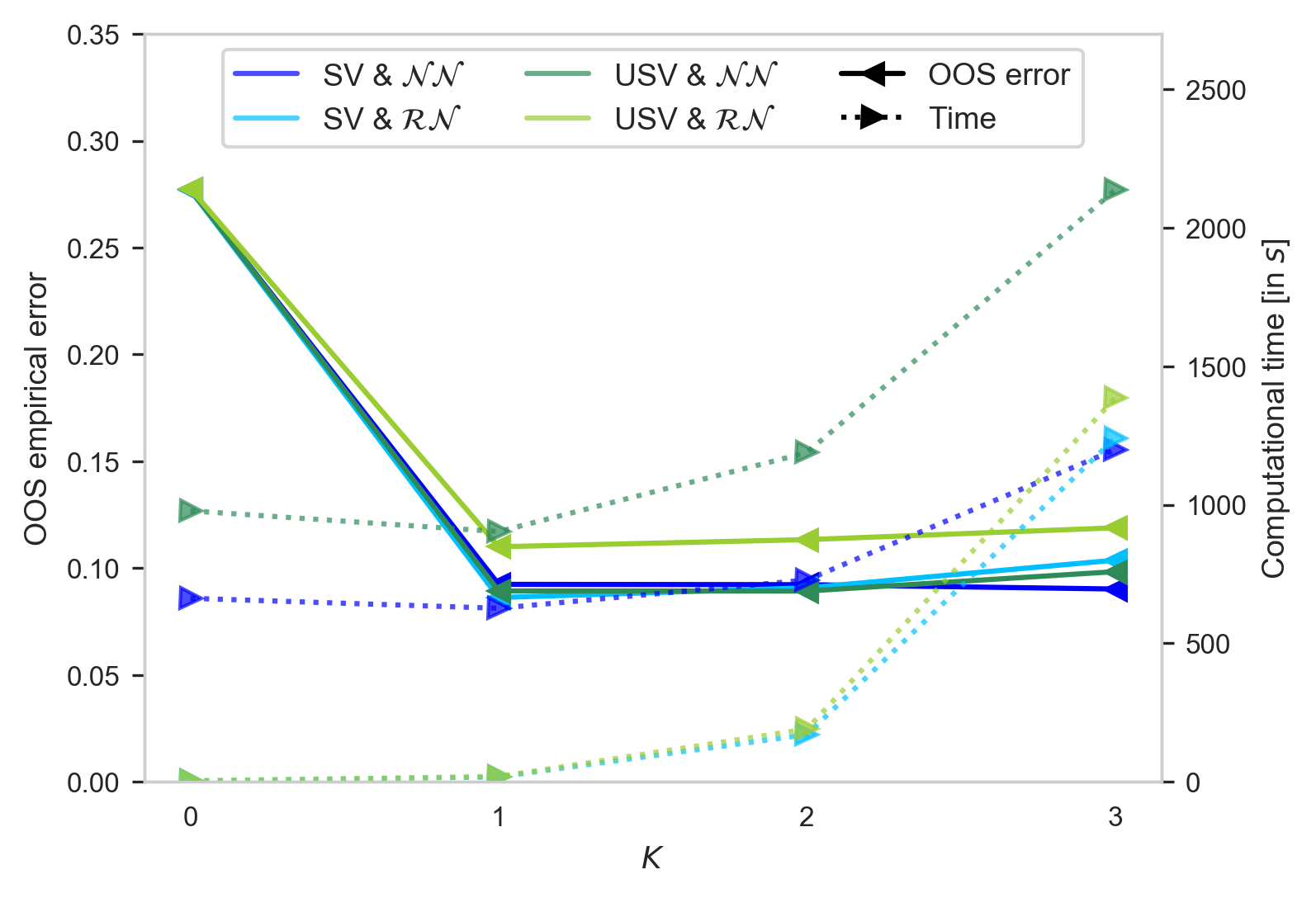}
		\subcaption{Learning performance and computational time for $(J,K) \in \lbrace 5 \rbrace \times \lbrace 0,1,2,3 \rbrace$}
		\label{FigHJM3}
	\end{minipage}\hfill
	\begin{minipage}[t][][t]{0.48\textwidth}
		\centering
		\includegraphics[height = 4.8cm, trim={34 8 4 0}, clip]{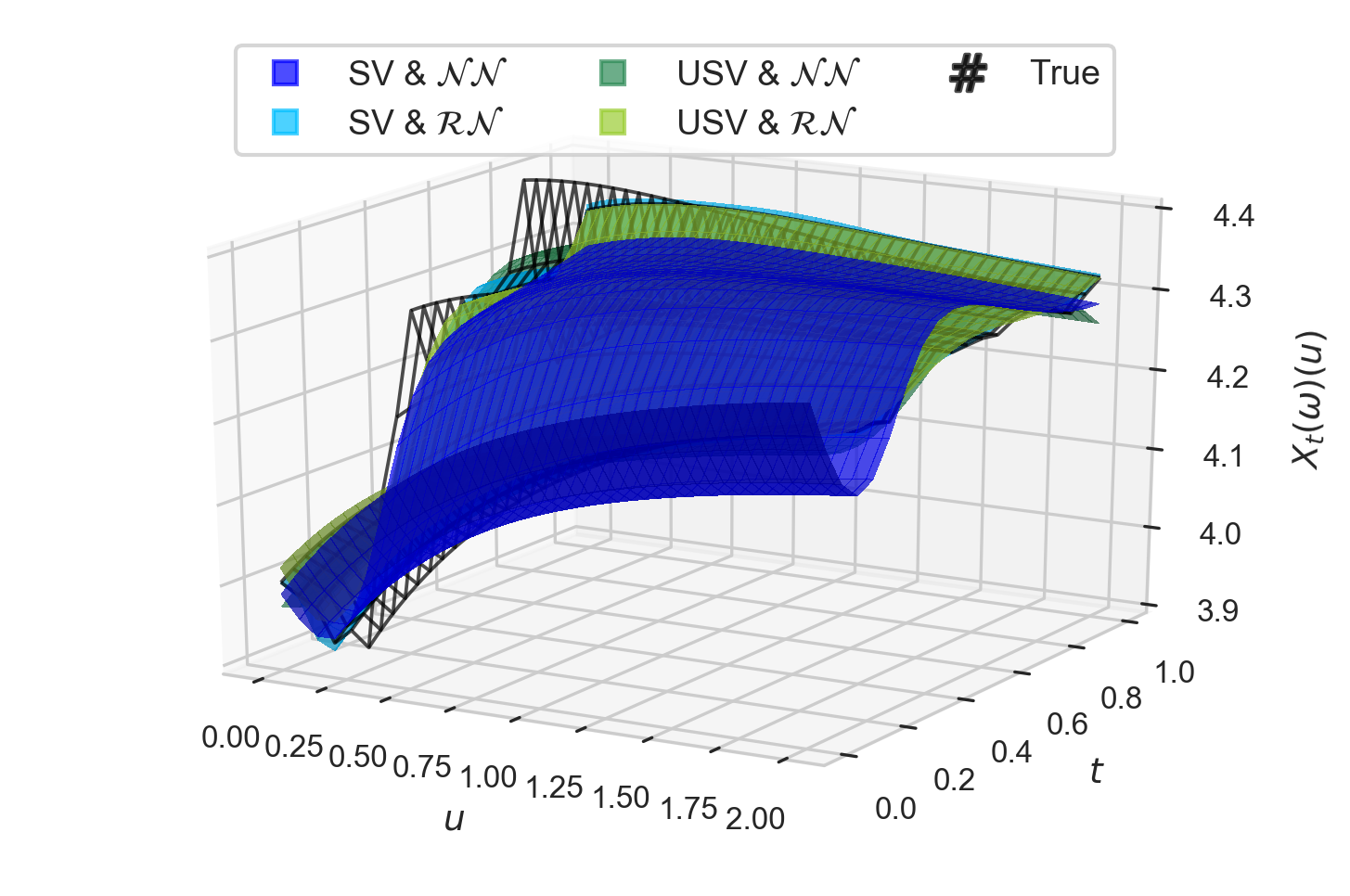}
		\subcaption{Approximation for $(J,K) = (5,3)$}
		\label{FigHJM4}
	\end{minipage}
	\caption{\small Learning the solution $X: [0,T] \times \Omega \rightarrow H$ of the HJM equation~\eqref{EqDefHJM} with neural networks (label ``$\mathcal{NN}$'') and random neural networks (label ``$\mathcal{RN}$'') in the chaos expansion, either via supervised learning (Algorithm~\ref{AlgSV}; label ``SV'') or unsupervised learning (Algorithm~\ref{AlgUSV}; label ``USV''). In (\subref{FigHJM1})+(\subref{FigHJM3}), the learning performance is displayed in terms of the out-of-sample (OOS) empirical error \eqref{EqDefLoss1} together with the computational time for different $J \in \lbrace 1,3,5,7 \rbrace$ and $K \in \lbrace 0,1,2,3 \rbrace$, respectively. In (\subref{FigHJM2})+(\subref{FigHJM4}), the learned solutions $[0,T] \times \mathbb{R} \ni (t,u) \mapsto X^{(I,J,K)}_t(\omega)(u) \in \mathbb{R}$, with $(I,J,K) = (1,7,2)$ and $(I,J,K) = (1,5,3)$, are compared to the true solution $[0,T] \times \mathbb{R} \ni (t,u) \mapsto X_t(\omega)(u) \in \mathbb{R}$ obtained via \eqref{EqDefHJMSol} for some $\omega \in \Omega$ of the test set.}
	\label{FigHJM}
\end{figure}

\vspace{-0.1cm}

\subsection{Zakai equation in filtering}
\label{SecNumericsZakai}

In the third example, we consider the Zakai equation in filtering theory (see \cite{kalman60,kushner64,zakai69,rozovsky18}), which describes the evolution of an unnormalized density function of the conditional distribution of an unobservable process. To this end, we consider for some fixed $m \in \mathbb{N}$ the Hilbert space $(H,\langle\cdot,\cdot\rangle_H) := (L^2(\mathbb{R}^m,\mathcal{L}(\mathbb{R}^m),w),\langle\cdot,\cdot\rangle_{L^2(\mathbb{R}^m,\mathcal{L}(\mathbb{R}^m),w)})$, where $\mathbb{R}^m \ni u \mapsto w(u) := (2\pi)^{-m/2} \exp\left( - \Vert u \Vert^2/2 \right) \in [0,\infty)$. Moreover, for $T > 0$, we assume that $Y := (Y_t)_{t \in [0,T]}: [0,T] \times \Omega \rightarrow \mathbb{R}^m$ is an unobserved signal process satisfying
\vspace{-0.1cm}
\begin{equation}
	\label{EqDefY}
	Y_t = Y_0 + \int_0^t \mu(Y_s) ds + \int_0^t \sigma(Y_s) d\widetilde{W}_s, \quad\quad t \in [0,T],
	\vspace{-0.1cm}
\end{equation}
where $Y_0$ is $\mathcal{F}_0$-measurable with probability density function $\chi_0: \mathbb{R}^m \rightarrow [0,\infty)$, where $\mu \in C^2_b(\mathbb{R}^m;\mathbb{R}^m)$ and $\sigma \in C^3_b(\mathbb{R}^m;\mathbb{R}^{m \times m})$ are the SDE-coefficients, and where $\widetilde{W} := (\widetilde{W}_t)_{t \in [0,T]}: [0,T] \times \Omega \rightarrow \mathbb{R}^m$ is an $m$-dimensional Brownian motion. Then, $Y: [0,T] \times \Omega \rightarrow \mathbb{R}^m$ is a Markov process with generator
\begin{equation}
	\vspace{-0.1cm}
	\label{EqDefGenerator}
	C^\infty_c(\mathbb{R}^m) \ni x \quad \mapsto \quad \mathcal{A} x := \frac{1}{2} \trace\left( \sigma(\cdot) \sigma(\cdot)^\top \nabla^2 x(\cdot) \right) + \mu(\cdot)^\top \nabla x(\cdot) \in C^\infty_c(\mathbb{R}^m) \subseteq L^2(\mathbb{R}^m,\mathcal{L}(\mathbb{R}^m),w),
	\vspace{-0.1cm}
\end{equation}
where $\nabla x(u) := \big( \frac{\partial x}{\partial u_i}(u) \big)_{i=1,...,m} \in \mathbb{R}^m$ and $\nabla^2 x(u) := \big( \frac{\partial^2 x}{\partial u_i \partial u_j}(u) \big)_{i,j=1,...,m} \in \mathbb{R}^{m \times m}$. Moreover, we assume that $Z := (Z_t)_{t \in [0,T]}: [0,T] \times \Omega \rightarrow \mathbb{R}^n$ is the observed process, which satisfies
\begin{equation}
	\vspace{-0.1cm}
	\label{EqDefObservation}
	Z_t = \int_0^t \kappa(Y_s) ds + W_t, \quad\quad t \in [0,T],
	\vspace{-0.1cm}
\end{equation}
where $\kappa: \mathbb{R}^m \rightarrow \mathbb{R}^n$ is a continuous function, and where $\widetilde{W} := (\widetilde{W}_t)_{t \in [0,T]}: [0,T] \times \Omega \rightarrow \mathbb{R}^n$ is an $n$-dimensional Brownian motion, independent from $\widetilde{W}: [0,T] \times \Omega \rightarrow \mathbb{R}^m$. Then, by assuming that $\chi_0 \in W^{2,2}(\mathbb{R}^m,\mathcal{L}(\mathbb{R}^m),w)$, we can apply~\cite[Theorem~6.3]{rozovsky18} to conclude that the solution of the SPDE
\begin{equation}
	\label{EqDefZakai}
	\begin{cases}
		dX_t & = \left( \mathcal{A}^* X_t + F(t,\cdot,X_t) \right) dt + B(t,\cdot,X_t) dW_t, \quad\quad t \in [0,T], \\
		X_0 & = \chi_0 \in L^2(\mathbb{R}^m,\mathcal{L}(\mathbb{R}^d),w),
	\end{cases}
\end{equation}
describes at each time $t \in [0,T]$ an unnormalized density function of the conditional distribution of $Y_t$ given the observations of $(Z_s)_{s \in [0,t]}$, which means that
\begin{equation*}
	\mathbb{P}\left[ Y_t \in A \big\vert (Z_s)_{s \in [0,t]} \right] = \frac{\int_A X_t(u) du}{\int_{\mathbb{R}^m} X_t(u) du}, \quad\quad A \in \mathcal{L}(\mathbb{R}^m),
\end{equation*}
where $\mathcal{A}^*: \dom(\mathcal{A}^*) := C^\infty_c(\mathbb{R}^m) \subseteq L^2(\mathbb{R}^m,\mathcal{L}(\mathbb{R}^m),w) \rightarrow L^2(\mathbb{R}^m,\mathcal{L}(\mathbb{R}^m),w)$ is the adjoint\footnote{In fact, $\mathcal{A}^*$ is defined as the adjoint of $\mathcal{A}: (C^\infty_c(\mathbb{R}^m),\Vert\cdot\Vert_{L^2(\mathbb{R}^m,\mathcal{L}(\mathbb{R}^m),du)}) \rightarrow (C^\infty_c(\mathbb{R}^m),\Vert\cdot\Vert_{L^2(\mathbb{R}^m,\mathcal{L}(\mathbb{R}^m),du)})$ in \eqref{EqDefGenerator}, i.e.~as the operator $\mathcal{A}^*: (C^\infty_c(\mathbb{R}^m),\Vert\cdot\Vert_{L^2(\mathbb{R}^m,\mathcal{L}(\mathbb{R}^m),du)}) \rightarrow (C^\infty_c(\mathbb{R}^m),\Vert\cdot\Vert_{L^2(\mathbb{R}^m,\mathcal{L}(\mathbb{R}^m),du)})$ such that for every $x,y \in C^\infty_c(\mathbb{R}^m)$ it holds that $\langle \mathcal{A} x, y \rangle_{L^2(\mathbb{R}^m,\mathcal{L}(\mathbb{R}^m),du)} = \langle x, \mathcal{A}^* y \rangle_{L^2(\mathbb{R}^m,\mathcal{L}(\mathbb{R}^m),du)}$. Hence, by applying integration by parts, it follows that $\mathcal{A}^* y = \frac{1}{2} \sum_{i,j,l=1}^m \frac{\partial}{\partial u_i \partial u_j} \left( \sigma_{i,l}(\cdot) \sigma_{j,l}(\cdot) y(\cdot) \right) - \sum_{l=1}^m \frac{\partial}{\partial u_l} \left( \mu_l(\cdot) y(\cdot) \right)$ for all $y \in C^\infty_c(\mathbb{R}^m)$, where $\sigma := (\sigma_{i,j})_{i,j=1,...,m} \in C^3_b(\mathbb{R}^m;\mathbb{R}^{m \times m})$ and $\mu := (\mu_l)_{l=1,...,m}^\top \in C^2_b(\mathbb{R}^m;\mathbb{R}^m)$.} of $\mathcal{A}$ given in \eqref{EqDefGenerator}, and where $W := (W_t)_{t \in [0,T]}: [0,T] \times \Omega \rightarrow Z := \mathbb{R}^m$ is the $m$-dimensional Brownian motion from \eqref{EqDefObservation}. Note that \eqref{EqDefZakai} is of the form \eqref{EqDefSPDE} with coefficients $F$ and $B$ given by
\begin{equation*}
	\begin{aligned}
		[0,T] \times \Omega \times H \ni (t,\omega,x) \quad \mapsto \quad F(t,\omega,x) & := x(\cdot) \kappa(\cdot)^\top \kappa(Y_t(\omega)) \in H \\
		[0,T] \times \Omega \times H \ni (t,\omega,x) \quad \mapsto \quad B(t,\omega,x) & := \left( v \mapsto x(\cdot) \kappa(\cdot)^\top v \right) \in L_2(\mathbb{R}^n;H),
	\end{aligned}
\end{equation*}
where $H := L^2(\mathbb{R}^m,\mathcal{L}(\mathbb{R}^m),w)$. Since $X := (X_t)_{t \in [0,T]}$ depends on the randomness of both Brownian motions $W$ and $\widetilde{W}$ (the latter via $Y$ in $F$ and $B$), we have to include both of them into the computation of the Wick polynomials $(\xi_\alpha)_{\alpha \in \mathcal{J}}$. This means that we consider the ($\mathbb{R}^n \times \mathbb{R}^m$)-valued $Q$-Brownian motion $\overline{W} := (W_t, \widetilde{W}_t)_{t \in [0,T]}: [0,T] \times \Omega \rightarrow \mathbb{R}^n \times \mathbb{R}^m$, with $Q := \id_{\mathbb{R}^n \times \mathbb{R}^m} \in L_1(\mathbb{R}^n \times \mathbb{R}^m;\mathbb{R}^n \times \mathbb{R}^m)$, and thus $\lambda_i := 1$ for all $i = 1,...,m+n$, and $\lambda_i := 0$ for all $i \in \mathbb{N} \cap [m+n+1,\infty)$, cf.~\eqref{EqDefBMi}.

For the numerical example, we choose $m = 2$ and $T = 0.5$, the drift function $\mathbb{R}^m \ni y \mapsto \mu(y) := 0.25 \frac{y}{1 + \Vert y \Vert^2} \in \mathbb{R}^m$, and diffusion function $\mathbb{R}^m \ni y \mapsto \sigma(y) := m^{-1/2} \mathbf{1} \in \mathbb{R}^{m \times m}$, with $\mathbf{1} \in \mathbb{R}^{m \times m}$ denoting the matrix having all entries equal to one, and assume that $Y_0 \sim \mathcal{N}_m(0,I_m)$, i.e.~$\mathbb{R}^m \ni u \mapsto \chi_0(u) := \left( 2\pi \right)^{-m/2} \exp\left( -\Vert u \Vert^2/2 \right) \in \mathbb{R}$. Moreover, we consider $m = n$ and $\mathbb{R}^m \ni y \mapsto \kappa(y) := 0.5 y \in \mathbb{R}^n$. In addition, let $(\omega_{m_1})_{m_1=1,...,M_1} \subseteq \Omega$ with $M_1 = 300$, let $(t_{m_2})_{m_2=0,...,M_2} \subseteq [0,T]$ be an equidistant grid with $M_2 = 20$, and let $(u_{m_3})_{m_3=1,...,M_3} \sim \mathcal{N}_m(0,I_m)$ be an i.i.d.~sequence of normally distributed random variables with $M_3 = 150$. After splitting the data into 80\%/20\% for training/testing along $(\omega_{m_1})_{m_1=1,...,M_1}$, we run the algorithm for both the supervised and unsupervised approach, and both with deterministic and random neural networks (using the activation function $\rho(s) := \tanh(s)$, and with $N = 25$ and $N = 75$ neurons, respectively). For the training of deterministic networks, we apply the Adam algorithm (see \cite{kingma15}) over $10^5$ epochs with learning rate $5 \cdot 10^{-4}$ and batchsize $40$, while the random networks are learned with the least squares method. 

Figure~\ref{FigZakai} shows that the truncated Wiener chaos expansion with deterministic and random neural networks can learn the solution of the Zakai equation~\eqref{EqDefZakai} using both Algorithm~\ref{AlgSV}+\ref{AlgUSV}. Hereby, the true solution $X: [0,T] \times \Omega \rightarrow H$ of \eqref{EqDefZakai} is approximated with the Monte-Carlo method in \cite{beck23}.

\begin{figure}[ht]
	\centering
	\begin{minipage}[t][][t]{0.48\textwidth}
		\centering
		\includegraphics[height = 4.8cm, trim={0 8 0 0}, clip]{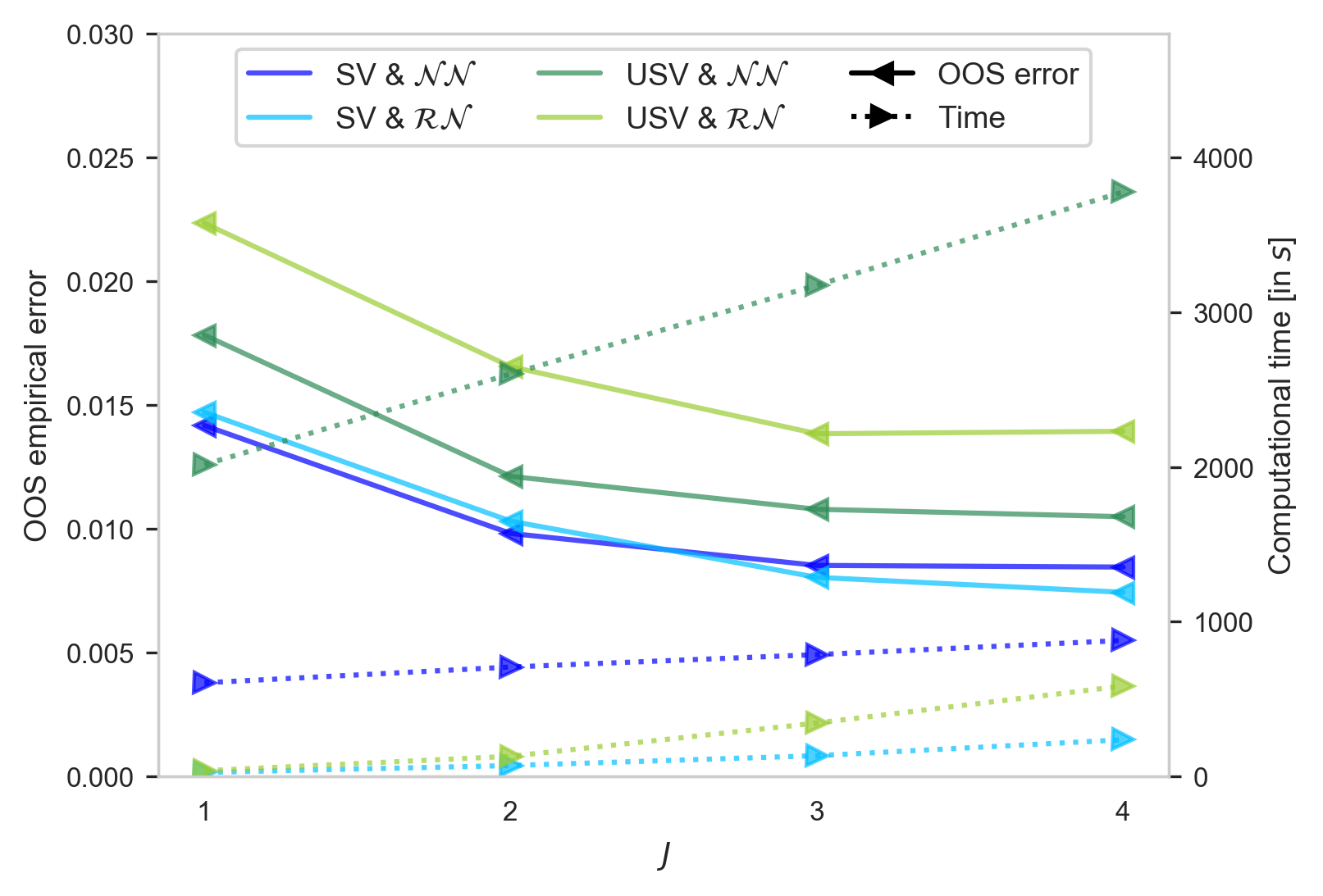}
		\subcaption{Learning performance and computational time}
		\label{FigZakai1}
	\end{minipage}\hfill
	\begin{minipage}[t][][t]{0.48\textwidth}
		\centering
		\includegraphics[height = 4.8cm, trim={34 8 4 0}, clip]{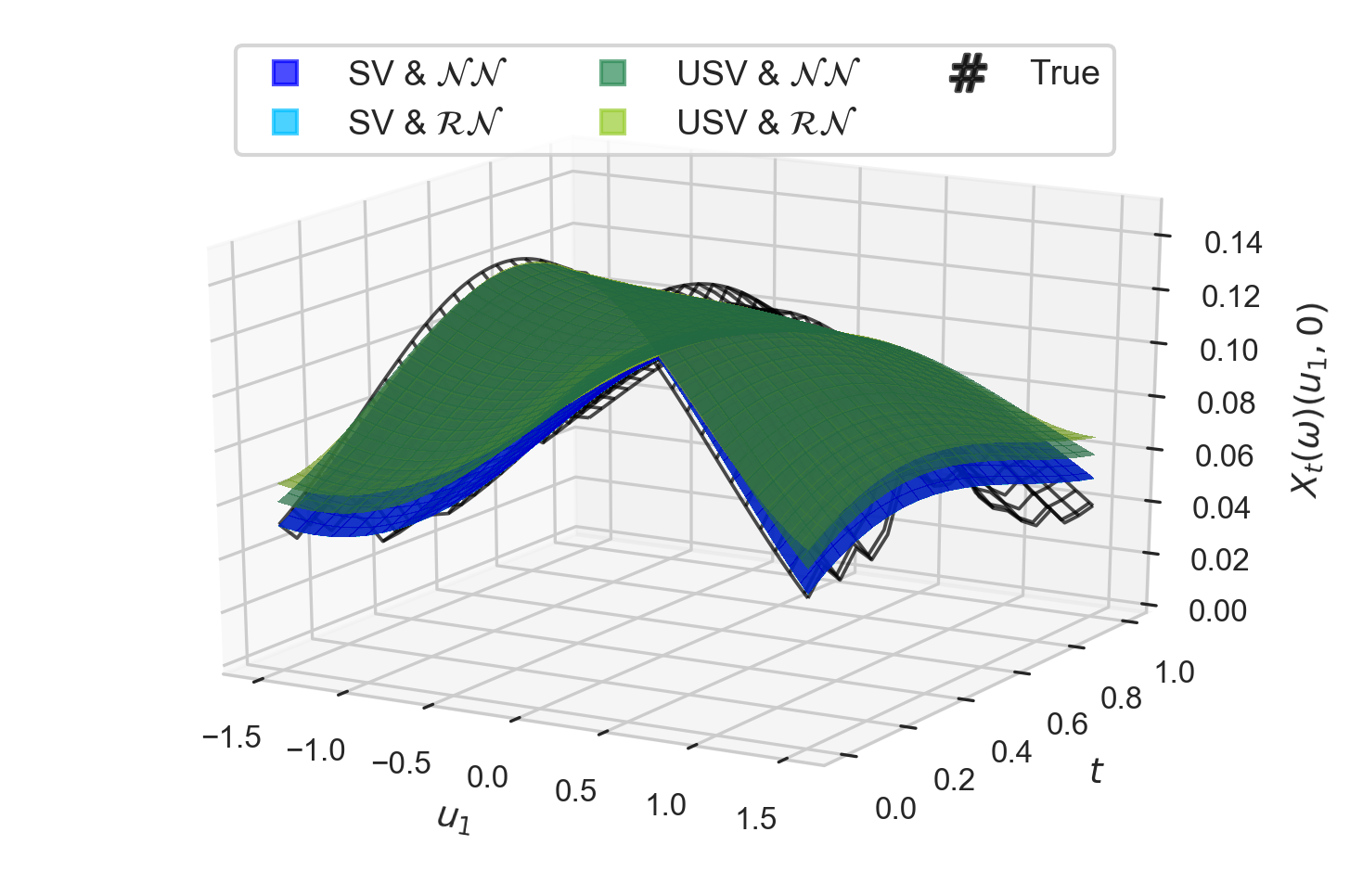}		
		\subcaption{Approximation for $J = 5$}
		\label{FigZakai2}
	\end{minipage}
	\caption{\small Learning the solution $X: [0,T] \times \Omega \rightarrow H$ of the Zakai equation~\eqref{EqDefZakai} with neural networks (label ``$\mathcal{NN}$'') and random neural networks (label ``$\mathcal{RN}$'') in the chaos expansion, either via supervised learning (Algorithm~\ref{AlgSV}; label ``SV'') or unsupervised learning (Algorithm~\ref{AlgUSV}; label ``USV''). In (\subref{FigZakai1}), the learning performance is displayed in terms of the out-of-sample (OOS) empirical error \eqref{EqDefLoss1} together with the computational time for different $J \in \lbrace 1,...,5 \rbrace$. In (\subref{FigZakai2}), the learned solutions $[0,T] \times \mathbb{R} \ni (t,u_1) \mapsto X^{(I,J,K)}_t(\omega)(u_1,0) \in \mathbb{R}$, with $(I,J,K) = (2m,5,1)$, are compared to the true solution $[0,T] \times \mathbb{R} \ni (t,u_1) \mapsto X_t(\omega)(u_1,0) \in \mathbb{R}$ obtained via the Monte-Carlo method in \cite{beck23} for some $\omega \in \Omega$ of the test set.}
	\label{FigZakai}
\end{figure}

\subsection{Conclusion}

Figure~\ref{FigStochHeat}-\ref{FigZakai} empirically demonstrate that deterministic and random neural networks in the truncated Wiener chaos are able to learn the solution of \eqref{EqDefSPDE}. In the supervised learning approach (see Algorithm~\ref{AlgSV}), the solution of \eqref{EqDefSPDE} is provided as a target, whereas the unsupervised learning approach (see Algorithm~\ref{AlgUSV}) recovers the solution by minimizing the (squared) difference of the left- and right-hand side of \eqref{EqDefSPDE}. The latter is naturally more challenging and results in a higher out-of-sample error. Moreover, random neural networks have a slightly lower approximation accuracy than deterministic neural networks but are computationally more efficient (see also \cite[Section~6]{neufeld23} for a comparison).

\section{Proofs}
\label{SecProofs}

\subsection{Proof of auxiliary results in Section~\ref{SecSPDE}}
\label{SecAuxProofsSPDE}

\begin{proof}[Proof of Proposition~\ref{PropSPDE}]
	For any $p \in [1,\infty)$, we first apply \cite[Theorem~7.2~(i)]{daprato14} to conclude that $X: [0,T] \times \Omega \rightarrow H$ exists, is unique up to modifications among the $\mathbb{F}$-predictable processes $\widetilde{X}: [0,T] \times \Omega \rightarrow H$ satisfying $\mathbb{P}\big[ \int_0^T \Vert \widetilde{X}_t \Vert_H^2 dt < \infty \big] = 1$, and admits a continuous modification.
	
	Now, for $p \in (2,\infty)$, we use \cite[Theorem~7.2~(iii)]{daprato14} (with constant $\widetilde{C}^{(p)}_{F,B,S,T} > 0$ depending only on $p \in (2,\infty)$, $C_{F,B} > 0$, $C_S := \sup_{t \in [0,T]} \Vert S_t \Vert_{L(H;H)} < \infty$, and $T > 0$) and that the initial condition $\chi_0 \in H$ is deterministic (see Assumption~\ref{AssSPDE}~\ref{AssSPDE5}) implying that $X_0 \in L^p(\Omega,\mathcal{F},\mathbb{P};H)$ to obtain that
	\begin{equation*}
		\begin{aligned}
			\mathbb{E}\left[ \sup_{t \in [0,T]} \Vert X_t \Vert_H^p \right] & \leq \widetilde{C}^{(p)}_{F,B,S,T} \left( 1 + \mathbb{E}\left[ \Vert X_0 \Vert_H^p \right] \right) \\
			& = \widetilde{C}^{(p)}_{F,B,S,T} \left( 1 + \Vert \chi_0 \Vert_H^p \right).
		\end{aligned}
	\end{equation*}
	On the other hand, for $p \in [1,2]$, we use Jensen's inequality, \cite[Theorem~7.2~(iii)]{daprato14} (with exponent $3$ and constant $C^{(3)}_{F,B,S,T} > 0$ depending only on $C_{F,B} > 0$, $C_S := \sup_{t \in [0,T]} \Vert S_t \Vert_{L(H;H)} < \infty$, and $T > 0$), that the initial condition $\chi_0 \in H$ is deterministic (see Assumption~\ref{AssSPDE}~\ref{AssSPDE5}) implying that $X_0 \in L^3(\Omega,\mathcal{F},\mathbb{P};H)$, and the inequality $(x+y)^{p/3} \leq \left( x^{p/3} + y^{p/3} \right)$ for any $x,y \geq 0$ to conclude that
	\begin{equation}
		\label{EqPropSPDEProof1}
		\begin{aligned}
			\mathbb{E}\left[ \sup_{t \in [0,T]} \Vert X_t \Vert_H^p \right] & \leq \mathbb{E}\left[ \sup_{t \in [0,T]} \Vert X_t \Vert_H^3 \right]^\frac{p}{3} \\
			& \leq \left( \widetilde{C}^{(3)}_{F,B,S,T} \left( 1 + \mathbb{E}\left[ \Vert X_0 \Vert_H^3 \right] \right) \right)^\frac{p}{3} \\
			& = \left( \widetilde{C}^{(3)}_{F,B,S,T} \right)^\frac{p}{3} \left( 1 + \Vert \chi_0 \Vert_H^3 \right)^\frac{p}{3} \\
			& \leq \left( \widetilde{C}^{(3)}_{F,B,S,T} \right)^\frac{p}{3} \left( 1 + \Vert \chi_0 \Vert_H^p \right).
		\end{aligned}
	\end{equation}
	Hence, defining the constant $C^{(p)}_{F,B,S,T} > 0$ by $C^{(p)}_{F,B,S,T} := \widetilde{C}^{(p)}_{F,B,S,T}$ if $p \in (2,\infty)$, and by $C^{(p)}_{F,B,S,T} := \big( \widetilde{C}^{(3)}_{F,B,S,T} \big)^{p/3}$ if $p \in [1,2]$, we obtain the result.
\end{proof}

\begin{lemma}
	\label{LemmaC0HSep}
	Let $(H,\langle \cdot, \cdot \rangle_H)$ be separable. Then, $(C^0([0,T];H),\Vert \cdot \Vert_{C^0([0,T];H)})$ is separable.
\end{lemma}
\begin{proof}
	By using that $(H,\langle \cdot, \cdot \rangle_H)$ is separable, there exists a sequence $(x_n)_{n \in \mathbb{N}} \subseteq H$ that is dense in $H$. From this, we define the countable set
	\begin{equation*}
		\Pol_\mathbb{Q}([0,T];H) := \left\lbrace [0,T] \ni t \mapsto \sum_{n=1}^N q_n(t) x_n \in H: \,
		\begin{matrix}
			N \in \mathbb{N}, \, q_1,...,q_N \in \Pol_\mathbb{Q}([0,T]) \\
			x_1,...,x_N \in H
		\end{matrix}
		\right\rbrace,
	\end{equation*}
	where $\Pol_\mathbb{Q}([0,T]) := \big\lbrace [0,T] \ni t \mapsto \sum_{k=0}^K c_k t^k \in \mathbb{R}: K \in \mathbb{N}, \, c_0,...,c_K \in \mathbb{Q} \big\rbrace$. Then, $\Pol_\mathbb{Q}([0,T];H)$ is a $\Pol_\mathbb{Q}([0,T])$-submodule (i.e.~for every $q \in \Pol_\mathbb{Q}([0,T])$ and $w \in \Pol_\mathbb{Q}([0,T];H)$ it holds that $q \cdot w \in \Pol_\mathbb{Q}([0,T];H)$). Moreover, $\Pol_\mathbb{Q}([0,T])$ is a subalgebra (i.e.~for every $q_1,q_2 \in \Pol_\mathbb{Q}([0,T])$ it holds that $q_1 + q_2 \in \Pol_\mathbb{Q}([0,T])$ and $q_1 \cdot q_2 \in \Pol_\mathbb{Q}([0,T])$) which is point separating (i.e.~for every distinct $t_1, t_2 \in [0,T]$ there exists some $q \in \Pol_\mathbb{Q}([0,T])$ with $q(t_1) \neq q(t_2)$) and nowhere vanishing (i.e.~for every $t_0 \in [0,T]$ there exists some $q \in \Pol_\mathbb{Q}([0,T])$ with $q(t_0) \neq 0$). In addition, for every $t \in [0,T]$, the set $\lbrace w(t): w \in \Pol_\mathbb{Q}([0,T];H) \rbrace \supseteq \lbrace x_n: n \in \mathbb{N} \rbrace$ is dense in $H$. Then, we can apply the vector-valued Stone-Weierstrass theorem in \cite[p.~103]{buck58} to conclude that $\Pol_\mathbb{Q}([0,T];H)$ is dense in $C^0([0,T];H)$. Since $\Pol_\mathbb{Q}([0,T];H)$ is countable, $(C^0([0,T];H),\Vert \cdot \Vert_{C^0([0,T];H)})$ is separable.
\end{proof}

\begin{proof}[Proof of Lemma~\ref{LemmaBMFourier}]
	For \ref{LemmaBMFourier1}, we fix some $j \in \mathbb{N}$ and $t \in [0,T]$. Then, by using the definition of $\xi_{i,j}$ in \eqref{EqDefXi}, that $\int_0^t g_j(s) ds = \langle \mathds{1}_{[0,t]}, g_j \rangle_{L^2([0,T],\mathcal{B}([0,T]),dt)}$, Ito's isometry, and that $(g_j)_{i \in \mathbb{N}}$ is a complete orthonormal basis of $(L^2([0,T],\mathcal{B}([0,T]),dt),\langle\cdot,\cdot\rangle_{L^2([0,T],\mathcal{B}([0,T]),dt)})$, it follows that
	\begin{equation*}
		\begin{aligned}
			& \left\Vert W^{(i)}_t - \sum_{j=1}^J \xi_{i,j} \int_0^t g_j(s) ds \right\Vert_{L^2(\Omega,\mathcal{F},\mathbb{P})} \\
			& \quad\quad = \mathbb{E}\left[ \left\vert W^{(i)}_t - \sum_{j=1}^J \xi_{i,j} \int_0^t g_j(s) ds \right\vert^2 \right]^\frac{1}{2} \\
			& \quad\quad = \mathbb{E}\left[ \left\vert \int_0^T \mathds{1}_{[0,t]}(s) dW^{(i)}_s - \sum_{j=1}^J \left( \int_0^T g_j(s) dW^{(i)}_s \right) \langle \mathds{1}_{[0,t]}, g_j \rangle_{L^2([0,T],\mathcal{B}([0,T]),dt)} \right\vert^2 \right]^\frac{1}{2} \\
			& \quad\quad = \mathbb{E}\left[ \left\vert \int_0^T \left( \mathds{1}_{[0,t]}(s) - \sum_{j=1}^J \langle \mathds{1}_{[0,t]}, g_j \rangle_{L^2([0,T],\mathcal{B}([0,T]),dt)} g_j(s) \right) dW^{(i)}_s \right\vert^2 \right]^\frac{1}{2} \\
			& \quad\quad = \left( \int_0^T \left( \mathds{1}_{[0,t]}(s) - \sum_{j=1}^J \langle \mathds{1}_{[0,t]}, g_j \rangle_{L^2([0,T],\mathcal{B}([0,T]),dt)} g_j(s) \right)^2 ds \right)^\frac{1}{2} \\
			& \quad\quad = \left\Vert \mathds{1}_{[0,t]} - \sum_{j=1}^J \langle \mathds{1}_{[0,t]}, g_j \rangle_{L^2([0,T],\mathcal{B}([0,T]),dt)} g_j \right\Vert_{L^2([0,T],\mathcal{B}([0,T]),dt)} \quad \overset{J \rightarrow \infty}{\longrightarrow} \quad 0.
		\end{aligned}
	\end{equation*}
	For \ref{LemmaBMFourier2}, we use Minkowski's inequality, that $(e_i)_{i \in \mathbb{N}}$ is an orthonormal basis of $(Z,\langle \cdot, \cdot \rangle_Z)$, that \mbox{$\sum_{i=1}^\infty \lambda_i < \infty$}, that $W^{(i)}_t - \sum_{j=1}^J \xi_{i,j} \int_0^t g_j(s) ds \sim W^{(1)}_t - \sum_{j=1}^J \xi_{1,j} \int_0^t g_j(s) ds$ are identically distributed for any $i \in \mathbb{N}$, Lemma~\ref{LemmaBM}, and \ref{LemmaBMFourier1} to conclude that
	\begin{equation*}
		\begin{aligned}
			& \left\Vert W_t - \sum_{i=1}^I \sum_{j=1}^J \sqrt{\lambda_i} \xi_{i,j} \left( \int_0^t g_j(s) ds \right) e_i \right\Vert_{L^2(\Omega,\mathcal{F},\mathbb{P};Z)} \\
			& \quad\quad = \mathbb{E}\left[ \left\Vert W_t - \sum_{i=1}^I \sum_{j=1}^J \sqrt{\lambda_i} \xi_{i,j} \left( \int_0^t g_j(s) ds \right) e_i \right\Vert_Z^2 \right]^\frac{1}{2} \\
			& \quad\quad \leq \mathbb{E}\left[ \left\Vert W_t - \sum_{i=1}^I \sqrt{\lambda_i} W^{(i)}_t e_i \right\Vert_Z^2 \right]^\frac{1}{2} + \mathbb{E}\left[ \left\Vert \sum_{i=1}^I \sqrt{\lambda_i} W^{(i)}_t e_i - \sum_{i=1}^I \sum_{j=1}^J \sqrt{\lambda_i} \xi_{i,j} \left( \int_0^t g_j(s) ds \right) e_i \right\Vert_Z^2 \right]^\frac{1}{2} \\
			& \quad\quad = \mathbb{E}\left[ \left\Vert W_t - \sum_{i=1}^I \sqrt{\lambda_i} W^{(i)}_t e_i \right\Vert_Z^2 \right]^\frac{1}{2} + \sum_{i=1}^I \lambda_i \mathbb{E}\left[ \left\vert W^{(i)}_t - \sum_{j=1}^J \xi_{i,j} \int_0^t g_j(s) ds \right\vert^2 \right]^\frac{1}{2} \\
			& \quad\quad \leq \mathbb{E}\left[ \left\Vert W_t - \sum_{i=1}^I \sqrt{\lambda_i} W^{(i)}_t e_i \right\Vert_Z^2 \right]^\frac{1}{2} + \left( \sum_{i=1}^\infty \lambda_i \right) \mathbb{E}\left[ \left\vert W^{(1)}_t - \sum_{j=1}^J \xi_{1,j} \int_0^t g_j(s) ds \right\vert^2 \right]^\frac{1}{2} \\
			& \quad\quad\quad \overset{I,J \rightarrow \infty}{\longrightarrow} \quad 0,
		\end{aligned}
	\end{equation*}
	which completes the proof.
\end{proof}

\begin{proof}[Proof of Lemma~\ref{LemmaWick}]
	Fix some $p \in [1,\infty)$. Moreover, we define the map
	\begin{equation*}
		L^2([0,T];L_2(Z_0;\mathbb{R})) \ni \psi \quad \mapsto \quad \mathscr{W}(\psi) := \int_0^T \psi(t) dW_t \in L^2(\Omega,\mathcal{F}_T,\mathbb{P})
	\end{equation*}
	returning the stochastic integral of a deterministic function $\psi: [0,T] \rightarrow L_2(Z_0;\mathbb{R})$ at terminal time $T > 0$. Then, by using Ito's isometry in \cite[Proposition~4.28]{daprato14} (with\footnote{The \emph{trace} of a nuclear operator $S \in L_1(Z;Z)$ is defined as $\trace(S) := \sum_{i=1}^\infty \langle S e_i, e_i \rangle_Z$.} $\trace\big( \big( \Xi_1 Q^{1/2} \big) \big( \Xi_2 Q^{1/2} \big)^* \big) = \langle \Xi_1, \Xi_2 \rangle_{L_2(Z_0;H)}$ for any $\Xi_1,\Xi_2 \in L_2(Z_0;H)$, see \cite[Appendix~C]{daprato14}), it follows for every $\psi,\phi \in L^2([0,T];L_2(Z_0;\mathbb{R}))$ that
	\begin{equation*}
		\mathbb{E}[\mathscr{W}(\psi)] = \mathbb{E}\left[ \int_0^T \psi(t) dW_t \right] = 0
	\end{equation*}
	and that
	\begin{equation*}
		\begin{aligned}
			\mathbb{E}[\mathscr{W}(\psi) \mathscr{W}(\phi)] & = \mathbb{E}\left[ \left( \int_0^T \psi(t) dW_t \right) \left( \int_0^T \phi(t) dW_t \right) \right] \\
			& = \mathbb{E}\left[ \int_0^T \langle \psi(t), \phi(t) \rangle_{L_2(Z_0;\mathbb{R})} dt \right] \\
			& = \langle \psi, \phi \rangle_{L^2([0,T];L_2(Z_0;\mathbb{R}))},
		\end{aligned}
	\end{equation*}
	which shows that $\mathscr{W}$ is an isonormal process in the sense of \cite[Definition~1.1.1]{nualart06}. Hence, by using \cite[Exercise~1.1.7]{nualart06} (see also \cite[Theorem~2.2.4~(i)]{nourdin12}), we conclude that 
	\begin{equation}
		\label{EqLemmaWickProof1}
		\left\lbrace q\left( \mathscr{W}(\psi_1),...,\mathscr{W}(\psi_n) \right): n \in \mathbb{N}, \, q \in \Pol(\mathbb{R}^n), \, \psi_1,...,\psi_n \in L^2([0,T];L_2(Z_0;\mathbb{R})) \right\rbrace
	\end{equation}
	is dense in $L^p(\Omega,\mathcal{F}_T,\mathbb{P})$, where $\Pol(\mathbb{R}^n)$ consists of polynomials of the form $\mathbb{R}^n \ni x := (x_1,...,x_n)^\top \mapsto \sum_{\mathbf{k} \in \mathbb{N}^n_{0,K}} a_\mathbf{k} \prod_{l=1}^n x_l^{k_l} \in \mathbb{R}$ for some $N \in \mathbb{N}$ and $(a_\mathbf{k})_{\mathbf{k} \in \mathbb{N}^n_{0,K}} \subseteq \mathbb{R}$, where $\mathbf{k} := (k_1,...,k_n) \in \mathbb{N}^n_{0,K}$. 
	
	Now, we show that $\linspan\lbrace \xi_\alpha: \alpha \in \mathcal{J} \rbrace$ is also dense in $L^p(\Omega,\mathcal{F}_T,\mathbb{P})$. Let us define for every $i,j \in \mathbb{N}$ the deterministic function $\big( t \mapsto \phi_{i,j}(t) := \left( z \mapsto g_j(t) \langle \widetilde{e}_i, z \rangle_{Z_0} \right) \big) \in L^2([0,T];L_2(Z_0;\mathbb{R}))$ satisfying $\xi_{i,j} = \int_0^T g_j(t) dW^{(i)}_t = \lambda_i^{-1/2} \int_0^T g_j(t) \langle e_i, dW_t \rangle_Z = \int_0^T \phi_{i,j}(t) dW_t$, where $(\widetilde{e}_i)_{i \in \mathbb{N}} := \big( Q^{1/2} e_i \big)_{i \in \mathbb{N}}$ is an orthonormal basis of $(Z_0,\langle \cdot,\cdot \rangle_{Z_0})$. Since $\xi_\alpha = \frac{1}{\sqrt{\alpha!}} \prod_{i,j=1}^\infty h_{\alpha_{i,j}}(\xi_{i,j}) = \frac{1}{\sqrt{\alpha!}} \prod_{i,j=1}^\infty h_{\alpha_{i,j}}(\mathscr{W}(\phi_{i,j}))$ for any $\alpha \in \mathcal{J}$ and $\linspan\left\lbrace \mathbb{R}^n \ni (s_1,...,s_n) \mapsto \prod_{i=1}^n h_{k_i}(s_i) \in \mathbb{R}: k_1,...,k_n \in \mathbb{N}_0 \right\rbrace = \Pol(\mathbb{R}^n)$ for any $n \in \mathbb{N}$, it therefore suffices to show that
	\begin{equation}
		\label{EqLemmaWickProof2}
		\left\lbrace q\left( \mathscr{W}(\phi_{i_1,j_1}),...,\mathscr{W}(\phi_{i_n,j_n}) \right): n \in \mathbb{N}, \, q \in \Pol(\mathbb{R}^n), \, (i_1,j_1),...,(i_n,j_n) \in \mathbb{N}^2 \right\rbrace
	\end{equation}
	is dense in $L^p(\Omega,\mathcal{F}_T,\mathbb{P})$. To this end, we fix some $Z \in L^p(\Omega,\mathcal{F}_T,\mathbb{P})$ and $\varepsilon \in (0,1)$. Then, by using that \eqref{EqLemmaWickProof1} is dense in $L^p(\Omega,\mathcal{F}_T,\mathbb{P})$ there exists some $n \in \mathbb{N}$, $q \in \Pol(\mathbb{R}^n)$, and $\psi_1,...,\psi_n \in L^2([0,T];L_2(Z_0;\mathbb{R}))$ such that
	\begin{equation}
		\label{EqLemmaWickProof3}
		\mathbb{E}\left[ \vert Z - q\left( \mathscr{W}(\psi_1),...,\mathscr{W}(\psi_n) \right) \vert^p \right]^\frac{1}{p} < \frac{\varepsilon}{2}.
	\end{equation}
	Moreover, for every fixed $l = 1,...,n$, we use that the system $(\phi_{i,j})_{i,j \in \mathbb{N}}$ forms an orthonormal basis of $(L^2([0,T];L_2(Z_0;\mathbb{R})),\langle\cdot,\cdot\rangle_{L^2([0,T];L_2(Z_0;\mathbb{R}))})$ to conclude that there exist some $I_l \in \mathbb{N}$ and $J_l \in \mathbb{N}$ such that $\phi_l := \sum_{i=1}^{I_l} \sum_{j=1}^{J_l} c_{l,i,j} \phi_{i,j} \in \linspan\left\lbrace \phi_{i,j}: i,j \in \mathbb{N} \right\rbrace$ (with $c_{l,i,j} := \langle \psi_l, \phi_{i,j} \rangle_{L^2([0,T];L_2(Z_0;\mathbb{R}))}$) satisfies
	\begin{equation}
		\label{EqLemmaWickProof6}
		\left\Vert \psi_l - \phi_l \right\Vert_{L^2([0,T];L_2(Z_0;\mathbb{R}))} < (2C_\psi)^{-1} C_{N(n+1)p}^{-\frac{1}{(n+1)p}} \varepsilon,
	\end{equation}
	where $C_\psi := 1 + \sum_{\beta \in \mathbb{N}^n_{0,N}} \mathbb{E}\left[ \left\vert \partial_\beta q\left( \mathscr{W}(\psi_1),...,\mathscr{W}(\psi_n) \right) \right\vert^{(n+1)p} \right]^{1/((n+1)p)} \geq 1$, and where $C_{N(n+1)p} > 0$ is the constant in \cite[Theorem~4.36]{daprato14} (with exponent $N(n+1)p \geq 2$). Hence, for every $k = 1,...,N$, we use \cite[Theorem~4.36]{daprato14} (with exponent $N(n+1)p \geq 2$ and constant $C_{k(n+1)p} \geq 1$, which is increasing in $k \in \mathbb{N}$), the inequality \eqref{EqLemmaWickProof6}, that $2C_\psi \geq 1$, that $C_{k(n+1)p} \geq 1$, and that $\varepsilon \in (0,1)$ to conclude that
	\begin{equation}
		\label{EqLemmaWickProof5}
		\begin{aligned}
			\mathbb{E}\left[ \left\vert \mathscr{W}(\phi_l-\psi_l) \right\vert^{k(n+1)p} \right]^\frac{1}{(n+1)p} & = \mathbb{E}\left[ \left\vert \int_0^T \left( \phi_l(t)-\psi_l(t) \right) dW_t \right\vert^{k(n+1)p} \right]^\frac{1}{(n+1)p} \\
			& \leq C_{k(n+1)p}^\frac{1}{(n+1)p} \left( \mathbb{E}\left[ \int_0^T \Vert \phi_l(t)-\psi_l(t) \Vert_{L_2(Z_0;\mathbb{R})}^2 dt \right] \right)^\frac{k}{2} \\
			& \leq C_{N(n+1)p}^\frac{1}{(n+1)p} \left\Vert \psi_l - \phi_l \right\Vert_{L^2([0,T];L_2(Z_0;\mathbb{R}))}^k \\
			& < C_{N(n+1)p}^\frac{1}{(n+1)p} \left( (2C_\psi)^{-1} C_{N(n+1)p}^{-\frac{1}{(n+1)p}} \varepsilon \right)^k \\
			& \leq (2C_\psi)^{-1} \varepsilon.
		\end{aligned}
	\end{equation}
	Thus, by using Minkowski's inequality, the inequality \eqref{EqLemmaWickProof3}, Taylor's theorem, i.e.~that $q(y) - q(x) = \sum_{\beta \in \mathbb{N}^n_{0,N}, \, \vert\beta\vert \geq 1} \frac{\partial_\beta q(x)}{\beta!} \prod_{l=1}^N (y_l-x_l)^{\beta_l}$ for any $x := (x_1,...,x_n)^\top, y := (y_1,...,y_n)^\top \in \mathbb{R}^n$, together with Minkowski's inequality, the generalized H\"older's inequality (with exponents $\frac{1}{(n+1)p} + \sum_{l=1}^n \frac{1}{(n+1)p} = \frac{1}{p}$), the inequality \eqref{EqLemmaWickProof5}, that $2C_\psi \geq 1$, and that $\varepsilon \in (0,1)$, it follows for $q\left( \mathscr{W}(\phi_1),...,\mathscr{W}(\phi_n) \right)$ belonging to \eqref{EqLemmaWickProof2} (as $\mathscr{W}(\phi_l) = \sum_{i=1}^{I_l} \sum_{j=1}^{J_l} c_{l,i,j} \mathscr{W}(\phi_{i,j})$ for any $l = 1,...,n$) that
	\begin{equation*}
		\begin{aligned}
			& \mathbb{E}\left[ \vert Z - q\left( \mathscr{W}(\phi_1),...,\mathscr{W}(\phi_n) \right) \vert^p \right]^\frac{1}{p} \\
			& \quad\quad \leq \mathbb{E}\left[ \vert Z - q\left( \mathscr{W}(\psi_1),...,\mathscr{W}(\psi_n) \right) \vert^p \right]^\frac{1}{p} + \mathbb{E}\left[ \vert q\left( \mathscr{W}(\psi_1),...,\mathscr{W}(\psi_n) \right) - q\left( \mathscr{W}(\phi_1),...,\mathscr{W}(\phi_n) \right) \vert^p \right]^\frac{1}{p} \\
			& \quad\quad < \mathbb{E}\left[ \vert Z - q\left( \mathscr{W}(\psi_1),...,\mathscr{W}(\psi_n) \right) \vert^p \right]^\frac{1}{p} \\
			& \quad\quad\quad\quad + \sum_{\beta \in \mathbb{N}^n_{0,N} \atop \vert\beta\vert \geq 1} \mathbb{E}\left[ \left\vert \frac{\partial_\beta q\left( \mathscr{W}(\psi_1),...,\mathscr{W}(\psi_n) \right)}{\beta!} \prod_{l=1}^n (\mathscr{W}(\phi_l)-\mathscr{W}(\psi_l))^{\beta_l} \right\vert^p \right]^\frac{1}{p} \\
			& \quad\quad < \mathbb{E}\left[ \vert Z - q\left( \mathscr{W}(\psi_1),...,\mathscr{W}(\psi_n) \right) \vert^p \right]^\frac{1}{p} \\
			& \quad\quad\quad\quad + \sum_{\beta \in \mathbb{N}^n_{0,N} \atop \vert\beta\vert \geq 1} \frac{1}{\beta!} \mathbb{E}\left[ \left\vert \partial_\beta q\left( \mathscr{W}(\psi_1),...,\mathscr{W}(\psi_n) \right) \right\vert^{(n+1)p} \right]^\frac{1}{(n+1)p} \prod_{l=1}^n \mathbb{E}\left[ \left\vert \mathscr{W}(\phi_l-\psi_l) \right\vert^{\beta_l (n+1) p} \right]^\frac{1}{(n+1)p} \\
			& \quad\quad < \mathbb{E}\left[ \vert Z - q\left( \mathscr{W}(\psi_1),...,\mathscr{W}(\psi_n) \right) \vert^p \right]^\frac{1}{p} + C_\psi \max_{\beta \in \mathbb{N}^n_{0,N} \atop \vert\beta\vert \geq 1} \prod_{l=1}^n \left( (2C_\psi)^{-1} \varepsilon \right)^{\beta_l} \\
			& \quad\quad \leq \frac{\varepsilon}{2} + C_\psi (2 C_\psi)^{-1} \varepsilon = \varepsilon.
		\end{aligned}
	\end{equation*}
	Since $Z \in L^p(\Omega,\mathcal{F}_T,\mathbb{P})$ and $\varepsilon > 0$ were chosen arbitrarily, this shows that \eqref{EqLemmaWickProof2} is dense in $L^p(\Omega,\mathcal{F}_T,\mathbb{P})$ and therefore that $\linspan\lbrace \xi_\alpha: \alpha \in \mathcal{J} \rbrace$ is also dense in $L^p(\Omega,\mathcal{F}_T,\mathbb{P})$. Finally, for $p = 2$, the Wick polynomials $(\xi_\alpha)_{\alpha \in \mathcal{J}}$ are by \cite[Proposition~1.1.1]{nualart06} a complete orthonormal basis of $L^2(\Omega,\mathcal{F}_T,\mathbb{P})$.
\end{proof}

\subsection{Proof of Theorem~\ref{ThmCameronMartin}}
\label{SecProofsCameronMartin}

\begin{proof}[Proof of Theorem~\ref{ThmCameronMartin}]
	For \eqref{EqThmCameronMartin1}, we fix some $\varepsilon > 0$. Then, by using the continuous modification, Proposition~\ref{PropSPDE}, and that the Brownian motion $W: [0,T] \times \Omega \rightarrow Z$ is the only random force driving \eqref{EqDefSPDE}, we observe that $X \in L^p(\Omega,\mathcal{F}_T,\mathbb{P};C^0([0,T];H))$. Moreover, by combining \cite[Lemma~1.2.19~(i)]{hytoenen16} with Lemma~\ref{LemmaWick}, the linear span of $\left\lbrace \Omega \ni \omega \mapsto \xi_\alpha(\omega) z \in C^0([0,T];H): \alpha \in \mathcal{J}, \, z \in C^0([0,T];H) \right\rbrace$ is dense in $L^p(\Omega,\mathcal{F}_T,\mathbb{P};C^0([0,T];H))$. Hence, there exist some $I,J,K \in \mathbb{N}$ and $(x_\alpha)_{\alpha \in \mathcal{J}_{I,J,K}} \subseteq C^0([0,T];H)$ such that
	\begin{equation*}
		\begin{aligned}
			\mathbb{E}\left[ \sup_{t \in [0,T]} \left\Vert X_t - \sum_{\alpha \in \mathcal{J}_{I,J,K}} x_\alpha(t) \xi_\alpha \right\Vert^p \right]^\frac{1}{p} = \mathbb{E}\left[ \left\Vert X - \sum_{\alpha \in \mathcal{J}_{I,J,K}} x_\alpha \xi_\alpha \right\Vert_{C^0([0,T];H)}^p \right]^\frac{1}{p} < \varepsilon.
		\end{aligned}
	\end{equation*}
	which shows the conclusion in \eqref{EqThmCameronMartin1}.
	
	For \eqref{EqThmCameronMartin2}, we fix an orthonormal basis $(y_n)_{n \in \mathbb{N}}$ of $(H,\langle \cdot, \cdot \rangle_H)$. Then, we claim that $(\xi_\alpha y_n)_{(\alpha,n) \in \mathcal{J} \times \mathbb{N}}$ is an orthonormal basis of the Bochner space $L^2(\Omega,\mathcal{F}_T,\mathbb{P};H)$, where $\langle F, G \rangle_{L^2(\Omega,\mathcal{F},\mathbb{P};H)} := \mathbb{E}[\langle F,G \rangle_H]$ is the inner product. Indeed, by using that $(\xi_\alpha)_{\alpha \in \mathcal{J}}$ are by Lemma~\ref{LemmaWick} orthonormal in $L^2(\Omega,\mathcal{F}_T,\mathbb{P})$ and that $(y_n)_{n \in \mathbb{N}}$ are by assumption orthonormal in $H$, it follows for every $\alpha,\beta \in \mathcal{J}$ and $l,m \in \mathbb{N}$ that
	\begin{equation*}
		\langle \xi_\alpha y_n, \xi_\beta y_m \rangle_{L^2(\Omega,\mathcal{F},\mathbb{P};H)} = \mathbb{E}\left[ \langle \xi_\alpha y_n , \xi_\beta y_m \rangle_H \right] = \mathbb{E}[\xi_\alpha \xi_\beta] \langle y_n, y_m \rangle_H = \delta_{\alpha,\beta} \delta_{\beta,m} = \delta_{(\alpha,n),(\beta,m)}.
	\end{equation*}
	Moreover, by combining \cite[Lemma~1.2.19~(i)]{hytoenen16} with Lemma~\ref{LemmaWick}, we conclude that the linear span of $\left\lbrace \Omega \ni \omega \mapsto \xi_\alpha(\omega) y_n \in H: (\alpha,n) \in \mathcal{J} \times \mathbb{N} \right\rbrace$ is dense in $L^2(\Omega,\mathcal{F}_T,\mathbb{P};H)$. This shows that $(\xi_\alpha y_n)_{(\alpha,n) \in \mathcal{J} \times \mathbb{N}}$ is an orthonormal basis of $L^2(\Omega,\mathcal{F}_T,\mathbb{P};H)$. Hence, the chaos expansion \eqref{EqThmCameronMartin2} follows as an orthogonal expansion along the basis elements $(\xi_\alpha y_n)_{(\alpha,n) \in \mathcal{J} \times \mathbb{N}}$ of $L^2(\Omega,\mathcal{F}_T,\mathbb{P};H)$.
\end{proof}

\subsection{Proof of auxiliary results in Section~\ref{SecUAT}}
\label{SecAuxProofsUAT}

\begin{lemma}
	\label{LemmaNNWellDef}
	Let $(H,\langle \cdot, \cdot \rangle_H)$ satisfy Assumption~\ref{AssHilbert}. Then, the time-extended restriction map
	\begin{equation}
		\label{EqLemmaWellDefProof1}
		(C^k_b(\mathbb{R} \times \mathbb{R}^m;\mathbb{R}^d),\Vert \cdot \Vert_{C^k_{pol,\gamma}(\mathbb{R} \times \mathbb{R}^m;\mathbb{R}^d)}) \ni f \,\, \mapsto \,\, \left( t \mapsto f(t,\cdot)\vert_U \right) \in (C^0([0,T];H),\Vert \cdot \Vert_{C^0([0,T];H)})
	\end{equation}
	is a continuous dense embedding. Moreover, for every $\rho \in \overline{C^k_b(\mathbb{R})}^\gamma$ and $\varphi \in \mathcal{NN}^\rho_{[0,T] \times U,d}$, it holds that $(t \mapsto \varphi(t,\cdot)) \in C^0([0,T];H)$.
\end{lemma}
\begin{proof}
	In order to show that \eqref{EqLemmaWellDefProof1} is a continuous embedding, we observe that for every $f \in C^k_b(\mathbb{R} \times \mathbb{R}^m;\mathbb{R}^d)$ it holds that
	\begin{equation*}
		\begin{aligned}
			\Vert f \Vert_{C^0([0,T];H)} & = \sup_{t \in [0,T]} \Vert f(t,\cdot) \Vert_H \\
			& \leq C_H \sup_{t \in [0,T]} \Vert f(t,\cdot) \Vert_{C^k_{pol,\gamma}(\mathbb{R}^m;\mathbb{R}^d)} \\
			& = C_H \sup_{t \in [0,T]} \max_{\beta \in \mathbb{N}^m_{0,k}} \sup_{u \in \mathbb{R}^m} \frac{\Vert \partial_{(0,\beta)} f(t,u) \Vert}{(1+\Vert u \Vert)^\gamma} \\
			& \leq C_H \max_{\beta \in \mathbb{N}^{m+1}_{0,k}} \sup_{(t,u) \in \mathbb{R} \times \mathbb{R}^m} \frac{\Vert \partial_\beta f(t,u) \Vert}{(1+\Vert (t,u) \Vert)^\gamma} \\
			& = C_H \Vert f \Vert_{C^k_{pol,\gamma}(\mathbb{R} \times \mathbb{R}^m;\mathbb{R}^d)}.
		\end{aligned}
	\end{equation*}
	Next, for showing that \eqref{EqLemmaWellDefProof1} is a dense embedding, we fix some $f \in C^0([0,T];H)$ and $\varepsilon > 0$. Moreover, we define the collection $(V_t)_{t \in [0,T]}$ of open subsets $V_t := \left\lbrace s \in [0,T]: \Vert f(t,\cdot) - f(s,\cdot) \Vert_H < \varepsilon/2 \right\rbrace \subseteq [0,T]$ for $t \in [0,T]$, which forms an open cover of $[0,T]$. Then, by using that $[0,T]$ is compact, there exists some $N \in \mathbb{N}$ and $t_1,...,t_N \in [0,T]$ such that $(V_{t_n})_{n=1,...,N}$ is a finite subcover of $[0,T]$. In addition, for every $n = 1,...,N$, we use that the restriction map \eqref{EqAssHilbert1} is a dense embedding (see Assumption~\ref{AssHilbert}) to conclude that there exists some $g_n \in C^k_b(\mathbb{R}^m;\mathbb{R}^d)$ such that $\Vert f(t_n,\cdot) - g_n\vert_U \Vert_H < \varepsilon/2$. Thus, by using a partition of unity $(\eta_n)_{n=1,...,N} \subseteq C^0([0,T])$ subordinate to $(V_{t_n})_{n=1,...,N}$ (i.e.~functions $(\eta_n)_{n=1,...,N} \subseteq C^0([0,T])$ such that $0 \leq \eta_n(t) \leq 1$ for all $t \in [0,T]$ and $n = 1,...,N$, that $\sum_{n=1}^N \eta_n(t) = 1$ for all $t \in [0,T]$, and that $\supp(\eta_n) := \overline{\lbrace t \in [0,T]: \eta_n(t) \neq 0 \rbrace} \subseteq V_{t_n}$ for all $n = 1,...,N$), we conclude for the function $\big( t \mapsto g(t) := \sum_{n=1}^N \eta_n(t) g_n \big) \in C^0([0,T];H)$ that
	\begin{equation*}
		\begin{aligned}
			& \Vert f - g\vert_{[0,T] \times U} \Vert_{C^0([0,T];H)} \\
			& \quad\quad = \sup_{t \in [0,T]} \left\Vert f(t,\cdot) - \sum_{n=1}^N \eta_n(t) g_n\vert_U \right\Vert_H \\
			& \quad\quad \leq \sup_{t \in [0,T]} \left\Vert f(t,\cdot) - \sum_{n=1}^N \eta_n(t) f(t_n,\cdot) \right\Vert_H + \sup_{t \in [0,T]} \left\Vert \sum_{n=1}^N \eta_n(t) f(t_n,\cdot) - \sum_{n=1}^N \eta_n(t) g_n\vert_U \right\Vert_H \\
			& \quad\quad \leq \sup_{t \in [0,T]} \left\Vert \sum_{n=1}^N \eta_n(t) (f(t,\cdot) - f(t_n,\cdot)) \right\Vert_H + \sup_{t \in [0,T]} \left\Vert \sum_{n=1}^N \eta_n(t) (f(t_n,\cdot) - g_n\vert_U) \right\Vert_H \\
			& \quad\quad \leq \sup_{t \in [0,T]} \sum_{n=1}^N \eta_n(t) \underbrace{\left\Vert f(t,\cdot) - f(t_n,\cdot) \right\Vert_H}_{< \varepsilon/2} + \sup_{t \in [0,T]} \sum_{n=1}^N \eta_n(t) \underbrace{\left\Vert f(t_n,\cdot) - g_n\vert_U \right\Vert_H}_{< \varepsilon/2} \\
			& \quad\quad < \frac{\varepsilon}{2} \sup_{t \in [0,T]} \sum_{n=1}^N \eta_n(t) + \frac{\varepsilon}{2} \sup_{t \in [0,T]} \sum_{n=1}^N \eta_n(t) \\
			& \quad\quad = \frac{\varepsilon}{2} + \frac{\varepsilon}{2} = \varepsilon.
		\end{aligned}
	\end{equation*}
	Since $f \in C^0([0,T];H)$ and $\varepsilon > 0$ were chosen arbitrarily, this shows that \eqref{EqLemmaWellDefProof1} is a dense embedding.
	
	Finally, we use the previous step to conclude that $C^0([0,T];H)$ satisfies the conditions\footnote{\label{FootnoteOpen0T}To be precise, if $k \geq 1$, we need to consider $C^0((0,T);H)$ instead of $C^0([0,T];H)$ in order to have an open set $(0,T) \times U$ for \cite[Definition~2.3]{neufeld24}. However, by continuous extension, we observe that \eqref{EqLemmaWellDefProof1} is a continuous dense embedding if and only if $(C^k_b(\mathbb{R} \times \mathbb{R}^m;\mathbb{R}^d),\Vert \cdot \Vert_{C^k_{pol,\gamma}(\mathbb{R} \times \mathbb{R}^m;\mathbb{R}^d)}) \ni f \mapsto (t \mapsto f(t,\cdot)\vert_U) \in (C^0((0,T);H),\Vert \cdot \Vert_{C^0((0,T);H)})$ is a continuous dense embedding.} of \cite[Definition~2.3]{neufeld24}. Hence, by applying \cite[Lemma~2.5]{neufeld24}, it follows for every $\rho \in \overline{C^k_b(\mathbb{R})}^\gamma$ and $\varphi \in \mathcal{NN}^\rho_{[0,T] \times U,d}$ that $(t \mapsto \varphi(t,\cdot)) \in C^0([0,T];H)$.
\end{proof}

\subsection{Proof of Theorem~\ref{ThmUAT}}
\label{SecProofUAT}

\begin{proof}[Proof of Theorem~\ref{ThmUAT}]
	Fix some $\varepsilon > 0$. Then, by using Theorem~\ref{ThmCameronMartin}, there exists some $I,J,K \in \mathbb{N}$ and $(x_\alpha)_{\alpha \in \mathcal{J}_{I,J,K}} \subseteq C^0([0,T];H)$ such that
	\begin{equation}
		\label{EqThmUATProof1}
		\mathbb{E}\left[ \sup_{t \in [0,T]} \left\Vert X_t - \sum_{\alpha \in \mathcal{J}_{I,J,K}} x_\alpha(t) \xi_\alpha \right\Vert_H^p \right]^\frac{1}{p} < \frac{\varepsilon}{2}.
	\end{equation}
	Moreover, by using Lemma~\ref{LemmaNNWellDef}, we observe that $C^0([0,T];H)$ satisfies the conditions$^\text{\ref{FootnoteOpen0T}}$ of \cite[Definition~2.3]{neufeld24}. Hence, we can apply the universal approximation result for deterministic neural networks in \cite[Theorem~2.8]{neufeld24} to conclude that for every $\alpha \in \mathcal{J}_{I,J,K}$ there exists some $\varphi_\alpha \in \mathcal{NN}^\rho_{[0,T] \times U,d}$ such that
	\begin{equation}
		\label{EqThmUATProof2}
		\left\Vert x_\alpha - \varphi_\alpha \right\Vert_{C^0([0,T];H)} = \sup_{t \in [0,T]} \left\Vert x_\alpha(t) - \varphi_\alpha(t,\cdot) \right\Vert_H < \frac{\varepsilon}{2 \vert \mathcal{J}_{I,J,K} \vert \mathbb{E}\left[ \vert \xi_\alpha \vert^p \right]^\frac{1}{p}}.
	\end{equation}
	Thus, by combining \eqref{EqThmUATProof1} and \eqref{EqThmUATProof2} with Minkowski's inequality, it follows that
	\begin{equation*}
		\begin{aligned}
			& \mathbb{E}\left[ \sup_{t \in [0,T]} \left\Vert X_t - \sum_{\alpha \in \mathcal{J}_{I,J,K}} \varphi_\alpha(t) \xi_\alpha \right\Vert_H^p \right]^\frac{1}{p} \\
			& \quad\quad \leq \mathbb{E}\left[ \sup_{t \in [0,T]} \left\Vert X_t - \sum_{\alpha \in \mathcal{J}_{I,J,K}} x_\alpha(t) \xi_\alpha \right\Vert_H^p \right]^\frac{1}{p} + \mathbb{E}\left[ \sup_{t \in [0,T]} \left\Vert \sum_{\alpha \in \mathcal{J}_{I,J,K}} x_\alpha(t) \xi_\alpha - \sum_{\alpha \in \mathcal{J}_{I,J,K}} \varphi_\alpha(t) \xi_\alpha \right\Vert_H^p \right]^\frac{1}{p} \\
			& \quad\quad \leq \mathbb{E}\left[ \sup_{t \in [0,T]} \left\Vert X_t - \sum_{\alpha \in \mathcal{J}_{I,J,K}} x_\alpha(t) \xi_\alpha \right\Vert_H^p \right]^\frac{1}{p} + \mathbb{E}\left[ \left( \sum_{\alpha \in \mathcal{J}_{I,J,K}} \sup_{t \in [0,T]} \left\Vert x_\alpha(t) \xi_\alpha - \varphi_\alpha(t) \xi_\alpha \right\Vert_H \right)^p \right]^\frac{1}{p} \\
			& \quad\quad \leq \mathbb{E}\left[ \sup_{t \in [0,T]} \left\Vert X_t - \sum_{\alpha \in \mathcal{J}_{I,J,K}} x_\alpha(t) \xi_\alpha \right\Vert_H^p \right]^\frac{1}{p} + \sum_{\alpha \in \mathcal{J}_{I,J,K}} \mathbb{E}\left[ \sup_{t \in [0,T]} \left\Vert x_\alpha(t) \xi_\alpha - \varphi_\alpha(t) \xi_\alpha \right\Vert_H^p \right]^\frac{1}{p} \\
			& \quad\quad \leq \mathbb{E}\left[ \sup_{t \in [0,T]} \left\Vert X_t - \sum_{\alpha \in \mathcal{J}_{I,J,K}} x_\alpha(t) \xi_\alpha \right\Vert_H^p \right]^\frac{1}{p} + \sum_{\alpha \in \mathcal{J}_{I,J,K}} \sup_{t \in [0,T]} \left\Vert x_\alpha(t) - \varphi_\alpha(t,\cdot) \right\Vert_H \mathbb{E}\left[ \vert \xi_\alpha \vert^p \right]^\frac{1}{p} \\
			& \quad\quad < \frac{\varepsilon}{2} + \sum_{\alpha \in \mathcal{J}_{I,J,K}} \frac{\varepsilon}{2 \vert \mathcal{J}_{I,J,K} \vert \mathbb{E}\left[ \vert \xi_\alpha \vert^p \right]^\frac{1}{p}} \mathbb{E}\left[ \vert \xi_\alpha \vert^p \right]^\frac{1}{p} \leq \varepsilon,
		\end{aligned}
	\end{equation*}
	which completes the proof.
\end{proof}

\subsection{Proof of Theorem~\ref{ThmRUAT}}
\label{SecProofRUAT}

\begin{proof}[Proof of Theorem~\ref{ThmRUAT}]
	Fix some $\varepsilon > 0$. Then, by using Theorem~\ref{ThmCameronMartin}, there exists some $I,J,K \in \mathbb{N}$ and $(x_\alpha)_{\alpha \in \mathcal{J}_{I,J,K}} \subseteq C^0([0,T];H)$ such that
	\begin{equation}
		\label{EqThmRUATProof1}
		\sup_{t \in [0,T]} \mathbb{E}\left[ \left\Vert X_t - \sum_{\alpha \in \mathcal{J}_{I,J,K}} x_\alpha(t) \xi_\alpha \right\Vert_H^p \right]^\frac{1}{p} < \frac{\varepsilon}{2}.
	\end{equation}
	Moreover, by using Lemma~\ref{LemmaNNWellDef}, we observe that $C^0([0,T];H)$ satisfies the conditions$^\text{\ref{FootnoteOpen0T}}$ of \cite[Definition~2.3]{neufeld24}. Hence, we can apply the universal approximation result for random neural networks in \cite[Theorem~2.8]{neufeld24} to conclude that for every $\alpha \in \mathcal{J}_{I,J,K}$ there exists some $\Phi_\alpha \in \mathcal{RN}^\rho_{[0,T] \times U,d}$ such that
	\begin{equation}
		\label{EqThmRUATProof2}
		\mathbb{E}\left[ \left\Vert x_\alpha - \Phi_\alpha \right\Vert_{C^0([0,T];H)}^p \right]^\frac{1}{p} = \mathbb{E}\left[ \sup_{t \in [0,T]} \left\Vert x_\alpha(t) - \Phi_\alpha(t,\cdot) \right\Vert_H^p \right]^\frac{1}{p} < \frac{\varepsilon}{2 \vert \mathcal{J}_{I,J,K} \vert \mathbb{E}\left[ \vert \xi_\alpha \vert^p \right]^\frac{1}{p}}.
	\end{equation}
	Thus, by using Minkowski's inequality, that $\Phi_\alpha \in \mathcal{RN}^\rho_{[0,T] \times U,d}$ is independent of $\xi_\alpha: \Omega \rightarrow \mathbb{R}$ for any $\alpha \in \mathcal{J}$ (as $(A_{0,n}, A_{1,n}, B_n)_{n \in \mathbb{N}}: \Omega \rightarrow \mathbb{R} \times \mathbb{R}^m \times \mathbb{R}$ are by Assumption~\ref{AssCDF} independent of $(W_t)_{t \in [0,T]}$), and the inequalities \eqref{EqThmRUATProof1}+\eqref{EqThmRUATProof2}, it follows that
	\begin{equation*}
		\begin{aligned}
			& \sup_{t \in [0,T]} \mathbb{E}\left[ \left\Vert X_t - \sum_{\alpha \in \mathcal{J}_{I,J,K}} \Phi_\alpha(t) \xi_\alpha \right\Vert_H^p \right]^\frac{1}{p} \\
			& \quad\quad \leq \mathbb{E}\left[ \sup_{t \in [0,T]} \left\Vert X_t - \sum_{\alpha \in \mathcal{J}_{I,J,K}} x_\alpha(t) \xi_\alpha \right\Vert_H^p \right]^\frac{1}{p} + \mathbb{E}\left[ \sup_{t \in [0,T]} \left\Vert \sum_{\alpha \in \mathcal{J}_{I,J,K}} x_\alpha(t) \xi_\alpha - \sum_{\alpha \in \mathcal{J}_{I,J,K}} \Phi_\alpha(t) \xi_\alpha \right\Vert_H^p \right]^\frac{1}{p} \\
			& \quad\quad \leq \mathbb{E}\left[ \sup_{t \in [0,T]} \left\Vert X_t - \sum_{\alpha \in \mathcal{J}_{I,J,K}} x_\alpha(t) \xi_\alpha \right\Vert_H^p \right]^\frac{1}{p} + \mathbb{E}\left[ \left( \sum_{\alpha \in \mathcal{J}_{I,J,K}} \sup_{t \in [0,T]} \left\Vert x_\alpha(t) \xi_\alpha - \Phi_\alpha(t) \xi_\alpha \right\Vert_H \right)^p \right]^\frac{1}{p} \\
			& \quad\quad \leq \mathbb{E}\left[ \sup_{t \in [0,T]} \left\Vert X_t - \sum_{\alpha \in \mathcal{J}_{I,J,K}} x_\alpha(t) \xi_\alpha \right\Vert_H^p \right]^\frac{1}{p} + \sum_{\alpha \in \mathcal{J}_{I,J,K}} \mathbb{E}\left[ \sup_{t \in [0,T]} \left\Vert x_\alpha(t) - \Phi_\alpha(t) \right\Vert_H^p \left\vert \xi_\alpha \right\vert^p \right]^\frac{1}{p} \\
			& \quad\quad \leq \mathbb{E}\left[ \sup_{t \in [0,T]} \left\Vert X_t - \sum_{\alpha \in \mathcal{J}_{I,J,K}} x_\alpha(t) \xi_\alpha \right\Vert_H^p \right]^\frac{1}{p} + \sum_{\alpha \in \mathcal{J}_{I,J,K}} \mathbb{E}\left[ \sup_{t \in [0,T]} \left\Vert x_\alpha(t) - \Phi_\alpha(t) \right\Vert_H^p \right]^\frac{1}{p} \mathbb{E}\left[ \left\vert \xi_\alpha \right\vert^p \right]^\frac{1}{p} \\
			& \quad\quad < \frac{\varepsilon}{2} + \sum_{\alpha \in \mathcal{J}_{I,J,K}} \frac{\varepsilon}{2 \vert \mathcal{J}_{I,J,K} \vert \mathbb{E}\left[ \vert \xi_\alpha \vert^p \right]^\frac{1}{p}} \mathbb{E}\left[ \vert \xi_\alpha \vert^p \right]^\frac{1}{p} \leq \varepsilon,
		\end{aligned}
	\end{equation*}
	which completes the proof.
\end{proof}

\subsection{Proof of auxiliary results in Section~\ref{SecAR}}
\label{SecAuxProofsAR}

\begin{proof}[Proof of Remark~\ref{RemSPDE}]
	Let Assumption~\ref{AssSPDEAff} hold. Then, Assumption~\ref{AssSPDE}~\ref{AssSPDE1}+\ref{AssSPDE5} are satisfied. Moreover, by using Assumption~\ref{AssSPDEAff}~\ref{AssSPDEAff2}+\ref{AssSPDEAff3}, the map $[0,T] \times \Omega \times H \ni (t,\omega,x) \mapsto F(t,\omega,x) := f_0(t) + f_1(t) x \in H$ is $(\mathcal{P}_T \otimes \mathcal{B}(H))/\mathcal{B}(H)$-measurable and the map $[0,T] \times \Omega \times H \ni (t,\omega,x) \mapsto B(t,\omega,x) := b_0(t) + b_1(t) x \in L_2(Z_0;H)$ is $(\mathcal{P}_T \otimes \mathcal{B}(L_2(Z_0;H)))/\mathcal{B}(L_2(Z_0;H))$-measurable, which shows Assumption~\ref{AssSPDE}~\ref{AssSPDE2}+\ref{AssSPDE3}. In addition, by using that $\Vert \widetilde{e}_i \Vert_Z = \big\Vert Q^{1/2} e_i \big\Vert_Z = \sqrt{\lambda_i} \Vert e_i \Vert_Z = \sqrt{\lambda_i}$ (as $Qe_i = \lambda e_i$) together with $\sum_{i=1}^\infty \lambda_i = C_\lambda < \infty$ and Assumption~\ref{AssSPDEAff}~\ref{AssSPDEAff4}, we obtain for every $t \in [0,T]$, $\omega \in \Omega$, and $x,y \in H$ that
	\begin{equation*}
		\begin{aligned}
			& \Vert F(t,\omega,x) - F(t,\omega,y) \Vert_H + \Vert B(t,\omega,x) - B(t,\omega,y) \Vert_{L_2(Z_0;H)} \\
			& \quad\quad = \Vert (f_0(t) + f_1(t) x) - (f_0(t) + f_1(t) y) \Vert_H + \Vert (b_0(t) + b_1(t) x) - (b_0(t) + b_1(t) y) \Vert_{L_2(Z_0;H)} \\
			& \quad\quad = \Vert f_1(t) (x-y) \Vert_H + \left( \sum_{i=1}^\infty \Vert (b_1(t) (x-y)) (\widetilde{e}_i) \Vert_H^2 \right)^\frac{1}{2} \\
			& \quad\quad \leq \Vert f_1(t) \Vert_{L(H;H)} \Vert x-y \Vert_H + \left( \sup_{z \in Z_0 \atop \Vert z \Vert_Z \leq 1} \sup_{\widetilde{x} \in H \atop \Vert \widetilde{x} \Vert_H \leq 1} \Vert (b_1(t) \widetilde{x}) z \Vert_H \right) \Vert x-y \Vert_H \left( \sum_{i=1}^\infty \Vert \widetilde{e}_i \Vert_Z^2 \right)^\frac{1}{2} \\
			& \quad\quad \leq \left( \Vert f_1(t) \Vert_{L(H;H)} + \sqrt{C_\lambda} \left( \sup_{z \in Z_0 \atop \Vert z \Vert_Z \leq 1} \sup_{\widetilde{x} \in H \atop \Vert \widetilde{x} \Vert_H \leq 1} \Vert (b_1(t) \widetilde{x}) z \Vert_H \right) \right) \Vert x-y \Vert_H \\
			& \quad\quad \leq C_{F,B} \Vert x-y \Vert_H
		\end{aligned}
	\end{equation*}
	and that
	\begin{equation*}
		\begin{aligned}
			& \Vert F(t,\omega,x) \Vert_H^2 + \Vert B(t,\omega,x) \Vert_{L_2(Z_0;H)}^2 = \Vert f_0(t) + f_1(t) x \Vert_H^2 + \sum_{i=1}^\infty \Vert (b_0(t) + b_1(t) x)(\widetilde{e}_i) \Vert_H^2 \\
			& \leq \left( \Vert f_0(t) \Vert_H + \Vert f_1(t) x \Vert_H \right)^2 + \sum_{i=1}^\infty \Vert b_0(t)(\widetilde{e}_i) \Vert_H^2 + \sum_{i=1}^\infty \Vert (b_1(t) x)(\widetilde{e}_i) \Vert_H^2 \\
			& \leq \left( \Vert f_0(t) \Vert_H + \Vert f_1(t) \Vert_{L(H;H)} \Vert x \Vert_H \right)^2 \\
			& \quad\quad + \left( \sup_{z \in Z_0 \atop \Vert z \Vert_Z \leq 1} \Vert b_0(t)(z) \Vert_H^2 + \left( \sup_{z \in Z_0 \atop \Vert z \Vert_Z \leq 1} \sup_{\widetilde{x} \in H \atop \Vert \widetilde{x} \Vert_H \leq 1} \Vert (b_1(t) \widetilde{x})(z) \Vert_H^2 \right) \Vert x \Vert_H^2 \right) \sum_{i=1}^\infty \Vert \widetilde{e}_i \Vert_Z^2 \\
			& \leq \left( \Vert f_0(t) \Vert_H + \Vert f_1(t) \Vert_{L(H;H)} + C_\lambda \sup_{z \in Z_0 \atop \Vert z \Vert_Z \leq 1} \Vert b_0(t)(z) \Vert_H^2 + C_\lambda \sup_{z \in Z_0 \atop \Vert z \Vert_Z \leq 1} \sup_{\widetilde{x} \in H \atop \Vert \widetilde{x} \Vert_H \leq 1} \Vert (b_1(t) \widetilde{x})(z) \Vert_H^2 \right)^2 \left( 1 + \Vert x \Vert_H^2 \right) \\
			& \leq C_{F,B}^2 \left( 1 + \Vert x \Vert_H \right)^2,
		\end{aligned}
	\end{equation*}
	which proves that Assumption~\ref{AssSPDE}~\ref{AssSPDE4} also holds. Hence, we can apply Proposition~\ref{PropSPDE} to conclude that \eqref{EqDefSPDE} admits a unique mild solution $X: [0,T] \times \Omega \rightarrow H$.
\end{proof}

\subsubsection{Malliavin calculus for Hilbert space-valued random variables}
\label{SubsecMall1}

In order to prove the approximation rates in Theorem~\ref{ThmDetAR}+\ref{ThmRandAR}, we first apply the Stroock-Taylor formula in \cite{stroock87} to Hilbert space-valued random variables that are infinitely many times Malliavin differentiable. To this end, we recall some notions of Malliavin calculus for Hilbert space-valued random variables (see e.g.~\cite{malliavin78,sanzsole05,nualart06,carmona07}) and introduce the abbreviations $L^2([0,T]^k;\widetilde{H}) := L^2([0,T]^k,\mathcal{B}([0,T]^k),du;\widetilde{H})$ and $L^2(\Omega;\widetilde{H}) := L^2(\Omega,\mathcal{F}_T,\mathbb{P};\widetilde{H})$ for any $k \in \mathbb{N}$ and any Hilbert space $(\widetilde{H},\langle \cdot, \cdot \rangle_{\widetilde{H}})$. Moreover, we denote the stochastic integral of every deterministic function $\psi \in L^2([0,T];L_2(Z_0;\mathbb{R}))$ by $\mathscr{W}(\psi) := \int_0^T \psi_t dW_t \in L^2(\Omega)$. In addition, we define $\mathscr{P}(H) \subseteq L^2(\Omega;H)$ as the linear span of polynomial $H$-valued random variables of the form
\begin{equation}
	\label{EqDefSmoothRandVar}
	\Omega \ni \omega \quad \mapsto \quad F(\omega) = q(\mathscr{W}(\psi_1)(\omega),...,\mathscr{W}(\psi_N)(\omega)) x \in H,
\end{equation}
for some $N \in \mathbb{N}$, $\psi_1,...,\psi_N \in L^2([0,T];L_2(Z_0;\mathbb{R}))$, $x \in H$, and a polynomial $q \in \Pol(\mathbb{R}^N)$. Then, for any $k \in \mathbb{N}$, we define the $k$-th \emph{Malliavin derivative} of $F \in \mathscr{P}(H)$ in \eqref{EqDefSmoothRandVar} as
\begin{equation*}
	\begin{aligned}
		& \Omega \ni \omega \quad \mapsto \quad D^k F(\omega) := \left( (s_1,...,s_k) \mapsto D_{s_1,...,s_k} F \right)(\omega) := \\
		& \sum_{\beta := (\beta_1,...,\beta_N) \in \mathbb{N}^N_0 \atop \vert \beta \vert = k} \frac{\partial^k q}{\partial u_1^{\beta_1} \cdots \partial u_N^{\beta_N}}(\mathscr{W}(\psi_1)(\omega),...,\mathscr{W}(\psi_N)(\omega)) \, \bigotimes_{n=1}^N \psi_n(\cdot)^{\otimes \beta_n} \otimes x \in L^2([0,T]^k;Z_0^{\otimes k} \otimes H),
	\end{aligned}
\end{equation*}
where the element $\bigotimes_{n=1}^N \psi_n(\cdot)^{\otimes \beta_n} \otimes x \in L^2([0,T]^k;Z_0^{\otimes k} \otimes H)$ denotes the function
\begin{equation*}
	\begin{aligned}
		& [0,T]^k \ni (s_1,...,s_k) \mapsto \left( (z_1,...,z_k) \mapsto x \prod_{n=1}^N \prod_{l=1}^{\beta_n} \psi_n(s_{\ind_{l,n}(\beta)})(z_{\ind_{l,n}(\beta)}) \right) \in Z_0^{\otimes k} \otimes H,
	\end{aligned}
\end{equation*}
with $\ind_{l,n}(\beta) := \beta_1 + ... + \beta_{n-1} + l \in \mathbb{N}$, and where $Z_0^{\otimes k} \otimes H$ can be understood as the $k$-times iterated space of Hilbert-Schmidt operators $L_2(Z_0;...\,L_2(Z_0;H)...)$ equipped with $\langle \Xi_1, \Xi_2 \rangle_{Z_0^{\otimes k} \otimes H} = \sum_{i_1,...,i_k=1}^\infty \langle \Xi_1\left( \widetilde{e}_{i_1},...,\widetilde{e}_{i_k} \right), \Xi_2\left( \widetilde{e}_{i_1},...,\widetilde{e}_{i_k} \right) \rangle_H$. Note that $D^k \mathscr{P}(H) \subseteq \mathscr{P}(Z_0^{\otimes k} \otimes H)$, that $D^k D^l F = D^{k+l} F$ for any $F \in \mathscr{P}(H)$, and that $D^k: \mathscr{P}(H) \subseteq L^2(\Omega;H) \rightarrow L^2(\Omega;L^2([0,T]^k;Z_0^{\otimes k} \otimes H))$ is closeable. Moreover, we define $\mathbb{D}^{k,2}(H)$ as the closure of $\mathscr{P}(H)$ with respect to the norm $\Vert \cdot \Vert_{\mathbb{D}^{k,2}(H)}$ induced by
\begin{equation*}
	\langle F,G \rangle_{\mathbb{D}^{k,2}} := \mathbb{E}\left[ \langle F,G \rangle_H \right] + \sum_{j=1}^k \mathbb{E}\left[ \langle D^j F, D^j G \rangle_{L^2([0,T]^j;Z_0^{\otimes j} \otimes H)} \right].
\end{equation*}
Then, the operator $D^k: \mathbb{D}^{k,2}(H) \rightarrow L^2(\Omega;H)$ is linear and continuous (see \cite[Proposition~1.5.7]{nualart06}). In addition, we define the vector space $\mathbb{D}^{\infty,2}(H) := \bigcap_{k \in \mathbb{N}} \mathbb{D}^{k,2}(H)$.

Furthermore, for any $k \in \mathbb{N}$, we define the \emph{$k$-th divergence operator} $\delta^k: \dom(\delta^k) \rightarrow L^2(\Omega;H)$ as the adjoint of the $k$-th Malliavin derivative $D^k: \mathbb{D}^{k,2}(H) \rightarrow L^2(\Omega;H)$. More precisely, the domain $\dom(\delta^k)$ consists of all $\Psi \in L^2(\Omega;L^2([0,T]^k;Z_0^{\otimes k} \otimes H))$ for which there exists a constant $C_\Psi > 0$ such that for every $F \in \mathbb{D}^{k,2}(H)$ it holds that $\big\vert \mathbb{E}\big[ \langle D^k F, \Psi \rangle_{L^2([0,T]^k;Z_0^{\otimes k} \otimes H)} \big] \big\vert \leq C_\Psi \Vert F \Vert_{\mathbb{D}^{k,2}(H)}$. In this case, $\delta^k(\Psi)$ is defined as the unique element in $L^2(\Omega;H)$ such that for every $F \in \mathbb{D}^{k,2}(H)$ we have
\begin{equation}
	\label{EqDefDivergence}
	\mathbb{E}\left[ \langle F, \delta^k(\Psi) \rangle_H \right] = \mathbb{E}\left[ \langle D^k F, \Psi \rangle_{L^2([0,T]^k;Z_0^{\otimes k} \otimes H)} \right].
\end{equation}
Note that $\mathbb{D}^{k,2}(L^2([0,T]^k;Z_0^{\otimes k} \otimes H)) \subseteq \dom(\delta^k)$ and that the restriction $\delta^k\vert_{\mathbb{D}^{k,2}(L^2([0,T]^k;Z_0^{\otimes k} \otimes H))}: \mathbb{D}^{k,2}(L^2([0,T]^k;Z_0^{\otimes k} \otimes H)) \rightarrow L^2(\Omega;H)$ is continuous (see \cite[Proposition~1.5.7]{nualart06}). 

Then, we apply the Stroock-Taylor formula in \cite[Theorem~6]{stroock87} to random variables in $\mathbb{D}^{\infty,2}(H)$. To this end, we define the subset $\mathcal{J}_k := \lbrace \alpha \in \mathcal{J}: \vert \alpha \vert = k \rbrace \subseteq \mathcal{J}$ and denote by $L^2(\Omega;H) \ni F \mapsto \Pi_k F := \sum_{\alpha \in \mathcal{J}_k} \mathbb{E}[F \xi_\alpha] \xi_\alpha \in L^2(\Omega;H)$ the projection onto the Wick polynomials $(\xi_\alpha)_{\alpha \in \mathcal{J}_k}$ of order $k$.

\begin{proposition}
	\label{PropStroock}
	Let $F \in \mathbb{D}^{\infty,2}(H)$. Then, for every $k \in \mathbb{N}$, we have $\Pi_k F = \frac{1}{k!} \delta^k\left( \mathbb{E}\left[ D^k F \right] \right)$, and thus
	\begin{equation}
		\label{EqPropStroock1}
		F = \sum_{k=0}^\infty \Pi_k F = \sum_{k=0}^\infty \frac{1}{k!} \delta^k\left( \mathbb{E}\left[ D^k F \right] \right),
	\end{equation}
	with $\Pi_0(F) = \delta^0\left( \mathbb{E}\left[ D^0 F \right] \right) := \mathbb{E}[F]$, where both sums in \eqref{EqPropStroock1} converge in $L^2(\Omega;H)$.
\end{proposition}
\begin{proof}
	We follow the proof of \cite[Theorem~6]{stroock87}. To this end, we fix some $F \in \mathbb{D}^{\infty,2}(H)$ and $k \in \mathbb{N}$, and assume that $(y_n)_{n \in \mathbb{N}}$ is an orthonormal basis of $(H,\langle \cdot,\cdot \rangle_H)$. Then, by using the definition of the Wick polynomials $(\xi_\alpha)_{\alpha \in \mathcal{J}}$ (see Definition~\ref{DefWick}),  that $y_n \prod_{i,j=1}^\infty h_{\alpha_{i,j}}(\xi_{i,j}) = \delta^k\big( \bigotimes_{i,j=1}^\infty \left( g_j \widetilde{e}_i \right)^{\otimes \alpha_{i,j}} \otimes y_n \big)$ for any $l \in \mathbb{N}$ and $\alpha \in \mathcal{J}$ with $\vert \alpha \vert = k$ (see \cite[Lemma~1]{stroock87}), that there are $\frac{k!}{\alpha!}$ possibilities to form $\alpha \in \mathcal{J}$ with indices $(i_1,j_1),...,(i_k,j_k) \in \mathbb{N}^2$, that $F \in \mathbb{D}^{\infty,2}(H) \subseteq \mathbb{D}^{k,2}(H)$ together with \eqref{EqDefDivergence}, the linearity of the expectation and inner product, that $\delta^k: \mathbb{D}^{k,2}(L^2([0,T]^k;Z_0^{\otimes k} \otimes H)) \rightarrow L^2(\Omega;H)$ is continuous (see \cite[Proposition~1.5.7]{nualart06}), and that $\big( \bigotimes_{l=1}^k \left( g_{j_l} \widetilde{e}_{i_l} \right) \otimes y_n \big)_{\mathbf{i} := (i_1,...,i_k) \in \mathbb{N}^k, \, \mathbf{j} := (j_1,...,j_k) \in \mathbb{N}^k, \, n \in \mathbb{N}}$ is an orthonormal basis of $(L^2([0,T]^k;Z_0^{\otimes k} \otimes H),\langle \cdot, \cdot \rangle_{L^2([0,T]^k;Z_0^{\otimes k} \otimes H)})$, it follows that	
	\begin{equation}
		\label{EqPropStroockProof1}
		\begin{aligned}
			\Pi_k F & = \sum_{\alpha \in \mathcal{J}_k} \mathbb{E}\left[ F \xi_\alpha \right] \xi_\alpha = \sum_{\alpha \in \mathcal{J}_k} \sum_{n=1}^\infty \mathbb{E}\left[ \langle F, \xi_\alpha y_n \rangle_H \right] \xi_\alpha y_n \\
			& = \sum_{\alpha \in \mathcal{J}_k} \sum_{n=1}^\infty \frac{1}{\alpha!} \mathbb{E}\left[ \bigg\langle F, y_n \prod_{i,j=1}^\infty h_{\alpha_{i,j}}(\xi_{i,j}) \bigg\rangle_H \right] y_n \prod_{i,j=1}^\infty h_{\alpha_{i,j}}(\xi_{i,j}) \\
			& = \sum_{\alpha \in \mathcal{J}_k} \sum_{n=1}^\infty \frac{1}{\alpha!} \mathbb{E}\left[ \bigg\langle F, \delta^k\left( \bigotimes_{i,j=1}^\infty \left( g_j \widetilde{e}_i \right)^{\otimes \alpha_{i,j}} \otimes y_n \right) \bigg\rangle_H \right] \delta^k\left( \bigotimes_{i,j=1}^\infty \left( g_j \widetilde{e}_i \right)^{\otimes \alpha_{i,j}} \otimes y_n \right) \\
			& = \frac{1}{k!} \sum_{\mathbf{i}, \mathbf{j} \in \mathbb{N}^k} \sum_{n=1}^\infty \mathbb{E}\left[ \bigg\langle F, \delta^k\left( \bigotimes_{l=1}^k \left( g_{j_l} \widetilde{e}_{i_l} \right) \otimes y_n \right) \bigg\rangle_H \right] \delta^k\left( \bigotimes_{l=1}^k \left( g_{j_l} \widetilde{e}_{i_l} \right) \otimes y_n \right) \\
			& = \frac{1}{k!} \sum_{\mathbf{i}, \mathbf{j} \in \mathbb{N}^k} \sum_{n=1}^\infty \mathbb{E}\left[ \Big\langle D^k F, \bigotimes_{l=1}^k \left( g_{j_l} \widetilde{e}_{i_l} \right) \otimes y_n \Big\rangle_{L^2([0,T]^k;Z_0^{\otimes k} \otimes H)} \right] \delta^k\left( \bigotimes_{l=1}^k \left( g_{j_l} \widetilde{e}_{i_l} \right) \otimes y_n \right) \\
			& = \frac{1}{k!} \sum_{\mathbf{i}, \mathbf{j} \in \mathbb{N}^k} \sum_{n=1}^\infty \Big\langle \mathbb{E}\left[ D^k F \right], \bigotimes_{l=1}^k \left( g_{j_l} \widetilde{e}_{i_l} \right) \otimes y_n \Big\rangle_{L^2([0,T]^k;Z_0^{\otimes k} \otimes H)} \delta^k\left( \bigotimes_{l=1}^k \left( g_{j_l} \widetilde{e}_{i_l} \right) \otimes y_n \right) \\
			& = \frac{1}{k!} \delta^k\left( \sum_{\mathbf{i}, \mathbf{j} \in \mathbb{N}^k} \sum_{n=1}^\infty \Big\langle \mathbb{E}\left[ D^k F \right], \bigotimes_{l=1}^k \left( g_{j_l} \widetilde{e}_{i_l} \right) \otimes y_n \Big\rangle_{L^2([0,T]^k;Z_0^{\otimes k} \otimes H)} \bigotimes_{l=1}^k \left( g_{j_l} \widetilde{e}_{i_l} \right) \otimes y_n \right) \\
			& = \frac{1}{k!} \delta^k\left( \mathbb{E}\left[ D^k F \right] \right),
		\end{aligned}
	\end{equation}
	where the sums in the first six lines converge in $L^2(\Omega;H)$, and where the sum in the second last line converges in $L^2([0,T]^k;Z_0^{\otimes k} \otimes H)$. Since $k \in \mathbb{N}$ was chosen arbitrarily and the Wick polynomials $(\xi_\alpha)_{\alpha \in \mathcal{J}}$ are orthonormal among different orders (see Lemma~\ref{LemmaWick}), we obtain the expansion \eqref{EqPropStroock1}.
\end{proof}

Moreover, we compute an upper bound for the $\Vert \cdot \Vert_{L^2([0,T];H)}$-norm of the $k$-th divergence operator of some given $\psi \in L^2([0,T]^k;Z_0^{\otimes k} \otimes H) \subseteq \dom(\delta^k)$ (see also \cite[Proposition~2.3]{nualart88} for the case $H = \mathbb{R}$).

\begin{lemma}
	\label{LemmaL2SkorokhodInt}
	For $k \in \mathbb{N}$, let $\psi \in L^2([0,T]^k;Z_0^{\otimes k} \otimes H)$. Then, $\psi \in \dom(\delta^k)$ and it holds that
	\begin{equation}
		\begin{aligned}
			\label{EqLemmaL2SkorokhodInt1}
			\mathbb{E}\left[ \left\Vert \delta^k(\psi) \right\Vert_H^2 \right] \leq k! \Vert \psi \Vert_{L^2([0,T]^k;Z_0^{\otimes k} \otimes H)}^2.
		\end{aligned}
	\end{equation}
\end{lemma}
\begin{proof}
	Fix some $k \in \mathbb{N}$ and $\psi \in L^2([0,T]^k;Z_0^{\otimes k} \otimes H) \subseteq \mathbb{D}^{k,2}(L^2([0,T]^k;Z_0^{\otimes k} \otimes H)) \subseteq \dom(\delta^k)$. Then, by iteratively using \cite[Proposition~1.2.2]{nualart06} and the notations $\epsilon(i,j) := (\delta_{(i,j),(m,n)})_{(m,n) \in \mathbb{N}^2} \in \mathcal{J} \subseteq \mathbb{N}_0^{\mathbb{N} \times \mathbb{N}}$ for any $i,j \in \mathbb{N}_2$ as well as $\epsilon(i_{l:k},j_{l:k}) := \epsilon(i_l,j_l) + ... + \epsilon(i_k,j_k) \in \mathcal{J} \subseteq \mathbb{N}_0^{\mathbb{N} \times \mathbb{N}}$ for any $i,j \in \mathbb{N}$ and $l = 1,...,k$, it follows for every $\alpha \in \mathcal{J}$ that
	\begin{equation*}
		\begin{aligned}
			& D^k_{s_1,...,s_k} \xi_\alpha = \sum_{i_k=1 \atop j_k=1}^\infty \sqrt{\alpha_{i_k,j_k}} D^{k-1} \xi_{\alpha-\epsilon(i_k,j_k)} g_{j_k}(s_k) \widetilde{e}_{i_k} \\
			& \quad\quad = \sum_{i_{k-1},i_k=1 \atop j_{k-1},j_k=1}^\infty \sqrt{\alpha_{i_k,j_k}} \sqrt{(\alpha-\epsilon(i_k,j_k))_{i_{k-1},j_{k-1}}} D^{k-2} \xi_{\alpha-\epsilon(i_{(k-1):k},j_{(k-1):k})} \left( g_{j_{k-1}}(s_{k-1}) \widetilde{e}_{i_{k-1}} \right) \otimes \left( g_{j_k}(s_k) \widetilde{e}_{i_k} \right) \\
			& \quad\quad = ... \\
			& \quad\quad = \sqrt{\alpha!} \sum_{i_1,...,i_k=1 \atop j_1,...,j_k=1}^\infty \xi_{\alpha-\epsilon(i_{1:k},j_{1:k})} \bigotimes_{l=1}^k \left( g_{j_l}(s_l) \widetilde{e}_{i_l} \right).
		\end{aligned}
	\end{equation*}
	Hence, by taking the expectation and using that $\mathbb{E}\left[ \xi_{\alpha-\epsilon(i_{1:k},j_{1:k})} \right] = \delta_{\alpha,\epsilon(i_{1:k},j_{1:k})}$, Minkowski's inequality, that $\big( \bigotimes_{j=1}^k (g_{j_l} \widetilde{e}_{i_l}) \big)_{\mathbf{i} := (i_1,...,i_k) \in \mathbb{N}^k, \, \mathbf{j} := (j_1,...,j_k) \in \mathbb{N}^k}$ forms an orthonormal basis of the Hilbert space $(L^2([0,T]^k;Z_0^{\otimes k}),\langle\cdot,\cdot\rangle_{L^2([0,T]^k;Z_0^{\otimes k})})$, and that there are $\frac{\vert\alpha\vert!}{\alpha!}$ possibilities to form $\epsilon(i_{1:k},j_{1:k}) = \alpha$ with $(i_1,j_1),...,(i_k,j_k) \in \mathbb{N}^2$, we have for every $\alpha \in \mathcal{J}$ that
	\begin{equation*}
		\begin{aligned}
			\left\Vert \mathbb{E}\left[ D^k \xi_\alpha \right] \right\Vert_{L^2([0,T]^k;Z_0^{\otimes k})}^2 & = \alpha! \left\Vert \sum_{\mathbf{i}, \mathbf{j} \in \mathbb{N}^k} \mathbb{E}\left[ \xi_{\alpha-\epsilon(i_{1:k},j_{1:k})} \right] \bigotimes_{l=1}^k \left( g_{j_l} \widetilde{e}_{i_l} \right) \right\Vert_{L^2([0,T]^k;Z_0^{\otimes k})}^2 \\
			& \leq \alpha! \sum_{\mathbf{i}, \mathbf{j} \in \mathbb{N}^k} \mathds{1}_{\lbrace \epsilon(i_{1:k},j_{1:k}) \rbrace}(\alpha) = \alpha! \frac{k!}{\alpha!} = k!.
		\end{aligned}
	\end{equation*}
	Next, for a fixed orthonormal basis $(y_n)_{n \in \mathbb{N}}$ of $(H,\langle\cdot,\cdot\rangle_H)$, let $\conv(\lbrace \pm \xi_\alpha y_n: (\alpha,n) \in \mathcal{J} \times \mathbb{N} \rbrace) := \big\lbrace \sum_{n=1}^N \lambda_n \xi_{\alpha_n} y_n: N \in \mathbb{N}, \, \alpha_1,...,\alpha_N \in \mathcal{J}, \, y_1,...,y_N \in H, \, \lambda_1,...,\lambda_N \in [-1,1], \, \sum_{n=1}^N \vert \lambda_n \vert = 1 \big\rbrace$ denote the set of convex combinations built from $\lbrace \pm \xi_\alpha y_n: (\alpha,n) \in \mathcal{J} \times \mathbb{N} \rbrace$. Then, for every $Y := \sum_{n=1}^N \lambda_n \xi_{\alpha_n} y_n \in \conv(\lbrace \pm \xi_\alpha y_n: (\alpha,n) \in \mathcal{J} \times \mathbb{N} \rbrace)$, it holds that
	\begin{equation*}
		\begin{aligned}
			\left\Vert \mathbb{E}\left[ D^k Y \right] \right\Vert_{L^2([0,T]^k;Z_0^{\otimes k} \otimes H)}^2 & = \left\Vert \sum_{n=1}^N \lambda_n \mathbb{E}\left[ D^k \xi_{\alpha_n} \right] y_n \right\Vert_{L^2([0,T]^k;Z_0^{\otimes k} \otimes H)}^2 \\
			& \leq \left( \sum_{n=1}^N \vert \lambda_n \vert \left\Vert \mathbb{E}\left[ D^k \xi_{\alpha_n} \right] y_n \right\Vert_{L^2([0,T]^k;Z_0^{\otimes k} \otimes H)} \right)^2 \\
			& \leq \left( \sum_{n=1}^N \vert \lambda_n \vert \right)^2 \max_{n=1,...,N} \left( \left\Vert \mathbb{E}\left[ D^k \xi_{\alpha_n} \right] \right\Vert_{L^2([0,T]^k;Z_0^{\otimes k})}^2 \Vert y_n \Vert_H^2 \right) \\
			& \leq k!.
		\end{aligned}
	\end{equation*}
	Finally, by using the dual norm of $L^2([0,T];H) \cong L^2([0,T];H)^*$, that $\linspan\lbrace \xi_\alpha y_n: (\alpha,n) \in \mathcal{J} \times \mathbb{N} \rbrace$ is dense in $L^2([0,T];H)$ (see Lemma~\ref{LemmaWick} and \cite[Lemma~1.2.19~(i)]{hytoenen16}), that $\lbrace Y \in \linspan\lbrace \xi_\alpha y_n: (\alpha,n) \in \mathcal{J} \times \mathbb{N} \rbrace: \Vert Y \Vert_{L^2(\Omega;H)} \leq 1 \rbrace = \conv(\lbrace \pm \xi_\alpha y_n: (\alpha,n) \in \mathcal{J} \times \mathbb{N} \rbrace)$, the identity~\eqref{EqDefDivergence}, and linearity, it follows that
	\begin{equation*}
		\begin{aligned}
			\mathbb{E}\left[ \left\Vert \delta^k(\psi) \right\Vert_H^2 \right] & = \sup_{Y \in L^2(\Omega;H) \atop \Vert Y \Vert_{L^2(\Omega;H)} \leq 1} \left\vert \mathbb{E}\left[ \langle \delta^k(\psi), Y \rangle_H \right] \right\vert^2 \\
			& = \sup_{Y \in \linspan\lbrace \xi_\alpha y_n: (\alpha,n) \in \mathcal{J} \times \mathbb{N} \rbrace \atop \Vert Y \Vert_{L^2(\Omega;H)} \leq 1} \left\vert \mathbb{E}\left[ \langle \delta^k(\psi), Y \rangle_H \right] \right\vert^2 \\
			& = \sup_{Y \in \conv(\lbrace \pm \xi_\alpha y_n: (\alpha,n) \in \mathcal{J} \times \mathbb{N} \rbrace)} \left\vert \mathbb{E}\left[ \langle \psi, D^k Y \rangle_{L^2([0,T]^k;Z_0^{\otimes k} \otimes H)} \right] \right\vert^2 \\
			& = \sup_{Y \in \conv(\lbrace \pm \xi_\alpha y_n: (\alpha,n) \in \mathcal{J} \times \mathbb{N} \rbrace)} \left\vert \Big\langle \psi, \mathbb{E}\left[ D^k Y \right] \Big\rangle_{L^2([0,T]^k;Z_0^{\otimes k} \otimes H)} \right\vert^2 \\
			& \leq \Vert \psi \Vert_{L^2([0,T]^k;Z_0^{\otimes k} \otimes H)}^2 \sup_{Y \in \conv(\lbrace \xi_\alpha y_n: (\alpha,n) \in \mathcal{J} \times \mathbb{N} \rbrace)} \left\Vert \mathbb{E}\left[ D^k Y \right] \right\Vert_{L^2([0,T]^k;Z_0^{\otimes k} \otimes H)}^2 \\
			& \leq k! \Vert \psi \Vert_{L^2([0,T]^k;Z_0^{\otimes k} \otimes H)}^2,
		\end{aligned}
	\end{equation*}
	which completes the proof.
\end{proof}

\subsubsection{Malliavin regularity of solution to \eqref{EqDefSPDE}}
\label{SubsecMall2}

For the proof of the approximation rates in Theorem~\ref{ThmDetAR}+\ref{ThmRandAR}, we first show that the solution of \eqref{EqDefSPDE} belongs at each time to $\mathbb{D}^{\infty,2}(H)$. Then, by upper bounding its Malliavin derivatives, we can estimate the higher order terms in the chaos expansion \eqref{EqThmCameronMartin2}.

To this end, we generalize the results in \cite[Theorem~7.1]{sanzsole05}, \cite[Section~5.5]{carmona07}, and \cite[Theorem~5.7]{kruse14} from $\mathbb{D}^{1,2}(H)$ to $\mathbb{D}^{\infty,2}(H)$. We use the notations $D^0 X_t := X_t$ for $t \in [0,T]$, $\mathbf{r}^* := \max_{j=1,...,k} r_j$ and $\mathbf{r}_{-l} := (r_1,...,r_{l-1},r_{l+1},...,r_k)$ for $\mathbf{r} := (r_1,...,r_k) \in [0,T]^k$, $\mathbf{z}_{-l} := (z_1,...,z_{l-1},z_{l+1},...,z_k)$ for $\mathbf{z} := (z_1,...,z_k) \in Z_0^{\otimes k}$, and $\lambda_\mathbf{i} := \prod_{l=1}^k \lambda_{i_l}$ for $\mathbf{i} := (i_1,...,i_k) \in \mathbb{N}^k$ and $(\lambda_i)_{i \in \mathbb{N}}$ given in \eqref{EqDefBMi}.

\begin{proposition}
	\label{PropMall}
	Let Assumption~\ref{AssSPDEAff} hold and let $X$ be a mild solution of \eqref{EqDefSPDE}. Then, by using the constants $C^{(2)}_{F,B,S,t} > 0$ (see Proposition~\ref{PropSPDE}), $C_S := \sup_{t \in [0,T]} \Vert S_t \Vert_{L(H;H)} <\infty$, and $C_\lambda := \sum_{i=1}^\infty \lambda_i < \infty$ (cf.~\eqref{EqDefBMi}), the following holds true:
	\begin{enumerate}
		\item\label{PropMall1} For every $t \in [0,T]$ we have $X_t \in \mathbb{D}^{\infty,2}(H)$. Moreover, for every $k \in \mathbb{N}$, a.e.~$\mathbf{r} := (r_1,...,r_k) \in [0,T]^k$, every $\mathbf{z} := (z_1,...,z_k) \in Z_0^{\otimes k}$, and every $t \in [0,T]$ it holds that
		\begin{equation}
			\label{EqPropMall1}
			\begin{aligned}
				\quad\quad \left( D^k_\mathbf{r} X_t \right)(\mathbf{z}) & = 
				\begin{cases}
					\begin{matrix*}[l]
						\sum_{l=1}^k S_{t-r_l} \big( \mathds{1}_{\lbrace 1 \rbrace}(k) b_0(r_l) + b_1(r_l) \big[ \big( D^{k-1}_{\mathbf{r}_{-l}} X_{r_l} \big)(\mathbf{z}_{-l}) \big] \big)(z_l)  \\
						\quad\quad + \int_{\mathbf{r}^*}^t S_{t-s} f_1(s)\big[ \big( D^k_\mathbf{r} X_s \big)(\mathbf{z}) \big] ds \\
						\quad\quad + \int_{\mathbf{r}^*}^t S_{t-s} b_1(s)\big[ \big( D^k_\mathbf{r} X_s \big)(\mathbf{z}) \big] dW_s, \quad\quad \mathbb{P}\text{-a.s.},
					\end{matrix*}
					& \text{if } t \in [\mathbf{r}^*,T], \\
					0, & \text{if } t \in [0,\mathbf{r}^*).
				\end{cases}
			\end{aligned}
		\end{equation}
		\item\label{PropMall2} For every $k \in \mathbb{N}$, a.e.~$\mathbf{r} := (r_1,...,r_k) \in [0,T]^k$, every $\mathbf{i} := (i_1,...,i_k) \in \mathbb{N}^k$, and every $t \in [0,T]$, it holds that
		\begin{equation}
			\label{EqPropMall2}
			\quad\quad \mathbb{E}\left[ \left\Vert \left( D^k_\mathbf{r} X_t \right)(\widetilde{e}_\mathbf{i}) \right\Vert_H^2 \right] \leq 2 C^{(2)}_{F,B,S,t} \left( 2 + \Vert \chi_0 \Vert_H^2 \right) \left( 3 \frac{C_S^2 C_{F,B}^2}{C_\lambda} \right)^k e^{3 k C_S^2 C_{F,B}^2 (t+1) t} \lambda_\mathbf{i}.
		\end{equation}
		\item\label{PropMall3} For every $k \in \mathbb{N}$, a.e.~$\mathbf{r} := (r_1,...,r_k) \in [0,T]^k$, every $\mathbf{z} := (z_1,...,z_k) \in Z_0^{\otimes k}$, and every $t \in [0,T]$, it holds that
		\begin{equation}
			\label{EqPropMall3}
			\begin{aligned}
				\quad \mathbb{E}\left[ \left( D^k_\mathbf{r} X_t \right)(\mathbf{z}) \right] & = 
				\begin{cases}
					\begin{matrix*}[l]
						\sum_{l=1}^k S_{t-r_l} \big( \mathds{1}_{\lbrace 1 \rbrace}(k) b_0(r_l) + b_1(r_l) \mathbb{E}\big[ \big( D^{k-1}_{\mathbf{r}_{-l}} X_{r_l} \big)(\mathbf{z}_{-l}) \big] \big)(z_l)  \\
						\quad\quad + \int_{\mathbf{r}^*}^t S_{t-s} f_1(s) \mathbb{E}\big[ \big( D^k_\mathbf{r} X_s \big)(\mathbf{z}) \big] ds,
					\end{matrix*}
					& \text{if } t \in [\mathbf{r}^*,T], \\
					0, & \text{if } t \in [0,\mathbf{r}^*).
				\end{cases}
			\end{aligned}
		\end{equation}
		\item\label{PropMall4} For every $k \in \mathbb{N}$, a.e.~$\mathbf{r} := (r_1,...,r_k) \in [0,T]^k$, every $\mathbf{i} := (i_1,...,i_k) \in \mathbb{N}^k$, and every $t \in [0,T]$, it holds that
		\begin{equation}
			\label{EqPropMall4}
			\left\Vert \mathbb{E}\left[ \left( D^k_\mathbf{r} X_t \right)(\widetilde{e}_\mathbf{i}) \right] \right\Vert_H \leq \sqrt{C^{(2)}_{F,B,S,t}} \left( 2 + \Vert \chi_0 \Vert_H^2 \right)^\frac{1}{2} \left( \frac{C_S C_{F,B}}{\sqrt{C_\lambda}} \right)^k e^{k C_S C_{F,B} t} \sqrt{\lambda_\mathbf{i}}.
		\end{equation}
	\end{enumerate}
\end{proposition}
\begin{proof}
	For \ref{PropMall1}, we show by induction on $k \in \mathbb{N}$ that $X_t \in \mathbb{D}^{k,2}(H)$ with \eqref{EqPropMall1} holding for $k$. For $k = 1$, the results follows from \cite[Lemma~5.3]{carmona07} and \cite[p.~152]{carmona07}. Now, for the induction step, we fix some $k \in \mathbb{N} \cap [2,\infty)$, assume that $X_t \in \mathbb{D}^{k-1,2}(H)$ with \eqref{EqPropMall1} holding for $k-1$, and aim to prove that $X_t \in \mathbb{D}^{k,2}(H)$ with \eqref{EqPropMall1} holding for $k$. To this end, we fix some $\widetilde{\mathbf{z}} := (z_1,...,z_k) \in Z_0^{\otimes (k-1)}$ and $\widetilde{\mathbf{r}} := (r_1,...,r_{k-1}) \in [0,T]^{k-1}$ (such that \eqref{EqPropMall1} holds for $k-1$), and define $\widetilde{\mathbf{r}}^* := \max_{j=1,...,k-1} r_j$. Then, by using \eqref{EqPropMall1} for $k-1$, the $H$-valued process $\big( \big( D^{k-1}_{\widetilde{\mathbf{r}}} X_t \big)(\widetilde{\mathbf{z}}) \big)_{t \in [\widetilde{\mathbf{r}}^*,T]}$ satisfies for every $t \in [\widetilde{\mathbf{r}}^*,T]$ that
	\begin{equation}
		\label{EqPropMallProof1}
		\begin{aligned}
			\left( D^{k-1}_{\widetilde{\mathbf{r}}} X_t \right)(\widetilde{\mathbf{z}}) & = \underbrace{\sum_{l=1}^{k-1} S_{t-r_l} \left( \mathds{1}_{\lbrace 1 \rbrace}(k-1) b_0(r_l) + b_1(r_l) \left[ \left( D^{k-2}_{\widetilde{\mathbf{r}}_{-l}} X_{r_l} \right)(\widetilde{\mathbf{z}}_{-l}) \right] \right)(z_l)}_{= S_{t-\widetilde{\mathbf{r}}^*} \sum_{l=1}^{k-1} S_{\widetilde{\mathbf{r}}^*-r_l} \big( \mathds{1}_{\lbrace 1 \rbrace}(k-1) b_0(r_l) + b_1(r_l) \big[ \big( D^{k-2}_{\widetilde{\mathbf{r}}_{-l}} X_{r_l} \big) (\widetilde{\mathbf{z}}_{-l}) \big] \big)(z_l)} \\
			& \quad\quad + \int_{\widetilde{\mathbf{r}}^*}^t S_{t-s} f_1(s)\left[ \left( D^{k-1}_{\widetilde{\mathbf{r}}} X_s \right)(\widetilde{\mathbf{z}}) \right] ds \\
			& \quad\quad + \int_{\widetilde{\mathbf{r}}^*}^t S_{t-s} b_1(s)\left[ \left( D^{k-1}_{\widetilde{\mathbf{r}}} X_s \right)(\widetilde{\mathbf{z}}) \right] dW_s, \quad\quad \mathbb{P}\text{-a.s.}
		\end{aligned}
	\end{equation}
	Hence, by using the initial value $\sum_{l=1}^{k-1} S_{\widetilde{\mathbf{r}}^*-r_l} \big( \mathds{1}_{\lbrace 1 \rbrace}(k-1) b_0(r_l) + b_1(r_l) \big[ \big( D^{k-2}_{\widetilde{\mathbf{r}}_{-l}} X_{r_l} \big)(\widetilde{\mathbf{z}}_{-l}) \big] \big)(z_l) \in H$ and that the coefficients $[\widetilde{\mathbf{r}}^*,T] \times \Omega \times H \ni (t,\omega,x) \mapsto f_1(t) x \in H$ as well as $[\widetilde{\mathbf{r}}^*,T] \times \Omega \times H \ni (t,\omega,x) \mapsto b_1(t) x \in L_2(Z_0;H)$ satisfy the Lipschitz and linear growth condition (see Assumption~\ref{AssSPDEAff}~\ref{AssSPDEAff4}), we can apply \cite[Lemma~5.3]{carmona07} to conclude that $\big( D^{k-1}_{\widetilde{\mathbf{r}}} X_t \big)(\widetilde{\mathbf{z}}) \in \mathbb{D}^{1,2}(H)$ for all $t \in [\widetilde{\mathbf{r}}^*,T]$. On the other hand, for every $t \in [0,\widetilde{\mathbf{r}}^*)$, we have $\big( D^{k-1}_{\widetilde{\mathbf{r}}} X_t \big)(\widetilde{\mathbf{z}}) = 0 \in \mathbb{D}^{1,2}(H)$. Since $\widetilde{\mathbf{r}} := (r_1,...,r_{k-1}) \in [0,T]^{k-1}$ and $\widetilde{\mathbf{z}} \in Z_0^{\otimes (k-1)}$ were chosen arbitrarily, this shows that $D^{k-1} X_t = \big( (\widetilde{\mathbf{r}};\widetilde{\mathbf{z}}) \mapsto \big( D^{k-1}_{\widetilde{\mathbf{r}}} X_t \big)(\widetilde{\mathbf{z}}) \big) \in \mathbb{D}^{1,2}(L^2([0,T]^{k-1};Z_0^{k-1} \otimes H))$ and thus $X_t \in D^{k,2}(H)$. Moreover, by using \eqref{EqPropMallProof1}, the chain rule for the Malliavin derivative (see \cite[Proposition~5.2]{carmona07}), the Malliavin derivative of a time integral (see \cite[Proposition~4.8]{kruse14}), the Malliavin derivative of a stochastic integral with $\mathbb{F}$-predictable integrand (see \cite[Proposition~5.4]{carmona07}), and that $D^k_\mathbf{r} X_s = 0 \in Z_0^{\otimes k} \otimes H$ for any $\mathbf{r} \in [0,T]^k$ and $s \in [0,\mathbf{r}^*)$ with $\mathbf{r}^* := \max_{j=1,...,k} r_j$, it follows for a.e.~$\mathbf{r} := (\widetilde{\mathbf{r}},r_k) := (r_1,...,r_k) \in [0,T]^k$, every $\mathbf{z} := (\widetilde{\mathbf{z}},z_k) := (z_1,...,z_k) \in Z_0^{\otimes k}$, and every $t \in [\mathbf{r}^*,T]$ that
	\begin{equation*}
		\begin{aligned}
			\left( D^k_\mathbf{r} X_t \right)(\mathbf{z}) & = \left( D_r \left( D^{k-1}_{\widetilde{\mathbf{r}}} X_t \right)(\widetilde{\mathbf{z}}) \right)(z_k) \\
			& = \left( D_r \sum_{l=1}^{k-1} S_{t-r_l} \left( \mathds{1}_{\lbrace 1 \rbrace}(k-1) b_0(r_l) + b_1(r_l)\left[ \left( D^{k-2}_{\widetilde{\mathbf{r}}_{-l}} X_{r_l} \right)(\widetilde{\mathbf{z}}_{-l}) \right] \right)(z_l) \right)(z_k) \\
			& \quad\quad + \left( D_r \int_{\widetilde{\mathbf{r}}^*}^t S_{t-s} f_1(s)\left[ \left( D^{k-1}_{\widetilde{\mathbf{r}}} X_s \right)(\widetilde{\mathbf{z}}) \right] ds \right)(z_k) \\
			& \quad\quad + \left( D_r \int_{\widetilde{\mathbf{r}}^*}^t S_{t-s} b_1(s)\left[ \left( D^{k-1}_{\widetilde{\mathbf{r}}} X_s \right)(\widetilde{\mathbf{z}}) \right] dW_s \right)(z_k) \\
			& = \sum_{l=1}^{k-1} S_{t-r_l} b_1(r_l)\left[ \left( D^{k-1}_{(\widetilde{\mathbf{r}}_{-l},r_k)} X_{r_l} \right)(\widetilde{\mathbf{z}}_{-l},z_k) \right](z_l) + S_{t-r_k} b_1(r_k)\left[ \left( D^k_{\widetilde{\mathbf{r}}} X_{r_k} \right)(\widetilde{\mathbf{z}}) \right](z_k) \\
			& \quad\quad + \int_{\widetilde{\mathbf{r}}^*}^t S_{t-s} f_1(s)\left[ \left( D^k_{(\widetilde{\mathbf{r}},r_k)} X_s \right)(\widetilde{\mathbf{z}},z_k) \right] ds \\
			& \quad\quad + \int_{\widetilde{\mathbf{r}}^*}^t S_{t-s} b_1(s)\left[ \left( D^k_{(\widetilde{\mathbf{r}},r_k)} X_s \right)(\widetilde{\mathbf{z}},z_k) \right] dW_s \\
			& = \sum_{l=1}^k S_{t-r_l} b_1(r_l)\left[ \left( D^{k-1}_{\mathbf{r}_{-l}} X_{r_l} \right)(\mathbf{z}_{-l}) \right](z_l) \\
			& \quad\quad + \int_{\mathbf{r}^*}^t S_{t-s} f_1(s)\left[ \left( D^k_\mathbf{r} X_s \right)(\mathbf{z}) \right] ds \\
			& \quad\quad + \int_{\mathbf{r}^*}^t S_{t-s} b_1(s)\left[ \left( D^k_\mathbf{r} X_s \right)(\mathbf{z}) \right] dW_s.
		\end{aligned}
	\end{equation*}
	Otherwise if $t \in [0,\mathbf{r}^*)$, we observe that $\left( D^k_\mathbf{r} X_t \right)(\mathbf{z}) = 0$. This shows that \eqref{EqPropMall1} holds for $k$, terminates the induction step, and therefore proves \ref{PropMall1}.
	
	For \ref{PropMall2}, we use induction on $k \in \mathbb{N}$ to show \eqref{EqPropMall2}. For the induction initialization $k = 1$, we fix some $i \in \mathbb{N}$. Moreover, by using that $C_S := \sup_{t \in [0,T]} \Vert S_t \Vert_{L(H;H)} < \infty$, that $\Vert \widetilde{e}_i \Vert_Z = \big\Vert Q^{1/2} e_i \big\Vert_Z = \sqrt{\lambda_i} \Vert e_i \Vert_Z = \sqrt{\lambda_i}$ (as $Qe_i = \lambda e_i$), and Assumption~\ref{AssSPDEAff}~\ref{AssSPDEAff4}, we obtain for every $0 \leq r \leq t \leq T$ and $x \in H$ that 
	\begin{equation}
		\label{EqPropMallProof2}
		\begin{aligned}
			& \left\Vert S_{t-r} \left( b_0(r) + b_1(r) x \right)(\widetilde{e}_i) \right\Vert_H \\
			& \quad\quad \leq \left\Vert S_{t-r} \right\Vert_{L(H;H)} \left\Vert \left( b_0(r) + b_1(r) x \right)(\widetilde{e}_i) \right\Vert_H \\
			& \quad\quad \leq C_S \left( \left\Vert b_0(r)(\widetilde{e}_i) \right\Vert_H + \left\Vert (b_1(r) x)(\widetilde{e}_i) \right\Vert_H \right) \\
			& \quad\quad \leq C_S \left( \left( \sup_{z \in Z_0 \atop \Vert z \Vert_Z \leq 1} \left\Vert b_0(r)(z) \right\Vert_H \right) \Vert \widetilde{e}_i \Vert_Z + \left( \sup_{z \in Z_0 \atop \Vert z \Vert_Z \leq 1} \sup_{y \in H \atop \Vert y \Vert_H \leq 1} \left\Vert (b_1(r) y)(z) \right\Vert_H \right) \Vert x \Vert_H \Vert \widetilde{e}_i \Vert_Z \right) \\
			& \quad\quad \leq \frac{C_S C_{F,B}}{C_\lambda} \sqrt{\lambda_i} \left( 1 + \Vert x \Vert_H \right).
		\end{aligned}
	\end{equation}
	In addition, by using again that $C_S := \sup_{t \in [0,T]} \Vert S_t \Vert_{L(H;H)} < \infty$ and Assumption~\ref{AssSPDEAff}~\ref{AssSPDEAff4}, we have for every $0 \leq s \leq t \leq T$ and $x \in H$ that 
	\begin{equation}
		\label{EqPropMallProof3}
		\left\Vert S_{t-s} f_1(s) x \right\Vert_H \leq \left\Vert S_{t-s} \right\Vert_{L(H;H)} \left\Vert f_1(s) \right\Vert_{L(H;H)} \Vert x \Vert_H \leq C_S C_{F,B} \Vert x \Vert_H
	\end{equation}
	and that
	\begin{equation}
		\label{EqPropMallProof4}
		\left\Vert S_{t-s} b_1(s) x \right\Vert_{L_2(Z_0;H)} \leq \left\Vert S_{t-s} \right\Vert_{L(H;H)} \left\Vert b_1(s) \right\Vert_{L(H;L_2(Z_0;H))} \Vert x \Vert_H \leq C_S C_{F,B} \Vert x \Vert_H.
	\end{equation}
	Then, by inserting the inequality $(x+y+z)^2 \leq 3 \left( x^2 + y^2 + z^2 \right)$ for any $x,y,z \geq 0$ into \eqref{EqPropMall1}, using the inequality~\eqref{EqPropMallProof2} together with Jensen's inequality as well as Ito's isometry in \cite[Corollary~4.29]{daprato14} (using the integrand $s \mapsto \Psi_s := \mathds{1}_{[r,t]}(s) S_{t-s} b_1(s) [(D_r X_t)(\widetilde{e}_i)]$ and that $\trace\big( \Psi_s Q^{1/2} \left( \Psi_s Q^{1/2} \right)^* \big) := \sum_{i=1}^\infty \langle \Psi_s Q^{1/2} \left( \Psi_s Q^{1/2} \right)^* e_i, e_i \rangle_Z = \Vert \Psi_s \Vert_{L_2(Z_0;H)}^2$), and the inequality $(x+y)^2 \leq 2 \left( x^2 + y^2 \right)$ for any $x,y \geq 0$ together with the inequalities \eqref{EqPropMallProof3}+\eqref{EqPropMallProof4}, we conclude for a.e.~$r \in [0,T]$ and every $t \in [r,T]$ that
	\begin{equation*}
		\begin{aligned}
			\mathbb{E}\left[ \left\Vert \left( D_r X_t \right)(\widetilde{e}_i) \right\Vert_H^2 \right] & \leq 3 \mathbb{E}\left[ \left\Vert \left( S_{t-r} b_0(r) + b_1(r) X_r \right)(\widetilde{e}_i) \right\Vert_H^2 \right] \\
			& \quad\quad + 3 \mathbb{E}\left[ \left\Vert \int_r^t S_{t-s} f_1(s) \left[ (D_r X_s)(\widetilde{e}_i) \right] ds \right\Vert_H^2 \right] \\
			& \quad\quad + 3 \mathbb{E}\left[ \left\Vert \int_r^t S_{t-s} b_1(s) \left[ (D_r X_s)(\widetilde{e}_i) \right] dW_s \right\Vert_H^2 \right] \\
			& \leq 3 \frac{C_S^2 C_{F,B}^2}{C_\lambda} \lambda_i \mathbb{E}\left[ \left( 1 + \Vert X_r \Vert_H \right)^2 \right] \\
			& \quad\quad + 3 t \mathbb{E}\left[ \int_r^t \left\Vert S_{t-s} f_1(s) \left[ (D_r X_s)(\widetilde{e}_i) \right] \right\Vert_H^2 ds \right] \\
			& \quad\quad + 3 \mathbb{E}\left[ \int_r^t \left\Vert S_{t-s} b_1(s) \left[ (D_r X_s)(\widetilde{e}_i) \right] \right\Vert_{L_2(Z_0;H)}^2 ds \right] \\
			& \leq 6 \frac{C_S^2 C_{F,B}^2}{C_\lambda} \lambda_i \left( 1 + \mathbb{E}\left[ \Vert X_r \Vert_H^2 \right] \right) \\
			& \quad\quad + 3 C_S^2 C_{F,B}^2 (t+1) \int_r^t \mathbb{E}\left[ \left\Vert (D_r X_s)(\widetilde{e}_i) \right\Vert_H^2 ds \right].
		\end{aligned}
	\end{equation*}
	Hence, by using the Gr\"onwall inequality and that $1 + \mathbb{E}\left[ \Vert X_r \Vert_H^2 \right] \leq C^{(2)}_{F,B,S,t} \left( 2 + \Vert \chi_0 \Vert_H^2 \right)$ by Proposition~\ref{PropSPDE}, it follows for a.e.~$r \in [0,T]$ and every $t \in [r,T]$ that
	\begin{equation*}
		\begin{aligned}
			\mathbb{E}\left[ \left\Vert \left( D_r X_t \right)(\widetilde{e}_i) \right\Vert_H^2 \right] & \leq 6 \frac{C_S^2 C_{F,B}^2}{C_\lambda} \lambda_i \left( 1 + \mathbb{E}\left[ \Vert X_r \Vert_H^2 \right] \right) e^{3 C_S^2 C_{F,B}^2 (t+1) t} \\
			& \leq 2 C^{(2)}_{F,B,S,t} \left( 2 + \Vert \chi_0 \Vert_H^2 \right) 3 \frac{C_S^2 C_{F,B}^2}{C_\lambda} e^{3 C_S^2 C_{F,B}^2 (t+1) t} \lambda_i,
		\end{aligned}
	\end{equation*}
	which shows \eqref{EqPropMall2} for $k = 1$, a.e.~$r \in [0,T]$, and any $t \in [r,T]$. Since $D_r X_t = 0 \in L_2(Z_0;H)$ for any $t \in [0,r)$, we obtain \eqref{EqPropMall2} for $k = 1$, completing the proof of the induction initialization. Now, for the induction step, we fix some $k \in \mathbb{N} \cap [2,\infty)$, assume that \eqref{EqPropMall2} holds for $k-1$, and aim to show that \eqref{EqPropMall2} holds for $k$. To this end, we fix some $\mathbf{i} := (i_1,...,i_k) \in \mathbb{N}^k$. Moreover, by using that $C_S := \sup_{t \in [0,T]} \Vert S_t \Vert_{L(H;H)} < \infty$, Assumption~\ref{AssSPDEAff}~\ref{AssSPDEAff4}, and that $\Vert \widetilde{e}_i \Vert_Z = \big\Vert Q^{1/2} e_i \big\Vert_Z = \sqrt{\lambda_i} \Vert e_i \Vert_Z = \sqrt{\lambda_i}$ (as $Qe_i = \lambda e_i$), we conclude for every $0 \leq r \leq t \leq T$ and $x \in H$ that 
	\begin{equation}
		\label{EqPropMallProof5}
		\begin{aligned}
			\left\Vert S_{t-r} \left( b_1(r) x \right)(\widetilde{e}_i) \right\Vert_H & \leq \left\Vert S_{t-r} \right\Vert_{L(H;H)} \left\Vert \left( b_1(r) x \right)(\widetilde{e}_i) \right\Vert_H \\
			& \leq C_S \left( \sup_{z \in Z_0 \atop \Vert z \Vert_Z \leq 1} \sup_{y \in H \atop \Vert y \Vert_H \leq 1} \left\Vert (b_1(r) y)(z) \right\Vert_H \right) \Vert x \Vert_H \Vert \widetilde{e}_i \Vert_Z \\
			& \leq \frac{C_S C_{F,B}}{C_\lambda} \sqrt{\lambda_i} \Vert x \Vert_H.
		\end{aligned}
	\end{equation}	
	Then, by inserting the inequality $(x+y+z)^2 \leq 3 \left( x^2 + y^2 + z^2 \right)$ for any $x,y,z \geq 0$ into \eqref{EqPropMall1}, by using the inequality~\eqref{EqPropMallProof2} together with Jensen's inequality as well as Ito's isometry in \cite[Corollary~4.29]{daprato14} (using the integrand $s \mapsto \Psi_s := \mathds{1}_{[r_k,t]}(s) S_{t-s} b_1(s) \left[ \left( D^k_\mathbf{r} X_t \right)(\widetilde{e}_\mathbf{i}) \right]$ and that $\trace\big( \Psi_s Q^{1/2} \left( \Psi_s Q^{1/2} \right)^* \big) := \sum_{i=1}^\infty \langle \Psi_s Q^{1/2} \left( \Psi_s Q^{1/2} \right)^* e_i, e_i \rangle_Z = \Vert \Psi_s \Vert_{L_2(Z_0;H)}^2$), and the inequality $(x+y)^2 \leq 2 \left( x^2 + y^2 \right)$ for any $x,y \geq 0$ together with the inequalities \eqref{EqPropMallProof3}+\eqref{EqPropMallProof4}, it follows for a.e.~$\mathbf{r} := (r_1,...,r_k) \in [0,T]^k$ with $r_k = \mathbf{r}^*$ and every $t \in [r_k,T]$ that
	\begin{equation*}
		\begin{aligned}
			\mathbb{E}\left[ \left\Vert \left( D^k_\mathbf{r} X_t \right)(\widetilde{e}_{\mathbf{i}}) \right\Vert_H^2 \right] & \leq 3 \mathbb{E}\left[ \left\Vert S_{t-r_k} b_1(r_k)\left[ \left( D^{k-1}_{\mathbf{r}_{-k}} X_{r_k} \right)(\widetilde{e}_{\mathbf{i}_{-k}}) \right](\widetilde{e}_{\mathbf{i}_k}) \right\Vert_H^2 \right]  \\
			& \quad\quad + 3 \mathbb{E}\left[ \left\Vert \int_{r_k}^t S_{t-s} f_1(s)\left[ \left( D^k_\mathbf{r} X_t \right)(\widetilde{e}_{\mathbf{i}}) \right] ds \right\Vert_H^2 \right] \\
			& \quad\quad + 3 \mathbb{E}\left[ \left\Vert \int_{r_k}^t S_{t-s} b_1(s)\left[ \left( D^k_\mathbf{r} X_t \right)(\widetilde{e}_{\mathbf{i}}) \right] dW_s \right\Vert_H^2 \right] \\
			& \leq 3 \frac{C_S^2 C_{F,B}^2}{C_\lambda} \lambda_{i_k} \mathbb{E}\left[ \left\Vert \left( D^{k-1}_{\mathbf{r}_{-k}} X_{r_k} \right)(\widetilde{e}_{\mathbf{i}_{-k}}) \right\Vert_H^2 \right] \\
			& \quad\quad + 3 t \mathbb{E}\left[ \int_{r_k}^t \left\Vert S_{t-s} f_1(s)\left[ \left( D^k_\mathbf{r} X_t \right)(\widetilde{e}_\mathbf{i}) \right] \right\Vert_H^2 ds \right] \\
			& \quad\quad + 3 \mathbb{E}\left[ \int_{r_k}^t \left\Vert S_{t-s} b_1(s)\left[ \left( D^k_\mathbf{r} X_t \right)(\widetilde{e}_\mathbf{i}) \right] \right\Vert_{L_2(Z_0;H)}^2 ds \right] \\
			& \leq 3 \frac{C_S^2 C_{F,B}^2}{C_\lambda} \lambda_{i_k} \mathbb{E}\left[ \left\Vert \left( D^{k-1}_{\mathbf{r}_{-k}} X_{r_k} \right)(\widetilde{e}_{\mathbf{i}_{-k}}) \right\Vert_H^2 \right] \\
			& \quad\quad + 3 C_S^2 C_{F,B}^2 (t+1) \int_0^t \mathbb{E}\left[ \left\Vert \left( D^k_\mathbf{r} X_t \right)(\widetilde{e}_{\mathbf{i}}) \right\Vert_H^2 \right] ds.
		\end{aligned}
	\end{equation*}
	Hence, by using the Gr\"onwall inequality together with the induction hypothesis (i.e.~that \eqref{EqPropMall2} holds for $k-1$), it follows for a.e.~$\mathbf{r} := (r_1,...,r_k) \in [0,T]^k$ with $r_k = \mathbf{r}^*$ and every $t \in [r_k,T]$ that
	\begin{equation*}
		\begin{aligned}
			& \mathbb{E}\left[ \left\Vert \left( D^k_\mathbf{r} X_t \right)(\widetilde{e}_{\mathbf{i}}) \right\Vert_H^2 \right] \\
			& \quad \leq 3 \frac{C_S^2 C_{F,B}^2}{C_\lambda} \lambda_{i_k} \mathbb{E}\left[ \left\Vert \left( D^{k-1}_{\mathbf{r}_{-k}} X_{r_k} \right)(\widetilde{e}_{\mathbf{i}_{-k}}) \right\Vert_H^2 \right] e^{3 C_S^2 C_{F,B}^2 (t+1) t} \\
			& \quad \leq 3 \frac{C_S^2 C_{F,B}^2}{C_\lambda} \lambda_{i_k} \left( 2 C^{(2)}_{F,B,S,t} \left( 2 + \Vert \chi_0 \Vert_H^2 \right) \left( \frac{3 C_S^2 C_{F,B}^2}{C_\lambda} \right)^{k-1} e^{3 (k-1) C_S^2 C_{F,B}^2 (t+1) t} \lambda_{\mathbf{i}_{-k}} \right) e^{3 C_S^2 C_{F,B}^2 (t+1) t} \\
			& \quad = 2 C^{(2)}_{F,B,S,t} \left( 2 + \Vert \chi_0 \Vert_H^2 \right) \left( \frac{3 C_S^2 C_{F,B}^2}{C_\lambda} \right)^k e^{3 k C_S^2 C_{F,B}^2 (t+1) t} \lambda_\mathbf{i},
		\end{aligned}
	\end{equation*}
	which shows \eqref{EqPropMall2} for $k$, a.e.~$\mathbf{r} := (r_1,...,r_k) \in [0,T]^k$, and any $t \in [r_k,T]$. Otherwise, if $r_k < \mathbf{r}^*$, we choose a permutation $\sigma: \lbrace 1,...,k \rbrace \rightarrow \lbrace 1,...,k \rbrace$ such that $\sigma(k) = \argmax_{l=1,...,k} r_l$ and use that $\left( D^k_\mathbf{r} X_t \right)(\widetilde{e}_{\mathbf{i}})$ is by \eqref{EqPropMall1} symmetric (i.e.~that $\left( D^k_{\mathbf{r}_\sigma} X_t \right)(\widetilde{e}_{\mathbf{i}_\sigma}) = \left( D^k_\mathbf{r} X_t \right)(\widetilde{e}_{\mathbf{i}})$ for all permutations $\sigma: \lbrace 1,...,k \rbrace \rightarrow \lbrace 1,...,k \rbrace$, where $\mathbf{r}_\sigma := (r_{\sigma(1)},...,r_{\sigma(k)})$ and $\mathbf{i}_\sigma := (i_{\sigma(1)},...,i_{\sigma(k)})$) to conclude for a.e.~$\mathbf{r} := (r_1,...,r_k) \in [0,T]^k$ and every $t \in [\mathbf{r}^*,T]$ that
	\begin{equation*}
		\begin{aligned}
			\mathbb{E}\left[ \left\Vert \left( D^k_\mathbf{r} X_t \right)(\widetilde{e}_{\mathbf{i}}) \right\Vert_H^2 \right] & = \mathbb{E}\left[ \left\Vert \left( D^k_{\mathbf{r}_\sigma} X_t \right)(\widetilde{e}_{\mathbf{i}_\sigma}) \right\Vert_H^2 \right] \\
			& \leq 2 C^{(2)}_{F,B,S,t} \left( 2 + \Vert \chi_0 \Vert_H^2 \right) \left( \frac{3 C_S^2 C_{F,B}^2}{C_\lambda} \right)^k e^{3 k C_S^2 C_{F,B}^2 (t+1) t} \lambda_{\mathbf{i}_\sigma} \\
			& = 2 C^{(2)}_{F,B,S,t} \left( 2 + \Vert \chi_0 \Vert_H^2 \right) \left( \frac{3 C_S^2 C_{F,B}^2}{C_\lambda} \right)^k e^{3 k C_S^2 C_{F,B}^2 (t+1) t} \lambda_\mathbf{i},
		\end{aligned}
	\end{equation*}
	which shows \eqref{EqPropMall2} for $k$, a.e.~$\mathbf{r} := (r_1,...,r_k) \in [0,T]^k$, and any $t \in [\mathbf{r}^*,T]$. Since $D^k_\mathbf{r} X_t = 0 \in Z_0^{\otimes k} \otimes H$ for any $t \in [0,\mathbf{r}^*)$, we obtain \eqref{EqPropMall2} for $k$, completing the induction step and therefore proving \ref{PropMall2}.
	
	For \ref{PropMall3}, we fix some $k \in \mathbb{N}$. Then, by using that $(\widetilde{e}_\mathbf{i})_{\mathbf{i} \in \mathbb{N}^k}$ is an orthonormal basis of $(Z_0^{\otimes k},\langle\cdot,\cdot\rangle_{Z_0^{\otimes k}})$ together with Minkowski's inequality, the inequalities~\eqref{EqPropMallProof4}+\eqref{EqPropMall2}, and the Cauchy-Schwarz inequality together with $\sum_{\mathbf{i} \in \mathbb{N}^k} \lambda_\mathbf{i} = \big( \sum_{i=1}^\infty \lambda_i \big)^k$, it follows for a.e.~fixed $\mathbf{r} := (r_1,...,r_k) \in [0,T]^k$, every fixed $\mathbf{z} := \sum_{\mathbf{i} \in \mathbb{N}^k} c_\mathbf{i} \widetilde{e}_\mathbf{i} \in Z_0^{\otimes k}$ with $c_\mathbf{i} := \langle \mathbf{z},\widetilde{e}_\mathbf{i} \rangle_{Z_0^{\otimes k}} \in \mathbb{R}$ satisfying $\sum_{\mathbf{i} \in \mathbb{N}^k} \vert c_\mathbf{i} \vert^2 < \infty$, and every fixed $t \in [\mathbf{r}^*,T]$ that
	\begin{equation*}
		\begin{aligned}
			& \mathbb{E}\left[ \int_0^t \left\Vert S_{t-s} b_1(s)\left[ \left( D^k_\mathbf{r} X_s \right)(\mathbf{z}) \right] \right\Vert_{L_2(Z_0;H)}^2 ds \right]^\frac{1}{2} \\
			& \quad \leq \sum_{\mathbf{i} \in \mathbb{N}^k} \vert c_\mathbf{i} \vert \left( \int_0^t \mathbb{E}\left[ \left\Vert S_{t-s} b_1(s)\left[ \left( D^k_\mathbf{r} X_s \right)(\widetilde{e}_\mathbf{i}) \right] \right\Vert_{L_2(Z_0;H)}^2 \right] \right)^\frac{1}{2} \\
			& \quad \leq C_S C_{F,B} \sqrt{t} \sum_{\mathbf{i} \in \mathbb{N}^k} \vert c_\mathbf{i} \vert \sup_{s \in [0,t]} \mathbb{E}\left[ \left\Vert \left( D^k_\mathbf{r} X_s \right)(\widetilde{e}_\mathbf{i}) \right\Vert_H^2 \right]^\frac{1}{2} \\
			& \quad \leq C_S C_{F,B} \sqrt{t} \sqrt{2 C^{(2)}_{F,B,S,t}} \left( 2 + \Vert \chi_0 \Vert_H^2 \right)^\frac{1}{2} \left( \frac{3 C_S^2 C_{F,B}^2}{C_\lambda} \right)^\frac{k}{2} e^{\frac{3}{2} k C_S^2 C_{F,B}^2 (t+1) t} \sum_{\mathbf{i} \in \mathbb{N}^k} \vert c_\mathbf{i} \vert \sqrt{\lambda_\mathbf{i}} \\
			& \quad \leq C_S C_{F,B} \sqrt{t} \sqrt{2 C^{(2)}_{F,B,S,t}} \left( 2 + \Vert \chi_0 \Vert_H^2 \right)^\frac{1}{2} \left( \frac{3 C_S^2 C_{F,B}^2}{C_\lambda} \right)^\frac{k}{2} e^{\frac{3}{2} k C_S^2 C_{F,B}^2 (t+1) t} \left( \sum_{\mathbf{i} \in \mathbb{N}^k} \vert c_\mathbf{i} \vert^2 \right)^\frac{1}{2} \left( \sum_{\mathbf{i} \in \mathbb{N}^k} \lambda_\mathbf{i} \right)^\frac{1}{2} \\
			& \quad \leq C_S C_{F,B} \sqrt{t} \sqrt{2 C^{(2)}_{F,B,S,t}} \left( 2 + \Vert \chi_0 \Vert_H^2 \right)^\frac{1}{2} \left( \frac{3 C_S^2 C_{F,B}^2}{C_\lambda} \right)^\frac{k}{2} e^{\frac{3}{2} k C_S^2 C_{F,B}^2 (t+1) t} \left( \sum_{\mathbf{i} \in \mathbb{N}^k} \vert c_\mathbf{i} \vert^2 \right)^\frac{1}{2} \left( \sum_{i=1}^\infty \lambda_i \right)^\frac{k}{2} < \infty.
		\end{aligned}
	\end{equation*}
	Hence, we can apply \cite[Proposition~4.28]{daprato14} to conclude that
	\begin{equation*}
		\mathbb{E}\left[ \int_0^t S_{t-s} b_1(s)\left[ \left( D^k_\mathbf{r} X_s \right)(\mathbf{z}) \right] dW_s \right] = 0.
	\end{equation*}
	Thus, by taking the expectation in \eqref{EqPropMall1}, we obtain the identity \eqref{EqPropMall3}.
	
	For \ref{PropMall4}, we use induction on $k \in \mathbb{N}$ to show \eqref{EqPropMall4}. For the induction initialization $k = 1$, we fix some $i \in \mathbb{N}$. Then, by applying the triangle inequality in \eqref{EqPropMall3} and using the inequalities~\eqref{EqPropMallProof2}+\eqref{EqPropMallProof3}, we conclude for a.e.~$r \in [0,T]$ and every $t \in [r,T]$ that
	\begin{equation*}
		\begin{aligned}
			\left\Vert \mathbb{E}\left[ \left( D_r X_t \right)(\widetilde{e}_i) \right] \right\Vert_H & \leq \left\Vert \left( S_{t-r} b_0(r) + b_1(r) \mathbb{E}\left[ X_r \right] \right)(\widetilde{e}_i) \right\Vert_H + \left\Vert \int_r^t S_{t-s} f_1(s) \mathbb{E}\left[ (D_r X_s)(\widetilde{e}_i) \right] ds \right\Vert_H \\
			& \leq \frac{C_S C_{F,B}}{\sqrt{C_\lambda}} \sqrt{\lambda_i} \left( 1 + \left\Vert \mathbb{E}\left[ X_r \right] \right\Vert_H \right) + C_S C_{F,B} \int_0^t \left\Vert \mathbb{E}\left[ (D_r X_s)(\widetilde{e}_i) \right] \right\Vert_H ds
		\end{aligned}
	\end{equation*}
	Hence, by using the Gr\"onwall inequality and that $1 + \mathbb{E}\left[ \Vert X_r \Vert_H \right] \leq \sqrt{C^{(2)}_{F,B,S,t}} \left( 2 + \Vert \chi_0 \Vert_H^2 \right)^{1/2}$ by Proposition~\ref{PropSPDE}, it follows for a.e.~$r \in [0,T]$ and every $t \in [r,T]$ that
	\begin{equation*}
		\begin{aligned}
			\left\Vert \mathbb{E}\left[ \left( D_r X_t \right)(\widetilde{e}_i) \right] \right\Vert_H & \leq \frac{C_S C_{F,B}}{\sqrt{C_\lambda}} \sqrt{\lambda_i} \left( 1 + \left\Vert \mathbb{E}\left[ X_r \right] \right\Vert_H \right) e^{C_S C_{F,B} t} \\
			& \leq \sqrt{C^{(2)}_{F,B,S,t}} \left( 2 + \Vert \chi_0 \Vert_H^2 \right)^\frac{1}{2} \frac{C_S C_{F,B}}{\sqrt{C_\lambda}} e^{C_S C_{F,B} t} \sqrt{\lambda_i},
		\end{aligned}
	\end{equation*}
	which shows \eqref{EqPropMall4} for $k = 1$, a.e.~$r \in [0,T]$, and any $t \in [r,T]$. Since $D_r X_t = 0 \in L_2(Z_0;H)$ for any $t \in [0,r)$, we obtain \eqref{EqPropMall4} for $k = 1$, completing the proof of the induction initialization. Now, for the induction step, we fix some $k \in \mathbb{N} \cap [2,\infty)$, assume that \eqref{EqPropMall4} holds for $k-1$, and aim to show that \eqref{EqPropMall4} holds for $k$. To this end, we fix some $\mathbf{i} \in \mathbb{N}^k$. Then, by applying the triangle inequality in \eqref{EqPropMall3} and using the inequalities~\eqref{EqPropMallProof2}+\eqref{EqPropMallProof3}, it follows for a.e.~$\mathbf{r} := (r_1,...,r_k) \in [0,T]^k$ with $r_k = \mathbf{r}^*$ and every $t \in [r_k,T]$ that
	\begin{equation*}
		\begin{aligned}
			\left\Vert \mathbb{E}\left[ \left( D^k_\mathbf{r} X_t \right)(\widetilde{e}_{\mathbf{i}}) \right] \right\Vert_H & \leq \left\Vert S_{t-r_k} b_1(r_k) \mathbb{E}\left[ \left( D^{k-1}_{\mathbf{r}_{-k}} X_{r_k} \right)(\widetilde{e}_{\mathbf{i}_{-k}}) \right](\widetilde{e}_{\mathbf{i}_k}) \right\Vert_H  \\
			& \quad\quad + \left\Vert \int_{r_k}^t S_{t-s} f_1(s) \mathbb{E}\left[ \left( D^k_\mathbf{r} X_t \right)(\widetilde{e}_{\mathbf{i}}) \right] ds \right\Vert_H \\
			& \leq \frac{C_S C_{F,B}}{\sqrt{C_\lambda}} \sqrt{\lambda_{i_k}} \left\Vert \mathbb{E}\left[ \left( D^{k-1}_{\mathbf{r}_{-k}} X_{r_k} \right)(\widetilde{e}_{\mathbf{i}_{-k}}) \right] \right\Vert_H \\
			& \quad\quad + C_S C_{F,B} \int_0^t \left\Vert \mathbb{E}\left[ \left( D^k_\mathbf{r} X_t \right)(\widetilde{e}_{\mathbf{i}}) \right] \right\Vert_H ds.
		\end{aligned}
	\end{equation*}
	Hence, by using the Gr\"onwall inequality together with the induction hypothesis (i.e.~that \eqref{EqPropMall4} holds for $k-1$), it follows for a.e.~$\mathbf{r} := (r_1,...,r_k) \in [0,T]^k$ with $r_k = \mathbf{r}^*$ and every $t \in [r_k,T]$ that
	\begin{equation*}
		\begin{aligned}
			& \left\Vert \mathbb{E}\left[ \left( D^k_\mathbf{r} X_t \right)(\widetilde{e}_{\mathbf{i}}) \right] \right\Vert_H \\
			& \quad\quad \leq \frac{C_S C_{F,B}}{\sqrt{C_\lambda}} \sqrt{\lambda_{i_k}} \left\Vert \mathbb{E}\left[ \left( D^{k-1}_{\mathbf{r}_{-k}} X_{r_k} \right)(\widetilde{e}_{\mathbf{i}_{-k}}) \right] \right\Vert_H e^{C_S C_{F,B} t} \\
			& \quad\quad \leq \frac{C_S C_{F,B}}{\sqrt{C_\lambda}} \sqrt{\lambda_{i_k}} \left( \sqrt{C^{(2)}_{F,B,S,t}} \left( 2 + \Vert \chi_0 \Vert_H^2 \right)^\frac{1}{2} \left( \frac{C_S C_{F,B}}{\sqrt{C_\lambda}} \right)^{k-1} e^{(k-1) C_S C_{F,B} t} \sqrt{\lambda_{\mathbf{i}_{-k}}} \right) e^{C_S C_{F,B} t} \\
			& \quad\quad = \sqrt{C^{(2)}_{F,B,S,t}} \left( 2 + \Vert \chi_0 \Vert_H^2 \right)^\frac{1}{2} \left( \frac{C_S C_{F,B}}{\sqrt{C_\lambda}} \right)^k e^{k C_S C_{F,B} t} \sqrt{\lambda_\mathbf{i}},
		\end{aligned}
	\end{equation*}
	which shows \eqref{EqPropMall4} for $k$, a.e.~$\mathbf{r} := (r_1,...,r_k) \in [0,T]^k$, and any $t \in [r_k,T]$. Otherwise, if $r_k < \mathbf{r}^*$, we choose a permutation $\sigma: \lbrace 1,...,k \rbrace \rightarrow \lbrace 1,...,k \rbrace$ satisfying $\sigma(k) = \argmax_{l=1,...,k} r_l$ and use that $\mathbb{E}\left[ \left( D^k_\mathbf{r} X_t \right)(\widetilde{e}_{\mathbf{i}}) \right]$ is by \eqref{EqPropMall3} symmetric (i.e.~that $\mathbb{E}\left[ \left( D^k_{\mathbf{r}_\sigma} X_t \right)(\widetilde{e}_{\mathbf{i}_\sigma}) \right] = \mathbb{E}\left[ \left( D^k_\mathbf{r} X_t \right)(\widetilde{e}_{\mathbf{i}}) \right]$ for all permutations $\sigma: \lbrace 1,...,k \rbrace \rightarrow \lbrace 1,...,k \rbrace$, where $\mathbf{r}_\sigma := (r_{\sigma(1)},...,r_{\sigma(k)})$ and $\mathbf{i}_\sigma := (i_{\sigma(1)},...,i_{\sigma(k)})$) to conclude for a.e.~$\mathbf{r} := (r_1,...,r_k) \in [0,T]^k$ and every $t \in [\mathbf{r}^*,T]$ that
	\begin{equation*}
		\begin{aligned}
			\left\Vert \mathbb{E}\left[ \left( D^k_\mathbf{r} X_t \right)(\widetilde{e}_\mathbf{i}) \right] \right\Vert_H & = \left\Vert \mathbb{E}\left[ \left( D^k_{\mathbf{r}_\sigma} X_t \right)(\widetilde{e}_{\mathbf{i}_\sigma}) \right] \right\Vert_H \\
			& \leq \sqrt{C^{(2)}_{F,B,S,t}} \left( 2 + \Vert \chi_0 \Vert_H^2 \right)^\frac{1}{2} \left( \frac{C_S C_{F,B}}{\sqrt{C_\lambda}} \right)^k e^{k C_S C_{F,B} t} \sqrt{\lambda_{\mathbf{i}_\sigma}} \\
			& = \sqrt{C^{(2)}_{F,B,S,t}} \left( 2 + \Vert \chi_0 \Vert_H^2 \right)^\frac{1}{2} \left( \frac{C_S C_{F,B}}{\sqrt{C_\lambda}} \right)^k e^{k C_S C_{F,B} t} \sqrt{\lambda_\mathbf{i}},
		\end{aligned}
	\end{equation*}
	which shows \eqref{EqPropMall4} for $k$, a.e.~$\mathbf{r} := (r_1,...,r_k) \in [0,T]^k$, and any $t \in [\mathbf{r}^*,T]$. Since $\mathbb{E}\left[ D^k_\mathbf{r} X_t \right] = 0 \in Z_0^{\otimes k} \otimes H$ for any $t \in [0,\mathbf{r}^*)$, we obtain \eqref{EqPropMall2} for $k$, completing the induction step and proving \ref{PropMall4}. 
\end{proof}

\pagebreak

\subsection{Proof of Theorem~\ref{ThmDetAR}}
\label{SecProofDetAR}

\begin{proof}[Proof of Theorem~\ref{ThmDetAR}]
	Let $X: [0,T] \times \Omega \rightarrow H := W^{k,2}(U,\mathcal{L}(U),w;\mathbb{R}^d)$ be a mild solution of \eqref{EqDefSPDE} satisfying Assumption~\ref{AssSPDEAff}+\ref{AssPropag} and fix some $I,J,K,N \in \mathbb{N}$. Then, we can split the approximation error into four parts: \labeltext{(i)}{ThmDetARProof1} The truncation error of $X_T$ by using only Wick polynomials up to order $K$; \labeltext{(ii)}{ThmDetARProof2} the projection error by using only the first $I$ basis elements $(\widetilde{e}_i)_{i=1,...,I}$ of $(Z_0,\langle \cdot,\rangle_{Z_0})$; \labeltext{(iii)}{ThmDetARProof3} the projection error by using only the first $J$ basis elements $(g_j)_{j = 1,...,J}$ of $(L^2([0,T]),\langle\cdot,\cdot\rangle_{L^2([0,T])})$; \labeltext{(iv)}{ThmDetARProof4} the approximation error of $x_\alpha^{(T)} := \mathbb{E}[X_T \xi_\alpha] \in W^{k,2}(U,\mathcal{L}(U),w;\mathbb{R}^d)$ by a deterministic neural network with $N$ neurons, for $\alpha \in \mathcal{J}_{I,J,K}$. 
	
	For \ref{ThmDetARProof1}, we use that the Wick polynomials $(\xi_\alpha)_{\alpha \in \mathcal{J}}$ are orthonormal in $(L^2(\Omega),\langle \cdot,\cdot \rangle_{L^2(\Omega)})$ (see Lemma~\ref{LemmaWick}), the projection $\Pi_k: L^2(\Omega;H) \rightarrow L^2(\Omega;H)$ onto the Wick polynomials $(\xi_\alpha)_{\alpha \in \mathcal{J}_k}$ of order $k$, Proposition~\ref{PropStroock}, and Lemma~\ref{LemmaL2SkorokhodInt} to obtain that
	\begin{equation*}
		\begin{aligned}
			\mathbb{E}\left[ \left\Vert \sum_{\alpha \in \mathcal{J}, \, \vert \alpha \vert > K} x_\alpha^{(T)} \xi_\alpha \right\Vert_H^2 \right] & = \sum_{k=K+1}^\infty \mathbb{E}\left[ \left\Vert \sum_{\alpha \in \mathcal{J}_k} \mathbb{E}[X_T \xi_\alpha] \xi_\alpha \right\Vert_H^2 \right] \\
			& = \sum_{k=K+1}^\infty \mathbb{E}\left[ \Vert \Pi_k X_T \Vert_H^2 \right] \\
			& = \sum_{k=K+1}^\infty \frac{1}{\left( k! \right)^2} \mathbb{E}\left[ \left\Vert \delta^k\left( \mathbb{E}\left[ D^k X_T \right] \right) \right\Vert_H^2 \right] \\
			& \leq \sum_{k=K+1}^\infty \frac{k!}{\left( k! \right)^2} \left\Vert \mathbb{E}\left[ D^k X_T \right] \right\Vert_{L^2([0,T]^k;Z_0^{\otimes k} \otimes H)}^2 \\
			& = \sum_{k=K+1}^\infty \frac{1}{k!} \int_{[0,T]^k} \sum_{\mathbf{i} \in \mathbb{N}^k} \left\Vert \mathbb{E}\left[ \left( D^k_\mathbf{r} X_T \right)(\widetilde{e}_\mathbf{i}) \right] \right\Vert_H^2 d\mathbf{r}.
		\end{aligned}
	\end{equation*}
	Hence, by using Proposition~\ref{PropMall}~\ref{PropMall4}, that $\sum_{\mathbf{i} \in \mathbb{N}^k} \lambda_\mathbf{i} = \big( \sum_{i=1}^\infty \lambda_i \big)^k = C_\lambda^k$, the identity $\sum_{k=K+1}^\infty \frac{c^k}{k!} = \frac{c^{K+1}}{(K+1)!} \sum_{k=K+1}^\infty \frac{c^{k-(K+1)}}{(k-(K+1))!} = \frac{c^{K+1}}{(K+1)!} \sum_{j=0}^\infty \frac{c^j}{j!} = e^c \frac{c^{K+1}}{(K+1)!}$ for any $c \in [0,\infty)$, and the constant $C_{11} := \big( C^{(2)}_{F,B,S,T} e^{C_S^2 C_{F,B}^2 T \exp(2 C_S C_{F,B} T)} \big)^{1/2} \geq 0$ (depending only on $C_{F,B} > 0$, $C_S := \sup_{t \in [0,T]} \Vert S_t \Vert_{L(H;H)} < \infty$, and $T > 0$), it follows that
	\begin{equation}
		\label{EqThmDetARProof1}
		\begin{aligned}
			\mathbb{E}\left[ \left\Vert \sum_{\alpha \in \mathcal{J}, \, \vert \alpha \vert > K} x_\alpha^{(T)} \xi_\alpha \right\Vert_H^2 \right] & \leq \sum_{k=K+1}^\infty \frac{T^k}{k!} \sum_{\mathbf{i} \in \mathbb{N}^k} \left( C^{(2)}_{F,B,S,T} \left( 2 + \Vert \chi_0 \Vert_H^2 \right) \left( \frac{C_S^2 C_{F,B}^2}{C_\lambda} \right)^k e^{2k C_S C_{F,B} T} \lambda_\mathbf{i} \right) \\
			& \leq C^{(2)}_{F,B,S,T} \left( 2 + \Vert \chi_0 \Vert_H^2 \right) \sum_{k=K+1}^\infty \frac{\left( C_S^2 C_{F,B}^2 T e^{2 C_S C_{F,B} T} \right)^k}{k!} \\
			& \leq C^{(2)}_{F,B,S,T} \left( 2 + \Vert \chi_0 \Vert_H^2 \right) e^{C_S^2 C_{F,B}^2 T \exp(2 C_S C_{F,B} T)} \frac{\left( C_S^2 C_{F,B}^2 T e^{2 C_S C_{F,B} T} \right)^{K+1}}{(K+1)!} \\
			& \leq C_{11}^2 \left( 2 + \Vert \chi_0 \Vert_H^2 \right) \frac{\left( C_S^2 C_{F,B}^2 T e^{2 C_S C_{F,B} T} \right)^{K+1}}{(K+1)!}.
		\end{aligned}
	\end{equation}
	This bounds the truncation error of $X_T$ by using the Wick polynomials only up to order $K$.
	
	For \ref{ThmDetARProof2}, we fix some $k \in 1,...,K$ and an orthonormal basis $(y_n)_{n \in \mathbb{N}}$ of $(H,\langle \cdot,\cdot \rangle_H)$. Then, by using that $\big( \bigotimes_{l=1}^k \left( g_{j_l} \widetilde{e}_{i_l} \right) \otimes y_n \big)_{\mathbf{i},\mathbf{j} \in \mathbb{N}^k, \, n \in \mathbb{N}}$ is an orthonormal basis of $(L^2([0,T]^k;Z_0^{\otimes k} \otimes H),\langle \cdot, \cdot \rangle_{L^2([0,T]^k;Z_0^{\otimes k} \otimes H)})$ together with the same steps as in line 1-7 of \eqref{EqPropStroockProof1}, Lemma~\ref{LemmaL2SkorokhodInt}, that $\Vert \widetilde{x} \Vert_{\widetilde{H}}^2 = \sum_{i=1}^\infty \vert \langle x, \widetilde{h}_i \rangle_{\widetilde{H}} \vert^2$ for any $\widetilde{x} \in \widetilde{H}$ and Hilbert space $(\widetilde{H},\langle\cdot,\cdot\rangle_{\widetilde{H}})$ with complete orthonormal basis $(\widetilde{h}_i)_{i \in \mathbb{N}}$, that $\big( y_n \prod_{l=1}^k g_{j_l}(\cdot) \big)_{\mathbf{j} \in \mathbb{N}^k, \, n \in \mathbb{N}}$ is an orthonormal basis of $(L^2([0,T]^k;H),\langle \cdot, \cdot \rangle_{L^2([0,T]^k;H)})$, we obtain that
	\vspace{-0.05cm}
	\begin{equation}
		\label{EqThmDetARProof2}
		\begin{aligned}
			& \mathbb{E}\left[ \left\Vert \sum_{\alpha \in \mathcal{J}_k \atop \exists i > I: \, \exists j \in \mathbb{N}: \, \alpha_{i,j} > 0} \mathbb{E}[X_T \xi_\alpha] \xi_\alpha \right\Vert_H^2 \right] = \mathbb{E}\left[ \left\Vert \sum_{\alpha \in \mathcal{J}_k \atop \exists i > I: \, \exists j \in \mathbb{N}: \, \alpha_{i,j} > 0} \sum_{n=1}^\infty \mathbb{E}\left[ \langle X_T, \xi_\alpha y_n \rangle_H \right] \xi_\alpha y_n \right\Vert_H^2 \right] \\
			& = \mathbb{E}\left[ \left\Vert \frac{1}{k!} \delta^k\left( \sum_{\mathbf{i},\mathbf{j} \in \mathbb{N}^k \atop \exists l: \, i_l > I} \sum_{n=1}^\infty \Big\langle \mathbb{E}\left[ D^k X_T \right] , \bigotimes_{l=1}^k \left( g_{j_l} \widetilde{e}_{i_l} \right) \otimes y_n \Big\rangle_{L^2([0,T]^k;Z_0^{\otimes k} \otimes H)} \bigotimes_{l=1}^k \left( g_{j_l} \widetilde{e}_{i_l} \right) \otimes y_n\right) \right\Vert_H^2 \right] \\
			& \leq \frac{k!}{\left( k! \right)^2} \left\Vert \sum_{\mathbf{i},\mathbf{j} \in \mathbb{N}^k \atop \exists l: \, i_l > I} \sum_{n=1}^\infty \Big\langle \mathbb{E}\left[ D^k X_T \right] , \bigotimes_{l=1}^k \left( g_{j_l} \widetilde{e}_{i_l} \right) \otimes y_n \Big\rangle_{L^2([0,T]^k;Z_0^{\otimes k} \otimes H)} \bigotimes_{l=1}^k \left( g_{j_l} \widetilde{e}_{i_l} \right) \otimes y_n \right\Vert_{L^2([0,T]^k;Z_0^{\otimes k} \otimes H)}^2 \\
			& = \frac{1}{k!} \sum_{\mathbf{i},\mathbf{j} \in \mathbb{N}^k \atop \exists l: \, i_l > I} \sum_{n=1}^\infty \left\vert \Big\langle \mathbb{E}\left[ D^k X_T \right] , \bigotimes_{l=1}^k \left( g_{j_l} \widetilde{e}_{i_l} \right) \otimes y_n \Big\rangle_{L^2([0,T]^k;Z_0^{\otimes k} \otimes H)} \right\vert^2 \\
			& = \frac{1}{k!} \sum_{\mathbf{i},\mathbf{j} \in \mathbb{N}^k \atop \exists l: \, i_l > I} \sum_{n=1}^\infty \left\vert \int_{[0,T]^k} \Big\langle \mathbb{E}\left[ \left( D^k_\mathbf{r} X_T \right)(\widetilde{e}_\mathbf{i}) \right], y_n \prod_{l=1}^k g_{j_l}(r_l) \Big\rangle_H d\mathbf{r} \right\vert^2 \\
			& = \frac{1}{k!} \sum_{\mathbf{i} \in \mathbb{N}^k \atop \exists l: \, i_l > I} \left\Vert \mathbb{E}\left[ \left( D^k X_T \right)(\widetilde{e}_\mathbf{i}) \right] \right\Vert_{L^2([0,T]^k;H)}^2.
		\end{aligned}
	\end{equation}
	Hence, using Proposition~\ref{PropMall}~\ref{PropMall4}, that $\sum_{\mathbf{i} \in \mathbb{N}^k, \, i_l > I} \lambda_\mathbf{i} = \big( \sum_{i=1}^\infty \lambda_i \big)^{k-1} \big( \sum_{i=I+1}^\infty \lambda_i \big) = C_\lambda^{k-1} \sum_{i=I+1}^\infty \lambda_i$, the inequality $\sum_{k=1}^K \frac{c^k}{(k-1)!} \leq c \sum_{k=1}^\infty \frac{c^{k-1}}{(k-1)!} = c e^c$ for any $c \in [0,\infty)$, as well as the constant $C_{12} := \big( C^{(2)}_{F,B,S,T} C_\lambda^{-1} C_S^2 C_{F,B}^2 T e^{2 C_S C_{F,B} T} e^{C_S^2 C_{F,B}^2 T \exp(2 C_S C_{F,B} T)} \big)^{1/2} \geq 0$ (depending only on $C_{F,B} > 0$, $C_\lambda := \sum_{i=1}^\infty \lambda_i < \infty$ (cf.~\eqref{EqDefBMi}), $C_S := \sup_{t \in [0,T]} \Vert S_t \Vert_{L(H;H)} < \infty$, and $T > 0$), we have
	\vspace{-0.05cm}
	\begin{equation}
		\label{EqThmDetARProof3}
		\begin{aligned}
			& \mathbb{E}\left[ \left\Vert \sum_{\alpha \in \mathcal{J}, \, \vert \alpha \vert \leq k} x_\alpha^{(T)} \xi_\alpha - \sum_{\alpha \in \mathcal{J}, \, \vert \alpha \vert \leq k \atop \forall i > I: \, \forall j \in \mathbb{N}: \, \alpha_{i,j} = 0} x_\alpha^{(T)} \xi_\alpha \right\Vert_H^2 \right] = \sum_{k=1}^K \mathbb{E}\left[ \left\Vert \sum_{\alpha \in \mathcal{J}_k \atop \exists i > I: \, \exists j \in \mathbb{N}: \, \alpha_{i,j} > 0} \mathbb{E}[X_T \xi_\alpha] \xi_\alpha \right\Vert_H^2 \right] \\
			& \quad\quad \leq \sum_{k=1}^K \frac{1}{k!} \sum_{\mathbf{i} \in \mathbb{N}^k \atop \exists l: \, i_l > I} \int_{[0,T]^k} \left\Vert \mathbb{E}\left[ \left( D^k_\mathbf{r} X_T \right)(\widetilde{e}_\mathbf{i}) \right] \right\Vert_H^2 d\mathbf{r} \\
			& \quad\quad \leq \sum_{k=1}^K \frac{T^k}{k!} \sum_{\mathbf{i} \in \mathbb{N}^k \atop \exists l: \, i_l > I} \left( C^{(2)}_{F,B,S,T} \left( 2 + \Vert \chi_0 \Vert_H^2 \right) \left( \frac{C_S^2 C_{F,B}^2 T}{C_\lambda} \right)^k e^{2k C_S C_{F,B} T} \lambda_\mathbf{i} \right) \\
			& \quad\quad = C^{(2)}_{F,B,S,T} \left( 2 + \Vert \chi_0 \Vert_H^2 \right) \sum_{k=1}^K \frac{1}{k!} \left( \frac{C_S^2 C_{F,B}^2 T}{C_\lambda} \right)^k e^{2k C_S C_{F,B} T} \sum_{l=1}^k \sum_{\mathbf{i} \in \mathbb{N}^k \atop i_l > I} \lambda_\mathbf{i} \\
			& \quad\quad \leq \frac{C^{(2)}_{F,B,S,T}}{C_\lambda} \left( 2 + \Vert \chi_0 \Vert_H^2 \right) \left( \sum_{k=1}^K \frac{\left( C_S^2 C_{F,B}^2 T e^{2 C_S C_{F,B} T} \right)^k}{(k-1)!} \right) \left( \sum_{i=I+1}^\infty \lambda_i \right) \\
			& \quad\quad \leq C_{12}^2 \left( 2 + \Vert \chi_0 \Vert_H^2 \right) \sum_{i=I+1}^\infty \lambda_i.
		\end{aligned}
	\end{equation}
	This bounds the projection error by using only the basis elements $(\widetilde{e}_i)_{i=1,...,I}$ of $(Z_0,\langle \cdot,\rangle_{Z_0})$.

	For \ref{ThmDetARProof3}, we fix again some $k \in 1,...,K$ and an orthonormal basis $(y_n)_{n \in \mathbb{N}}$ of $(H,\langle \cdot,\cdot \rangle_H)$. Then, by following the arguments of \eqref{EqThmDetARProof2} (including line 1-7 of \eqref{EqPropStroockProof1} and Lemma~\ref{LemmaL2SkorokhodInt}), we obtain that
	\begin{equation*}
		\begin{aligned}
			& \mathbb{E}\left[ \left\Vert \sum_{\alpha \in \mathcal{J}_k \atop \underset{\exists i \in \mathbb{N}: \, \exists j > J: \, \alpha_{i,j} > 0}{\forall i > I: \, \forall j \in \mathbb{N}: \, \alpha_{i,j} = 0}} \mathbb{E}[X_T \xi_\alpha] \xi_\alpha \right\Vert_H^2 \right] = \mathbb{E}\left[ \left\Vert \sum_{\alpha \in \mathcal{J}_k \atop \underset{\exists i \in \mathbb{N}: \, \exists j > J: \, \alpha_{i,j} > 0}{\forall i > I: \, \forall j \in \mathbb{N}: \, \alpha_{i,j} = 0}} \sum_{n=1}^\infty \mathbb{E}[X_T \xi_\alpha] \xi_\alpha y_n \right\Vert_H^2 \right] \\
			& \leq \frac{1}{k!} \left\Vert \sum_{\mathbf{i},\mathbf{j} \in \mathbb{N}^k \atop \underset{\exists l: \, j_l > J}{\forall l: i_l \leq I}} \sum_{n=1}^\infty \Big\langle \mathbb{E}\left[ D^k X_T \right] , \bigotimes_{l=1}^k \left( g_{j_l} \widetilde{e}_{i_l} \right) \otimes y_n \Big\rangle_{L^2([0,T]^k;Z_0^{\otimes k} \otimes H)} \bigotimes_{l=1}^k \left( g_{j_l} \widetilde{e}_{i_l} \right) \otimes y_n \right\Vert_{L^2([0,T]^k;Z_0^{\otimes k} \otimes H)}^2 \\
			& = \frac{1}{k!} \sum_{\mathbf{i},\mathbf{j} \in \mathbb{N}^k \atop \underset{\exists l: \, j_l > J}{\forall l: i_l \leq I}} \sum_{n=1}^\infty \left\vert \Big\langle \mathbb{E}\left[ D^k X_T \right] , \bigotimes_{l=1}^k \left( g_{j_l} \widetilde{e}_{i_l} \right) \otimes y_n \Big\rangle_{L^2([0,T]^k;Z_0^{\otimes k} \otimes H)} \right\vert^2 \\
			& \leq \frac{1}{k!} \sum_{\mathbf{i},\mathbf{j} \in \mathbb{N}^k \atop \exists l: \, j_l > J} \sum_{n=1}^\infty \left\vert \Big\langle \int_{[0,T]^k} \mathbb{E}\left[ \left( D^k_\mathbf{r} X_T \right)(\widetilde{e}_\mathbf{i}) \right] \prod_{l=1}^k g_{j_l}(r_l) d\mathbf{r}, y_n \Big\rangle_H \right\vert^2 \\
			& = \frac{1}{k!} \sum_{\mathbf{i},\mathbf{j} \in \mathbb{N}^k \atop \exists l: \, j_l > J} \left\Vert \int_{[0,T]^k} \mathbb{E}\left[ \left( D^k_\mathbf{r} X_T \right)(\widetilde{e}_\mathbf{i}) \right] \prod_{l=1}^k g_{j_l}(r_l) d\mathbf{r} \right\Vert^2.
		\end{aligned}
	\end{equation*}
	Hence, Proposition~\ref{PropMall}~\ref{PropMall4}, that $\sum_{\mathbf{j} \in \mathbb{N}^k, \, j_l > J} \big( \int_{[0,T]^k} \prod_{l=1}^k \vert g_{j_l}(r_l) \vert d\mathbf{r} \big)^2 = C_g^{k-1} \sum_{j=J+1}^\infty \Vert g_j \Vert_{L^1([0,T])}^2$ by Fubini's theorem (with $C_g := \sum_{j=1}^\infty \Vert g_j \Vert_{L^1([0,T])}^2$), that $\sum_{\mathbf{i} \in \mathbb{N}^k} \lambda_\mathbf{i} = C_\lambda^k$, and the same steps as in \eqref{EqThmDetARProof3} together with the constant $C_{13} := \big( C^{(2)}_{F,B,S,T} C_g^{-1} C_S^2 C_{F,B}^2 C_g e^{2 C_S C_{F,B} T} e^{C_S^2 C_{F,B}^2 C_g \exp(2 C_S C_{F,B} T)} \big)^{1/2} \geq 0$ (depending only on $C_{F,B} > 0$, $C_\lambda := \sum_{i=1}^\infty \lambda_i < \infty$ (cf.~\eqref{EqDefBMi}), $C_g > 0$, $C_S := \sup_{t \in [0,T]} \Vert S_t \Vert_{L(H;H)} < \infty$, and $T > 0$) show that
	\begin{equation}
		\label{EqThmDetARProof4}
		\begin{aligned}
			& \mathbb{E}\left[ \left\Vert \sum_{\alpha \in \mathcal{J}_k \atop \forall i > I: \, \forall j \in \mathbb{N}: \, \alpha_{i,j} = 0} x_\alpha^{(T)} \xi_\alpha - \sum_{\alpha \in \mathcal{J}_{I,J,K}} x_\alpha^{(T)} \xi_\alpha \right\Vert_H^2 \right] = \sum_{k=1}^K \mathbb{E}\left[ \left\Vert \sum_{\alpha \in \mathcal{J}_k \atop \underset{\exists i \in \mathbb{N}: \, \exists j > J: \, \alpha_{i,j} > 0}{\forall i > I: \, \forall j \in \mathbb{N}: \, \alpha_{i,j} = 0}} \mathbb{E}[X_T \xi_\alpha] \xi_\alpha \right\Vert_H^2 \right] \\
			& \leq \sum_{k=1}^K \frac{1}{k!} \sum_{\mathbf{i},\mathbf{j} \in \mathbb{N}^k \atop \exists l: \, j_l > J} \left( \int_{[0,T]^k} \left\Vert \mathbb{E}\left[ \left( D^k_\mathbf{r} X_T \right)(\widetilde{e}_\mathbf{i}) \right] \right\Vert_H \prod_{l=1}^k \vert g_{j_l}(r_l) \vert d\mathbf{r} \right)^2 \\
			& \leq \sum_{k=1}^K \frac{1}{k!} \sum_{\mathbf{i} \in \mathbb{N}^k} \left( C^{(2)}_{F,B,S,T} \left( 2 + \Vert \chi_0 \Vert_H^2 \right) \left( \frac{C_S^2 C_{F,B}^2}{C_\lambda} \right)^k e^{2k C_S C_{F,B} T} \lambda_\mathbf{i} \right) \sum_{l=1}^k \sum_{\mathbf{j} \in \mathbb{N}^k \atop j_l > J} \left( \int_{[0,T]^k} \prod_{l=1}^k \vert g_{j_l}(r_l) \vert d\mathbf{r} \right)^2 \\
			& \leq \frac{C^{(2)}_{F,B,S,T}}{C_g} \left( 2 + \Vert \chi_0 \Vert_H^2 \right) \left( \sum_{k=1}^K \frac{\left( C_S^2 C_{F,B}^2 C_g e^{2 C_S C_{F,B} T} \right)^k}{(k-1)!} \right) \left( \sum_{j=J+1}^\infty \Vert g_j \Vert_{L^1([0,T])}^2 \right) \\
			& \leq C_{13}^2 \left( 2 + \Vert \chi_0 \Vert_H^2 \right) \sum_{j=J+1}^\infty \Vert g_j \Vert_{L^1([0,T])}^2.
		\end{aligned}
	\end{equation}
	This bounds the projection error by using only the basis elements $(g_j)_{j = 1,...,J}$ of $(L^2([0,T]),\langle\cdot,\cdot\rangle_{L^2([0,T])})$.
	
	For \ref{ThmDetARProof4}, we fix some $\alpha \in \mathcal{J}$. Then, by using Assumption~\ref{AssPropag}, i.e.~that $x_\alpha^{(T)} := \mathbb{E}[X_T \xi_\alpha] \in W^{k,2}(U,\mathcal{L}(U),w;\mathbb{R}^d)$ is either constant (and thus can be approximated by a deterministic neural network $\varphi_\alpha^{(T)} \in \mathcal{NN}^\rho_{U,d}$ having $N$ neurons without any error, where we set $c_\alpha := 0$) or that $x_\alpha^{(T)} \in L^1(\mathbb{R}^m,\mathcal{L}(\mathbb{R}^m),du;\mathbb{R}^d)$ has $(\lceil\gamma\rceil+2)$-times differentiable Fourier transform such that the constant $c_\alpha > 0$ defined in \eqref{EqAssPropag} is finite, we can apply the approximation rate for deterministic neural networks in \cite[Theorem~3.6]{neufeld24} together with \cite[Proposition~3.8]{neufeld24} to obtain a constant $C_1 > 0$ (depending only on $\gamma \in [0,\infty)$ and $\psi \in \mathcal{S}_0(\mathbb{R};\mathbb{C})$) and some deterministic neural network $\varphi_\alpha^{(T)} \in \mathcal{NN}^\rho_{U,d}$ with $N$ neurons such that
	\begin{equation*}
		\left\Vert x_\alpha^{(T)} - \varphi_\alpha^{(T)} \right\Vert_{W^{k,2}(U,\mathcal{L}(U),w;\mathbb{R}^d)} \leq C_1 \Vert \rho \Vert_{C^k_{pol,\gamma}(\mathbb{R})} \frac{C^{(\gamma)}_{U,w} m^\frac{k}{2} \pi^\frac{m+1}{4}}{\zeta_1^\frac{m}{2} \left\vert C^{(\psi,\rho)}_m \right\vert \Gamma\left( \frac{m+1}{2} \right)^\frac{1}{2}} \frac{c_\alpha}{\sqrt{N}}.
	\end{equation*}
	Hence, by using that the Wick polynomials $(\xi_\alpha)_{\alpha \in \mathcal{J}}$ are orthonormal in $(L^2(\Omega),\langle\cdot,\cdot\rangle_{L^2(\Omega)})$, we have
	\begin{equation}
		\label{EqThmDetARProof6}
		\begin{aligned}
			& \mathbb{E}\left[ \left\Vert \sum_{\alpha \in \mathcal{J}_{I,J,K}} x_\alpha^{(T)} \xi_\alpha - \sum_{\alpha \in \mathcal{J}_{I,J,K}} \varphi_\alpha^{(T)} \xi_\alpha \right\Vert_H^2 \right]^\frac{1}{2} = \left( \sum_{\alpha \in \mathcal{J}_{I,J,K}} \mathbb{E}\left[ \left\Vert x_\alpha^{(T)} \xi_\alpha - \sum_{\alpha \in \mathcal{J}_{I,J,K}} \varphi_\alpha^{(T)} \xi_\alpha \right\Vert_H^2 \right] \right)^\frac{1}{2} \\
			& \quad\quad = \left( \sum_{\alpha \in \mathcal{J}_{I,J,K}} \left\Vert x_\alpha^{(T)} - \varphi_\alpha^{(T)} \right\Vert_H^2 \mathbb{E}\left[ \xi_\alpha^2 \right] \right)^\frac{1}{2} \\
			& \quad\quad \leq C_2 \Vert \rho \Vert_{C^k_{pol,\gamma}(\mathbb{R})} \frac{C^{(\gamma)}_{U,w} m^\frac{k}{2} \pi^\frac{m+1}{4}}{\zeta_1^\frac{m}{2} \left\vert C^{(\psi,\rho)}_m \right\vert \Gamma\left( \frac{m+1}{2} \right)^\frac{1}{2}} \frac{\left( \sum_{\alpha \in \mathcal{J}_{I,J,K}} c_\alpha^2 \right)^\frac{1}{2}}{\sqrt{N}}.
		\end{aligned}
	\end{equation}
	This bounds the approximation error of $x_\alpha^{(T)} := \mathbb{E}[X_T \xi_\alpha] \in W^{k,2}(U,\mathcal{L}(U),w;\mathbb{R}^d)$ by some deterministic neural network $\varphi_\alpha \in \mathcal{NN}^\rho_{U,d}$ with $N$ neurons.
	
	Finally, by combining \eqref{EqThmDetARProof1}+\eqref{EqThmDetARProof3}+\eqref{EqThmDetARProof4}+\eqref{EqThmDetARProof6} with Minkowski's and by using that the Wick polynomials $(\xi_\alpha)_{\alpha \in \mathcal{J}}$ are orthonormal among different orders (see Lemma~\ref{LemmaWick}) as well as the constants $C_2 := \max(C_{11},C_{12},C_{13}) > 0$ (depending only on $C_{F,B} > 0$, $C_g > 0$, $\sum_{i=1}^\infty \lambda_i < \infty$, $C_S := \sup_{t \in [0,T]} \Vert S_t \Vert_{L(H;H)} < \infty$, and $T > 0$) and $C_1 > 0$ introduced above, it follows that
	\begin{equation}
		\label{EqThmDetARProof7}
		\begin{aligned}
			& \mathbb{E}\left[ \left\Vert X_T - \sum_{\alpha \in \mathcal{J}_{I,J,K}} \varphi_\alpha^{(T)} \xi_\alpha \right\Vert_H^2 \right]^\frac{1}{2} \\
			& \leq \mathbb{E}\left[ \left\Vert X_T - \sum_{\alpha \in \mathcal{J} \atop \vert\alpha\vert \leq K} x_\alpha^{(T)} \xi_\alpha \right\Vert_H^2 \right]^\frac{1}{2} + \sum_{k=1}^K \mathbb{E}\left[ \left\Vert \sum_{\alpha \in \mathcal{J}_k} x_\alpha^{(T)} \xi_\alpha - \sum_{\alpha \in \mathcal{J}_k \atop \forall i > I: \, \forall j \in \mathbb{N}: \, \alpha_{i,j} = 0} x_\alpha^{(T)} \xi_\alpha \right\Vert_H^2 \right]^\frac{1}{2} \\
			& \quad\quad + \mathbb{E}\left[ \left\Vert \sum_{\alpha \in \mathcal{J}_k \atop \forall i > I: \, \forall j \in \mathbb{N}: \, \alpha_{i,j} = 0} x_\alpha^{(T)} \xi_\alpha - \sum_{\alpha \in \mathcal{J}_{I,J,K}} x_\alpha^{(T)} \xi_\alpha \right\Vert_H^2 \right]^\frac{1}{2} + \mathbb{E}\left[ \left\Vert \sum_{\alpha \in \mathcal{J}_{I,J,K}} x_\alpha^{(T)} \xi_\alpha - \sum_{\alpha \in \mathcal{J}_{I,J,K}} \varphi_\alpha^{(T)} \xi_\alpha \right\Vert_H^2 \right]^\frac{1}{2} \\
			& \leq C_2 \left( 2 + \Vert \chi_0 \Vert_H^2 \right)^\frac{1}{2} \left( \left( \sum_{i=I+1}^\infty \lambda_i \right)^\frac{1}{2} + \left( \sum_{j=J+1}^\infty \Vert g_j \Vert_{L^1([0,T])}^2 \right)^\frac{1}{2} + \frac{\left( C_S C_{F,B} \sqrt{T} e^{C_S C_{F,B} T} \right)^{K+1}}{\sqrt{(K+1)!}} \right) \\
			& \quad\quad + C_1 \Vert \rho \Vert_{C^k_{pol,\gamma}(\mathbb{R})} \frac{C^{(\gamma)}_{U,w} m^\frac{k}{2} \pi^\frac{m+1}{4}}{\zeta_1^\frac{m}{2} \left\vert C^{(\psi,\rho)}_m \right\vert \Gamma\left( \frac{m+1}{2} \right)^\frac{1}{2}} \frac{\left( \sum_{\alpha \in \mathcal{J}_{I,J,K}} c_\alpha^2 \right)^\frac{1}{2}}{\sqrt{N}},
		\end{aligned}
	\end{equation}
	which completes the proof.
\end{proof}

\pagebreak

\subsection{Proof of Theorem~\ref{ThmRandAR}}
\label{SecProofRandAR}

\begin{proof}[Proof of Theorem~\ref{ThmRandAR}]
	Let $X: [0,T] \times \Omega \rightarrow W^{k,2}(U,\mathcal{L}(U),w;\mathbb{R}^d)$ be a mild solution of \eqref{EqDefSPDE} with coefficients satisfying Assumption~\ref{AssSPDEAff} and fix some $I,J,K,N \in \mathbb{N}$. Then, by following the proof of Theorem~\ref{ThmDetAR}, the approximation error can be split up into the four parts \ref{ThmDetARProof1}-\ref{ThmDetARProof4}. While \ref{ThmDetARProof1}-\ref{ThmDetARProof3} consist of the same steps as in the proof of Theorem~\ref{ThmDetAR}, part~\ref{ThmDetARProof4} is now the approximation error of $x_\alpha^{(T)} := \mathbb{E}[X_T \xi_\alpha] \in W^{k,2}(U,\mathcal{L}(U),w;\mathbb{R}^d)$ by a random neural network with $N$ neurons, for $\alpha \in \mathcal{J}_{I,J,K}$. 
	
	For \ref{ThmDetARProof4}, we fix some $\alpha \in \mathcal{J}$. Then, by using Assumption~\ref{AssPropag}, i.e.~that the function $x_\alpha^{(T)} := \mathbb{E}[X_T \xi_\alpha] \in W^{k,2}(U,\mathcal{L}(U),w;\mathbb{R}^d)$ is either constant (and thus can be approximated by a random neural network $\Phi_\alpha^{(T)} \in \mathcal{RN}^\rho_{U,d}$ having $N$ neurons without any error, where we set $c_\alpha := 0$) or that $x_\alpha^{(T)} \in L^1(\mathbb{R}^m,\mathcal{L}(\mathbb{R}^m),du;\mathbb{R}^d)$ has $(\lceil\gamma\rceil+2)$-times differentiable Fourier transform such that the constant $c_\alpha > 0$ defined in \eqref{EqAssPropag} is finite, we can apply the approximation rate for random neural networks in \cite[Corollary~4.20]{neufeld23} together with \cite[Proposition~4.22]{neufeld23} to obtain the same constant $C_1 > 0$ as in Theorem~\ref{ThmDetAR} and a random neural network $\Phi_\alpha^{(T)} \in \mathcal{RN}^\rho_{U,d}$ with $N$ neurons such that
	\begin{equation*}
		\mathbb{E}\left[ \left\Vert x_\alpha^{(T)} - \Phi_\alpha^{(T)} \right\Vert_{W^{k,2}(U,\mathcal{L}(U),w;\mathbb{R}^d)}^2 \right]^\frac{1}{2} \leq C_1 \Vert \rho \Vert_{C^k_{pol,\gamma}(\mathbb{R})} \frac{C^{(\gamma)}_{U,w} m^\frac{k}{2} \pi^\frac{m+1}{4}}{\zeta_1^\frac{m}{2} \left\vert C^{(\psi,\rho)}_m \right\vert \Gamma\left( \frac{m+1}{2} \right)^\frac{1}{2}} \frac{c_\alpha}{\sqrt{N}}.
	\end{equation*}
	Since the Wick polynomials $(\xi_\alpha)_{\alpha \in \mathcal{J}}$ are orthonormal in $(L^2(\Omega),\langle\cdot,\cdot\rangle_{L^2(\Omega)})$ and $\Phi_\alpha^{(T)} \in \mathcal{RN}^\rho_{U,d}$ is independent of $(X_t)_{t \in [0,T]}$ (as $(A_{1,n},B_n)_{n \in \mathbb{N}}$ are by Assumption~\ref{AssPDF} independent of $(W_t)_{t \in [0,T]}$), it follows that
	\begin{equation}
		\label{EqThmRandARProof2}
		\begin{aligned}
			& \mathbb{E}\left[ \left\Vert \sum_{\alpha \in \mathcal{J}_{I,J,K}} x_\alpha^{(T)} \xi_\alpha - \sum_{\alpha \in \mathcal{J}_{I,J,K}} \Phi_\alpha^{(T)} \xi_\alpha \right\Vert_H^2 \right]^\frac{1}{2} = \left( \sum_{\alpha \in \mathcal{J}_{I,J,K}} \mathbb{E}\left[ \left\Vert x_\alpha^{(T)} \xi_\alpha - \Phi_\alpha^{(T)} \xi_\alpha \right\Vert_H^2 \right] \right)^\frac{1}{2} \\
			& \quad\quad = \left( \sum_{\alpha \in \mathcal{J}_{I,J,K}} \mathbb{E}\left[ \left\Vert x_\alpha^{(T)} - \Phi_\alpha^{(T)} \right\Vert_H^2 \right] \mathbb{E}\left[ \xi_\alpha^2 \right] \right)^\frac{1}{2} \\
			& \quad\quad \leq C_1 \Vert \rho \Vert_{C^k_{pol,\gamma}(\mathbb{R})} \frac{C^{(\gamma)}_{U,w} m^\frac{k}{2} \pi^\frac{m+1}{4}}{\zeta_1^\frac{m}{2} \left\vert C^{(\psi,\rho)}_m \right\vert \Gamma\left( \frac{m+1}{2} \right)^\frac{1}{2}} \frac{\left( \sum_{\alpha \in \mathcal{J}_{I,J,K}} c_\alpha^2 \right)^\frac{1}{2}}{\sqrt{N}}.
		\end{aligned}
	\end{equation}
	This bounds the approximation error of $x_\alpha^{(T)} := \mathbb{E}[X_T \xi_\alpha] \in W^{k,2}(U,\mathcal{L}(U),w;\mathbb{R}^d)$ by some random neural network $\Phi_\alpha \in \mathcal{RN}^\rho_{U,d}$ with $N$ neurons.
	
	Thus, by combining the four approximation errors in \ref{ThmDetARProof1}-\ref{ThmDetARProof4}, where part \ref{ThmDetARProof1}-\ref{ThmDetARProof3} consist of the same steps as in the proof of Theorem~\ref{ThmDetAR}, we can follow the inequalities in \eqref{EqThmDetARProof7} to obtain the conclusion.
\end{proof}

\subsection{Proof of results in Section~\ref{SecNumerics}}
\label{SecProofsNE}

\begin{proof}[Proof of Lemma~\ref{LemmaStochHeat}]
	We recall from \cite[Exercise~5.9]{hairer09} that the Laplacian $\Delta: W^{2,2}(\mathbb{R}^m,\mathcal{L}(\mathbb{R}^m),w) \rightarrow L^2(\mathbb{R}^m,\mathcal{L}(\mathbb{R}^m),w)$ defined in \eqref{EqDefLapl} generates the $C_0$-semigroup $(S_t)_{t \in [0,T]}$ defined in \eqref{EqDefLaplSemigr}. Moreover, by using the coefficients $[0,T] \times \Omega \times L^2(\mathbb{R}^m,\mathcal{L}(\mathbb{R}^m),w) \ni (t,\omega,x) \mapsto F(t,\omega,x) := 0 \in L^2(\mathbb{R}^m,\mathcal{L}(\mathbb{R}^m),w)$ as well as $[0,T] \times \Omega \times L^2(\mathbb{R}^m,\mathcal{L}(\mathbb{R}^m),w) \ni (t,\omega,x) \mapsto B(t,\omega,x) := b_0 := (z \mapsto z) \in L_2(\mathbb{R};L^2(\mathbb{R}^m,\mathcal{L}(\mathbb{R}^m),w))$, and that $\int_{\mathbb{R}^m} \phi_t(u-v) dv = \int_{\mathbb{R}^m} \phi_t(y) dy = 1$ for any $u \in \mathbb{R}^m$ (by using the substitution $u-v \mapsto y$), we conclude for every $t \in [0,T]$ that
	\begin{equation*}
		\begin{aligned}
			& S_t \chi_0 + \int_0^t S_{t-s} F(s,\cdot,X_s) ds + \int_0^t S_{t-s} B(s,\cdot,X_s) dW_s \\
			& \quad\quad = \left( u \mapsto \int_{\mathbb{R}^m} \phi_{t-s}(u-v) \chi_0(v) dv + \int_0^t \int_{\mathbb{R}^m} \phi_{t-s}(u-v) dv dW_s \right) \\
			& \quad\quad = \left( u \mapsto \int_{\mathbb{R}^m} \phi_{t-s}(u-v) \chi_0(v) dv + W_t \right)	\\
			& \quad\quad = S_t \chi_0 + b_0 W_t.
		\end{aligned}
	\end{equation*}
	This shows that \eqref{EqLemmaStochHeat1} is a mild solution of \eqref{EqDefStochHeat}.
\end{proof}

\vspace{0.3cm}

\subsection*{Acknowledgments:} The first author was partly supported by the Nanyang Assistant Professorship Grant (NAP Grant) \emph{Machine Learning based Algorithms in Finance and Insurance}. The second author was partly supported by the FinsureTech Hub of ETH Zurich.

\vspace{0.3cm}

\bibliographystyle{plain}
\bibliography{mybib}

\end{document}